\documentclass[11pt,letterpaper]{article}
\usepackage[letterpaper, total={6.5in, 9in}]{geometry}

\usepackage{amsmath}
\usepackage{amssymb}
\usepackage{enumitem}
\usepackage{fancyhdr}
\usepackage{tikz}
\usepackage{bbm}
\usetikzlibrary{trees}
\usepackage{hyperref}

\usepackage{url}
\usepackage{graphicx}%
\usepackage{multirow}%
\usepackage{amsmath,amsfonts}%
\usepackage{amsthm}%
\usepackage{mathrsfs}%
\usepackage[title]{appendix}%
\usepackage{xcolor}%
\usepackage{textcomp}%
\usepackage{manyfoot}%
\usepackage{booktabs}%
\usepackage{algorithm}%
\usepackage{algpseudocode}%
\usepackage{listings}%
\usepackage{chngcntr}
\usepackage{dsfont}
\usepackage{comment}

\usepackage[skip=10pt plus1pt]{parskip}

\newcommand{\ball}{\mathrm{B}}
\newcommand{\R}{\mathbb{R}}

\newcommand{\norm}[1]{\left\lVert #1 \right\rVert}
\newcommand{\abs}[1]{\left\vert #1 \right\rvert}

\newcommand{\cov}[1]{\mathrm{Cov}{\left( #1\right)}}

\newcommand{\local}{\mathrm{loc}}

\newcommand{\K}{\widetilde{K}}
\newcommand{\Q}{\widetilde{Q}}

\newcommand{\opnorm}[1]{\norm{#1}_{\operatorname{op}}}

\newcommand{\E}{\mathbb{E}}

\newcommand{\proj}{\operatorname{proj}}
\newcommand{\prox}{\operatorname{prox}}
\newcommand{\dist}{\operatorname{dist}}
\newcommand{\epi}{\operatorname{epi}}
\newcommand{\rin}{r_{\mathrm{in}}}

\newcommand{\argmin}{\mathrm{argmin}  }

\newcommand{\lam}{\lambda}
\newcommand{\nn}{\nonumber}

\newcommand{\D}{\mathrm{d}}

\newcommand{\cW}{\mathcal{W}}
\newcommand{\KL}{\mathsf{KL}}

\newcommand{\Var}{\mathsf{Var}}

\newcommand{\vol}{\mathsf{vol}}

\newcommand{\inner}[2]{\left\langle #1  ,#2 \right\rangle}

\newcommand{\brac}[1]{\left( #1\right)}

\newcommand{\red}[1]{\textcolor{red}{#1}} 
\newcommand{\blue}[1]{\textcolor{blue}{#1}} 
\newcommand*{\QEDA}{\hfill\qedsymbol}

\allowdisplaybreaks

\newtheorem{thm}{Theorem}[section]
\newtheorem{theorem}[thm]{Theorem}

\newtheorem{lemma}[thm]{Lemma}

\newtheorem{proposition}[thm]{Proposition}

\newtheorem{remark}[thm]{Remark}

\renewcommand{\qedsymbol}{$\square$}

\makeatletter
\newcounter{subsubsubsection}[subsubsection]
\renewcommand\thesubsubsubsection{\thesubsubsection.\arabic{subsubsubsection}}

\newcommand{\subsubsubsection}[1]{%
  \par\addvspace{0.00ex \@plus 1ex \@minus .2ex}
  \refstepcounter{subsubsubsection}%
  \noindent{%
    \normalfont\normalsize\rmfamily\mdseries\scshape\raggedright
    \thesubsubsubsection.\enskip #1\par
  }%
  \addvspace{0.05in \@plus .2ex}
}
\makeatother

\usepackage{float}
\usepackage{caption}    
\usepackage{subcaption} 

	\title{Constrained and Composite Sampling via Proximal Sampler}
	\date{February 15, 2026 (1st revision: November 13, 2026)}
\author{
    	Thanh Dang  \thanks{Department of Computer Science, University of Rochester, Rochester, NY 14620. 
        (email: {\tt ycloud777@gmail.com}).}
        \qquad
	 	Jiaming Liang \thanks{
        Goergen Institute for Data Science and Artificial Intelligence and Department of Computer Science, University of Rochester, Rochester, NY 14620 (email: {\tt jiaming.liang@rochester.edu}). This work was supported by AFOSR grant FA9550-25-1-0182.}
	}

\begin{document}
\maketitle
	\begin{abstract}	


We study two log-concave sampling problems: constrained sampling and composite sampling. First, we consider sampling from a distribution proportional to $\exp(-f(x))$ supported on a convex set $K\subset\R^d$, where $f$ is convex. Using an epigraph lifting, we reduce this problem to uniform sampling up to a one-dimensional exponential tilt over a convex set in $\R^{d+1}$. We then apply the proximal sampler to the lifted distribution. We develop an unbiased implementation of its restricted Gaussian oracle by employing the cutting-plane method and rejection sampling, requiring only a separation oracle for $K$ and evaluation and subgradient oracles for $f$, without relying on projections or barrier functions.


Second, we study composite sampling from a distribution proportional to $\exp(-f(x)-h(x))$, where $f$ and $h$ are closed and convex. This structure arises naturally in Bayesian inference, where $f$ represents data fidelity and $h$ encodes prior information. By applying an epigraph lifting to $h$, we reduce composite sampling to constrained sampling in $\R^{d+1}$, allowing us to build on the approach developed in the first part. This leads to a double epigraph lifting formulation in $\R^{d+2}$, for which we implement the proximal sampler using different combinations of subgradient and proximal oracles for $f$ and $h$. For both sampling problems, we establish mixing time bounds in R\'enyi and $\chi^2$ divergences.

Finally, we revisit the convergence analysis of the proximal sampler from a variational perspective. Starting from variational formulations of its two half-steps as optimization problems over probability measures, we use a dynamic formulation of entropy-regularized optimal transport to recast them as optimal control problems.  This allows us to derive their governing evolution PDEs directly, without first analyzing particle-level stochastic dynamics. We then employ Otto calculus to establish de Bruijn identities and derive contraction bounds in Rényi and $\chi^2$ divergences. Overall, our results strengthen the connection between optimization and sampling at both the algorithmic level, through epigraph lifting and oracle-based implementations, and the analytical level, through variational analysis over probability measures.
	
	{\bf Key words.} Log-concave sampling, epigraph lifting, proximal sampler, optimal control, Otto calculus, separation oracles. 
		\end{abstract}

\section{Introduction}
\label{sec:intro}

In this paper, we study two log-concave sampling problems: constrained sampling and composite sampling.

\textbf{Problem formulation.} Let $K\subseteq\mathbb{R}^d$ be a non-empty closed convex body that contains the origin, and let $f:\mathbb{R}^d\to\mathbb{R}$ be a closed convex function on $K$. We consider sampling from a distribution $\nu$ for which
\begin{equation}
\label{silif:def:nu}
\nu(x) \propto \exp(-f(x)) \mathbf{1}_K(x) = \exp \left(-f(x) - I_K(x)\right),
\end{equation}
where $I_K$ is the extended indicator function of the set $K$, taking the value $0$ if $x \in K$ and $+\infty$ otherwise; while $ \mathbf{1}_K(x)$ equals $1$ on $K$ and $0$ otherwise. To motivate the sampling approach studied in this paper, let us consider the analogy with convex optimization. The problem $\min\{f(x): x\in K\}$ can be rewritten in epigraph form $\min \{t: (x,t) \in Q\}$,
where $Q = \{(x,t) \in K \times \mathbb{R} : f(x) \le t\}$ is the epigraph of $f$ restricted to $K$. 
This reformulation lifts the problem from $\mathbb{R}^d$ to $\mathbb{R}^{d+1}$ and indirectly handles $f$ through the constraint $f(x) \leq t$ instead of an objective function. It has been widely used in optimization, for example, the bundle methods \cite{lemarechal1975extension,lemarechal1995new,wolfe1975method}. 

Motivated by this idea from optimization, we reformulate constrained sampling \eqref{silif:def:nu} in an epigraph form. Let us consider the distribution
\begin{equation}
\label{silif:def:pi}
\pi(x,t) \propto  \exp(-at-I_Q(x,t)),\qquad Q = \{(x,t)\in \R^{d+1} : x \in K,   f(x) \le a t\},
\end{equation}
for some scaling constant $a > 0$. The crucial observation is that the $X$-marginal of $\pi$ coincides with the target distribution $\nu$.
Indeed, integrating out $t$ yields $\pi^X(x)=\int \pi(x,t) \mathrm{d}t \propto \left(\int_{f(x)/a}^{\infty} e^{-at} \mathrm{d}t\right)\mathbf{1}_K(x)=\frac{1}{a}e^{-f(x)}\mathbf{1}_K(x)\propto\nu(x)$. Thus, if one can efficiently sample $(x,t)$ from $\pi$ over the lifted set $Q$, then discarding the $t$-coordinate produces samples from the target $\nu$. 

  We next consider the composite sampling problem of drawing samples from
\begin{equation}
\label{dolif:def:nu}
    \tilde{\nu}(x)\ \propto\ \exp \left(-f(x)-h(x)\right),
\end{equation}
where $f$ and $h$ are closed convex functions on $\R^d$. Such composite structure arises naturally in Bayesian models: $f$ typically represents the negative log-likelihood, while $h$ plays the role of a negative log-prior. 
Via a a two-step epigraph lifting (\emph{double lifting}), we obtain
\begin{equation}
\label{dolif:def:piandQ:intro}
\tilde{\pi}(x,s,t)\ \propto\ \exp \left(-bt-I_{\Q}(x,s,t)\right),
\qquad
\Q=\{(x,s,t) \in \R^{d+2}:\ h(x)\le as,\ f(x)+as\le bt\}
\end{equation}
for scaling constants $a,b >0$. Integrating out $(s,t)$ shows that $\tilde{\pi}^X=\tilde{\nu}$, thus sampling $(x,s,t)\sim\tilde{\pi}$ and discarding $(s,t)$ yields samples from $\tilde{\nu}$. 

At this point, we observe that both the constrained and composite sampling problems \eqref{silif:def:nu} and \eqref{dolif:def:nu} have been reduced to sampling from distributions that are uniform on convex sets up to a one-dimensional exponential tilt, as in \eqref{silif:def:pi} and \eqref{dolif:def:piandQ:intro}, respectively. This links our setting to uniform sampling on convex bodies where a broad range of algorithmic tools are available. In particular, motivated by the recent works ~\cite{dangliang2025oracle,kook2026inandout}, we will employ the proximal sampler by~\cite{leestructured21}, to sample from the lifted measures.




The proximal sampler was first proposed by
\cite{leestructured21} for sampling from log-concave
distributions $\pi^X(x)\propto\exp(-f(x))$.
Given a step size $\eta>0$, it performs Gibbs sampling
on the augmented distribution
\[
\pi^{X,Y}(x,y)\propto
\exp\left(-f(x)-\frac{\|x-y\|^2}{2\eta}\right),
\]
whose $X$-marginal is $\pi^X$. Thus, the resulting Gibbs sampler preserves $\pi^X$ as its stationary
distribution without introducing asymptotic bias.
Each iteration consists of two alternating half-steps:
\begin{equation}\label{eq:half}
    y_k\sim\pi^{Y|X}(\cdot|x_k)
=\mathcal N(x_k,\eta I_d),
\qquad
x_{k+1}\sim\pi^{X|Y}(\cdot|y_k).
\end{equation}
The second half-step is called the restricted Gaussian
oracle (RGO), which samples from a distribution proportional
to $\exp(-f(x)-\|x-y_k\|^2/(2\eta))$.
For our lifted targets, the RGO reduces to sampling
from a Gaussian distribution truncated to a convex set,
and its efficient implementation is a central challenge.

A subsequent paper \cite{chen2022improved} observed that
these two half-steps in \eqref{eq:half} can be embedded into a pair of
simultaneous forward and backward diffusion flows.
This viewpoint extends the analysis of the proximal sampler
beyond log-concave targets \cite{leestructured21} to
distributions satisfying functional inequalities such as
the Poincar\'e inequality and the log-Sobolev inequality.


\textbf{Proximal sampler revisited.} 
We recall the following variational representation of the proximal sampler
from \cite[Theorem 10]{chen2022improved}: the law $\rho_k^Y$ (resp., $\rho_{k+1}^X$) of $y_k$ (resp., $x_{k+1}$) can be written as
\begin{equation}\label{eq:proxsamplervariational}
\rho_k^Y
=\underset{\mu\in\mathcal P_2(\R^d)}{\argmin}
\frac{1}{2\eta}\cW_{2,2\eta}^2(\rho_k^X,\mu),
\qquad \rho_{k+1}^X = \underset{\mu\in \mathcal P_2(\mathbb R^d)}\argmin \left\{ \int f
\D\mu+\frac1{2\eta}\cW_{2,2\eta}^2(\rho_k^Y,\mu) \right\},
\end{equation}
where $\cW_{2,\varepsilon}$ denotes the entropy-regularized Wasserstein-2 distance 
\begin{equation}\label{eq:EW2}
\cW_{2,\varepsilon}^2(\rho,\mu)
:=
\inf_{\gamma\in\Pi(\rho,\mu)}
\left\{
\int \|x-y\|^2 \D \gamma
-
\varepsilon H(\gamma)
\right\}.
\end{equation}
It is shown in Lemma~\ref{lem:fixedpoint} that the target $\pi \propto \exp(-f(x)$ is a fixed point of the proximal sampler \eqref{eq:proxsamplervariational}. In particular,
if $\rho_k^X=\pi$, then
$\rho_k^Y=\pi P_\eta$ and $\rho_{k+1}^X=\pi$,
where $(P_t)_{t\geq 0}$ denotes the heat semigroup.
Our goal is to establish contraction toward this fixed point
by analyzing the two variational half-steps in
\eqref{eq:proxsamplervariational}.
Motivated by the simultaneous flows studied
in~\cite{chen2022improved}, we investigate the continuous-time
dynamics of these half-steps.
To this end, define the reference curves
\begin{equation}
\label{eq:referencecurves}
\pi_t=\pi P_t,
\qquad
\pi_t^{\mathrm b}=\pi P_{\eta-t},
\qquad
0\leq t\leq\eta.
\end{equation}
The curve $(\pi_t)_{0\leq t\leq \eta}$ evolves forward from $\pi$ to $\pi P_\eta$,
while the curve $(\pi_t^{\mathrm b})_{0\leq t\leq \eta}$ evolves backward from $\pi P_\eta$ to $\pi$.
For each $k\ge 0$, consider the following continuous-time variational problems:
\begin{align}
\rho_t
=
&\operatorname*{argmin}_{\sigma\in\mathcal P_2(\mathbb R^d)}
\left\{
\frac1{2t}\cW_{2,2t}^2(\rho_k^X,\sigma)
\right\},\qquad \rho_0=\rho_k^X\qquad\text{for}\quad 0<t\leq \eta,
\label{eq:forwardprox}\\
\rho_t^{\mathrm b}
=
&\operatorname*{argmin}_{\sigma\in\mathcal P_2(\mathbb R^d)}
\left\{
\int \Phi_t^{\mathrm b}\D\sigma
+
\frac1{2t}\cW_{2,2t}^2(\rho_k^Y,\sigma)
\right\},\qquad \rho_0^{\mathrm b}=\rho^Y_k\qquad \text{for}\quad 0<t\leq
\eta, \label{eq:backwardprox}
\end{align}
where $\Phi_t^{\mathrm b}=-\log\pi_t^{\mathrm b}$. 
At $t=\eta$, \eqref{eq:forwardprox} and \eqref{eq:backwardprox} recover the two updates in
\eqref{eq:proxsamplervariational}, since
$\Phi_\eta^{\mathrm b}=f+\log Z$.
Thus, $(\rho_t)_{0\leq t\leq\eta}$ and
$(\rho_t^{\mathrm b})_{0\leq t\leq\eta}$
interpolate the two half-steps of the proximal sampler,
with $\pi_t$ and $\pi_t^{\mathrm b}$ serving as their
corresponding reference curves at the fixed point.

The key step in our analysis is to identify the evolution PDE governing these interpolation curves. 
In the spirit of the Benamou-Brenier dynamic formula of the Wasserstein distance, we derive a dynamic formulation of $\cW_{2,2t}^2(\rho,\mu)$ and use it to recast \eqref{eq:forwardprox} and \eqref{eq:backwardprox} as two optimal control problems.
We identify the optimal controls through an intuitive duality argument and then follow up with rigorous analysis, yielding the heat equation for $\rho_t$ and a time-dependent Fokker-Planck equation for $\rho_t^{\mathrm b}$. 
We then use Otto calculus to derive de Bruijn identities and combine them with Poincar\'e inequalities to establish contraction bounds in $\chi^2$ and R\'enyi divergences.

Unlike analyses that begin with particle-level stochastic dynamics, our approach derives the governing PDEs directly
from the variational problems over probability measures. This provides a more direct and principled
variational route to the contraction analysis of the proximal sampler.

\textbf{Contributions.}
A central goal of our work is to bring ideas from optimization into sampling. To this end, we develop efficient algorithms for both  constrained sampling \eqref{silif:def:nu} and composite sampling \eqref{dolif:def:nu} via epigraph lifting and the proximal sampler. These can be viewed as sampling counterparts of constrained and composite optimization. We summarize our main theoretical results as follows.

\emph{Proximal sampler analysis.} 
We provide a novel variational analysis of the proximal sampler following the approach outlined above, and establish contraction in R\'enyi and $\chi^2$ divergences.


\emph{Constrained sampling.} Denote by $R_\mathrm{in},R_\mathrm{out} >0$ as respectively the inradius and outer radius of $K$. Assume $f$ is $L$-Lipschitz continuous on $K$
and there exists an initial distribution $\nu_0$ that is $M$-warm with respect to $\nu$. For the proximal sampler (Algorithm~\ref{silif:alg:ASF_pi}), we implement the RGO via Algorithm~\ref{silif:RGO} by combining a rejection sampler with a cutting-plane method for the lifted set $Q$ in \eqref{silif:def:pi}.  In particular, this implementation requires only a subgradient oracle for $f$ and a separation oracle for $K$. Unlike existing constrained samplers that depend on
projections, reflections, barrier functions, or mirror maps, 
our approach enforces feasibility using only minimal separation-based oracle access.

Combining Theorems~\ref{silif:theo:outer} and \ref{silif:theo:separation} yields an iteration complexity of
\begin{equation}
\label{intro:silif:itercomplex}
\widetilde {\cal O}\!\left(
q
\left[
\frac{(LR_\mathrm{out}+d)^2}{R_\mathrm{in}^2}
\|\cov{\nu}\|_{\mathrm{op}}
+d
\right]
\log\frac{M}{\varepsilon}
\right)
\end{equation}
to get $\varepsilon$-accuracy in R\'enyi divergence $\mathcal{R}_q$ for $q\geq1$. Meanwhile, the complexity in $\chi^2$ divergence is
\[
\widetilde {\cal O}\!\left(
\left[
\frac{(LR_\mathrm{out}+d)^2}{R_\mathrm{in}^2}
\|\cov{\nu}\|_{\mathrm{op}}
+d
\right]
\log\frac{M}{\varepsilon}
\right).
\]
Regarding oracle complexity, each iteration of Algorithm~\ref{silif:alg:ASF_pi} makes $\widetilde {\mathcal{O}}(d)$ expected calls to the separation oracle for $K$ and the subgradient oracle for $f$. Finally, regarding proposal complexity, the average number of proposals is $\mathcal{O}(M)$ per proximal sampler iteration.

\emph{Composite sampling.} In addition to the assumptions on $f$ and $h$ at \eqref{dolif:def:nu}, we assume they are respectively $L_f$ and $L_h$-Lipschitz continuous. 
Moreover, suppose there exists an $M$-warm initial distribution $\tilde{\nu}_0$ with respect to $\tilde{\nu}$. For the proximal sampler (Algorithm~\ref{dolif:alg:ASF_pi}), we implement the RGO via Algorithm~\ref{dolif:RGO} by combining a rejection sampler with a cutting-plane method on $\tilde Q$ in \eqref{dolif:def:piandQ:intro}. The latter requires a separation oracle for $\tilde Q$. The composite viewpoint offers flexibility here: we show in Subsection~\ref{subsec:RGO:composite}  that any combination of subgradient or proximal oracles for $f$ and $h$ can be used to construct this separation oracle.
Combining Theorems~\ref{dolif:theo:outer} and \ref{dolif:theo:separation} yields an iteration complexity
\[
\widetilde{\cal O}\!\left(
q
\left[
\big(1+L_f^2+L_h^2\big)\|\cov{\tilde{\nu}}\|_{\mathrm{op}}+d
\right]
\log\frac{M}{\varepsilon}
\right)
\]
to reach $\varepsilon$-accuracy in R\'enyi divergence $\mathcal{R}_q$ for $q\geq1$. The iteration complexity in $\chi^2$ divergence is similar by removing $q$. Regarding oracle complexity, each iteration of Algorithm~\ref{dolif:alg:ASF_pi} makes 
$\widetilde{\mathcal{O}} (d)$ expected calls to the separation oracle for $\Q$. Finally, regarding proposal complexity, the average number of proposals is $\mathcal{O}(M)$ per proximal sampler iteration.


\textbf{Comparison to existing works.}  
Our work fits into a growing literature on the interplay between sampling and optimization. One line of works including~\cite{durmus2019analysis,jordan1998variational,wibisono2018sampling} and others views sampling dynamics through an optimization lens in the space of probability measures. Another line of works designs sampling algorithms by importing tools from convex optimization. In this second line, constrained samplers typically enforce feasibility using projections~\cite{brosse2017sampling,Bubeck2018}, reflections~\cite{du2025non,wang2025accelerating}, barrier functions~\cite{kook2022sampling,kook2024gaussian}, or mirror maps~\cite{ahn2021efficient,hsieh2018mirrored,srinivasan2024fast,zhang2020wasserstein}. Our work covers both of these developments.


Paper \cite{kook2025algodiffusion} also studies log-concave sampling via lifting, closely related in spirit to our approach, but assumes a ground-set condition on the potential $V$. 
In contrast, we work under (A1)-(A2) for constrained sampling (Section~\ref{sec:constrained:mainpaper}) and (B1)-(B2) for composite sampling (Section~\ref{sec:composite:mainpaper}). 
These assumptions are incomparable: the ground-set condition neither implies nor is implied by ours, and it excludes some convex Lipschitz potentials covered here. As an example, fix  a dimension $d\ge 1$ and take $K=2\ball_{d}(0)\subseteq\R^d$, where $\ball_{d}(0)$ is the unit ball in $\R^d$. Consider the
Laplace-type potential $f(x)=\beta |x_1|$
on $K$, where $\beta>10d$. Then $V=f$ is
convex and Lipschitz on $K$ and (A1)-(A2) holds. Moreover, we have
$\min_{x\in K}V(x)=0$. In the notation of
\cite{kook2025algodiffusion}, the ground set of $V$ is
$\mathsf L_{V,g}
=
\{x:\ V(x)-\min_{x\in K}V(x)\le 10d \}
= \{x\in K:\ |x_1|\le 10d/\beta\}$. 
Since $\beta>10d$, this violates the ground-set assumption in
\cite{kook2025algodiffusion} which requires
$B(0,1)\subset \mathsf L_{V,g}$. We have a similar example in the composite setting by taking
$h\equiv 0$ and $f(x)=\beta\|x\|_1$ on $\R^d$, where $\beta>10d$. Another difference in \cite{kook2025algodiffusion} is that they fix the lifting parameter $a=d$. In contrast, we choose $a=(LR_\mathrm{out}+d)/R_\mathrm{in}$ for constrained sampling and choose $a,b$ as in~\eqref{dolif:eq:chooseabeta} for composite sampling. Our choices can lead to a sharper iteration complexity. Indeed, for constrained sampling and in the case $L=\mathcal O(\sqrt d)$ and $R_\mathrm{out}\asymp R_\mathrm{in}\gtrsim d$, our iteration complexity in Renyi divergence is $\mathcal O\left(d\log d\left(\|\cov{\nu}\|_{\mathrm{op}}+1\right)\right)$ per \eqref{intro:silif:itercomplex}, while \cite[Theorem~2.15]{kook2025algodiffusion} and the step size of order $\mathcal O\left((d^2\log d)^{-1}\right)$ therein lead to an iteration complexity of $\mathcal O\left(d^2\log d\left(\|\cov{\nu}\|_{\mathrm{op}}+1\right)\right)$.

\underline{Constrained sampling:} uniform sampling on a convex set is classical as it is close to efficient volume computation. Notable works on this topic include~\cite{kook2026inandout,diaconis2012gibbs,narayanan2022mixing,laddha2023convergence,lovasz1999hit,cheneldanhitandrun,chen2018fast,kannan2009random}, among others. 

For sampling of general constrained targets on $K$, \cite{brosse2017sampling,lamperski2021projected,durmus2018efficient,Bubeck2018} study Langevin Monte Carlo (LMC)-type algorithms using Euclidean projections onto $K$. \cite{du2025non,wang2025accelerating} use skew-projections built from Euclidean projections, with \cite{du2025non} including \cite{Bubeck2018} as a special case. \cite{ahn2021efficient,srinivasan2024fast,hsieh2018mirrored,zhang2020wasserstein,li2022mirror,jiang2021mirror} incorporate self-concordant mirror maps into LMC. Barrier methods are combined with Gaussian cooling in \cite{kook2024gaussian,cousins2018gaussian} and Riemannian HMC in \cite{kook2022sampling}.

\underline{Composite sampling:} \cite{mou2022composite,liu2026proximal} study Metropolis-Hastings algorithms with proximal proposals, while \cite{salim2019proximal,lau2022bregman,pereyra2016proximal,ehrhardt2024proximal} combine Langevin updates with proximal mappings. For targets $\pi(x)\propto e^{-f(x)-h(x)}$, \cite{mou2022composite} assume smooth $f$, convex Lipschitz $h$, a log-Sobolev inequality, and additional dissipativity conditions. Their method requires sampling from
$Z(u)^{-1}\exp(-h(y)-\|y-u\|^2/(4\eta))$
and evaluating $Z(u)$, both potentially difficult outside of tractable cases such as separable penalties, $\ell_1$ priors, and group Lasso. Under strong convexity, bounded condition number, and an $\mathcal O(\sqrt d)$ Lipschitz constant for $h$, their total variation iteration complexity is $\mathcal O(d\log d \varepsilon^{-1}+\log^2\varepsilon^{-1})$. Assuming smooth $f$ and strongly convex $f+h$, \cite{liu2026proximal} improves  on the earlier work \cite{mou2022composite} by  removing the normalizer requirement while retaining the proximal sampling oracle assumption, achieving $\widetilde{\mathcal O}(\kappa\sqrt d\log^4\varepsilon^{-1})$ iterations in total variation. They also consider extensions to the case of Lipschitz-continuous $f$.

 
 We observe that both \cite{mou2022composite,liu2026proximal} assume exact proximal sampling oracle of $h$. Meanwhile, in Section \ref{sec:composite:mainpaper} and for the $h$ component, we only assume access to its proximal map. One can employ $\prox_h$ inside a rejection sampler to build a proximal sampling oracle of $h$, but this would further contribute to the algorithm complexity and is something \cite{mou2022composite,liu2026proximal} do not take into account. Regarding the assumptions on $f$ and $h$, we assume in Section \ref{sec:composite:mainpaper} that $f$ and $h$ are convex and Lipschitz. We do not require smoothness, strong convexity, log-Sobolev inequality or dissipativity assumptions; some of which appear in the results of \cite{mou2022composite,liu2026proximal} discussed above.  

Meanwile, distributed Gibbs samplers are studied in \cite{yuan2023class,vono2022admm,vono2019admm,rendell2020global}. For targets proportional to $\exp(-\sum_{i=1}^n f_i(x)-\sum_{j=1}^m g_j(x))$, \cite{yuan2023class} use the augmented density $\exp\big(-\sum_{i=1}^n f_i(x_i)-\sum_{j=1}^m g_j(y_j)-\sum_{i=1}^n\sum_{j=1}^m\frac{\sigma_{ij}}{2\eta}\|x_i-y_j\|^2\big)$
which approximates the original target for small $\eta$. Their simplest setting assumes smooth, strongly convex potentials, whereas our Section~\ref{sec:composite:mainpaper} assumes closed, convex, Lipschitz continuous $f$ and $h$. Finally, \cite{eftekhari2023forward,ghaderi2024smoothing} study Bregman forward-backward envelope smoothing schemes  and~\cite{habring2024subgradient} studies subgradient-based Langevin schemes.

\textbf{Organization.}
Section~\ref{sec:prelim} introduces notation. Section~\ref{sec:outer} develops the variational analysis of the proximal sampler. Sections~\ref{sec:constrained:mainpaper} and~\ref{sec:composite:mainpaper} give the corresponding algorithms and analyses for constrained and composite sampling. Section~\ref{sec:conclusion} presents concluding remarks.
Appendix~\ref{appen:general} collects technical results.
Appendices~\ref{appen:outer}, \ref{appen:constrained}, and
\ref{appen:composite} provide supporting results and proofs for
Sections~\ref{sec:outer}, \ref{sec:constrained:mainpaper}, and
\ref{sec:composite:mainpaper}, respectively.

\section{Notation and Definitions}
\label{sec:prelim}
Throughout the paper, we denote $\|\cdot\|$ the Euclidean norm, $\|\cdot\|_{\mathrm{op}}$ the operator norm, and $I_n$ the $n\times n$ identity matrix. Let $\ball_d(c)$ denote an $\ell_2$ unit ball in $\R^d$ centered at $c$. We write $[a]_+=\max\{a,0\}$. Let ${\mathcal{O}}(\cdot)$ denote the standard big-O notation.

\textbf{Absolute continuity.} For measures $\mu,\nu$ on $(E,\mathcal{F})$, we write $\mu\ll\nu$ if there exists $f:E\to\R$ such that $\mu(A)=\int_A f \D\nu$ for all $A\in\mathcal{F}$. The function $f$ is the Radon-Nikodym derivative, denoted $\frac{\D\mu}{\D\nu}$.

\textbf{Metric.} Let $\phi:\R_{\ge0}\to\R$ be convex with $\phi(1)=0$. For probability measures $\mu\ll\nu$ on $(E,\cal F)$, define the $\phi$-divergence
$D_\phi(\mu\|\nu)=\int_E \phi \left(\frac{\D\mu}{\D\nu}\right) \D\nu$.
This recovers, e.g., Kullback-Leibler divergence for $\phi(x)=x\log x$ and $\chi^q$ for $\phi(x)=x^q-1$ with $q\geq 1$.
The $q$-R\'enyi divergence is $\mathcal{R}_q(\mu\|\nu)=\frac{1}{q-1}\log \bigl(\chi^q(\mu\|\nu)+1\bigr)$, with the value at $q=1$ defined as the Kullback-Leibler divergence. 

\textbf{Entropy.} If $\mu$ is absolutely continuous with respect to
Lebesgue measure and has density, still denoted by $\mu$, then define the entropy $H(\mu)
:= -\int_E \mu(x)\log \mu(x) \D x$. 

\textbf{Normalizing constants.} 
For a measurable $\Theta:\R^d\to\R\cup\{+\infty\}$, define the normalizing constant $N_\Theta:=\int_{\R^d} e^{-\Theta(z)}\D z$ whenever this integral is finite.

\textbf{Poincar\'e inequality.} A distribution $\nu$ satisfies the Poincar\'e inequality (PI) with constant $C_{\mathrm{PI}}$ if $\Var_\nu(\psi)\le C_{\mathrm{PI}}\E_\nu[\|\nabla\psi\|^2]$ for all smooth bounded $\psi$. In Lemmas~\ref{silif:lem:PIconstant} and \ref{dolif:lem:PIconstant}, we bound PI constants of $\pi$ in \eqref{silif:def:pi} and of $\tilde{\pi}$ in \eqref{dolif:def:piandQ}, respectively. 

\textbf{Oracles.} Assume $S\subseteq \R^d$ and $f:\R^d\to\R$ are both closed and convex . A membership oracle for $S$ decides if $x\in S$. A separation oracle for $S$ either certifies $x\in S$ or returns $g(x)$ such that $\langle g(x),x-y\rangle >0$ for all $y\in S$ (hence it subsumes membership). A projection oracle returns $\proj_S(y)=\argmin\{\norm{x-y}^2: x\in S\}$; clearly $\proj_S(y)=y$ if $y\in S$.
An evaluation oracle returns $f(x)$, and a subgradient oracle returns some subgradient of $f$. Finally, a proximal oracle for $f$ returns $\prox_{ f}(x):=\argmin_{y\in \R^d}\left\{f(y)+\frac{1}{2}\|y-x\|^2\right\}$ for a given $x$. 

\textbf{Volumes, distance to a set, and $\mbox{minwidth}$.} Let $\vol(S)$ and $\vol_{d-1}(\partial S)$ denote the volumes of $S\subseteq \R^d$ and the boundary set $\partial S \subseteq \R^{d-1}$, respectively. Let $\dist (y,S):=\inf_{z\in S}\|y-z\|$ denote the Euclidean distance from $y$ to $S$. Let $\mbox{minwidth}(S)=\min_{\norm{a}=1}\left\{\max_{y\in S}a^\top y-\min_{y\in S}a^\top y\right\}$.

\textbf{Inradius and outer radius.} Let $R_{\mathrm{in}}(S):=\sup\{r\geq 0:r\ball_d(0)\subseteq S\}$ and $R_{\mathrm{out}}(S):=\inf\{R\geq 0:S\subseteq R\ball_d(0)\}$ denote the inradius and outer radius of $S$, respectively.

\textbf{Parallel sets} For $r\ge 0$, define the parallel set $S_r := S + r \ball_d(0) = \{y\in \R^d:\dist (y,S)\le r\}$. 

\textbf{Warmness.} We say $\mu$ is $M$-warm with respect to\ $\nu$ if $\mu\ll\nu$ and $\D\mu/\D\nu\le M$ a.e.

\textbf{Narrow topology.}
A sequence $(\mu_n)_{n\geq1}\subset\mathcal P(\mathbb R^d)$ converges
narrowly to $\mu\in\mathcal P(\mathbb R^d)$ if $\int_{\mathbb R^d} f \D\mu_n
\longrightarrow
\int_{\mathbb R^d} f \D\mu$ for every bounded continuous $f$ on $\R^d$.  This is the same as saying for every bounded continuous $f$, the map
$s\mapsto\int f \D\rho_s$ is continuous on $[0,t]$.

\textbf{Function spaces and admissible class.}  Let $\mathbb D\subseteq \R^d$, then $C_c^\infty(\mathbb D)$ denotes the space of infinitely differentiable
real-valued functions with compact support in $\mathbb D$, while
$C_b(\mathbb D)$ denotes the space of bounded continuous functions on $\mathbb D$.

Moreover, $\mathcal A_t(\mu,\nu)$ is the class of admissible pairs
$(\rho,u)$ such that $\rho_0=\mu$, $\rho_t=\nu$, and
\begin{align}
\rho\in C_{\mathrm{narrow}}([0,t];\mathcal P_2(\mathbb R^d))&,
\qquad
u:[0,t]\times\mathbb R^d\to\mathbb R^d
\text{ is Borel}, \label{def:classAt}\\
\int_0^t\int_{\mathbb R^d}\|u_s(x)\|^2 \D\rho_s(x) \D s<\infty&,
\qquad
\partial_s\rho_s
=
-\nabla\cdot(\rho_su_s)
+
\frac12\Delta\rho_s.\label{eq:finiteactionandpde}
\end{align}
 The last identity holds in the sense of distributions. 


\textbf{Heat semigroups.} For $t>0$, let $g_t(x):=(2\pi t)^{-d/2}\exp(-\|x\|^2/(2t))$. 
For a probability measure $\mu$ on $\mathbb R^d$, define
$\mu P_t:=\mu*g_t$. For a measurable function $f$, define
$P_tf(x):=\int_{\mathbb R^d}f(y)g_t(x-y)\D y$ whenever the integral is
well-defined.

\textbf{First variation.} Let $\mathcal F:\mathcal P_2(\mathbb R^d)\to
\mathbb R$. A function $\frac{\delta\mathcal F}{\delta\rho}$ is called a
first variation of $\mathcal F$ at $\rho$ if, for every signed measure
$\sigma$ with $\sigma(\mathbb R^d)=0$, $\left.\frac{\D}{\D\varepsilon}\mathcal F(\rho+\varepsilon\sigma)
\right|_{\varepsilon=0}
=
\int_{\mathbb R^d}
\frac{\delta\mathcal F}{\delta\rho}(x) \D\sigma(x).$

\textbf{Wasserstein space.}
$\mathcal P_2(\mathbb R^d)$ denotes the set of probability measures on
$\mathbb R^d$ with finite second moment, equipped with the
$2$-Wasserstein distance $\cW_2$. A tangent vector at $\rho$ has the form
as $\xi=-\nabla\cdot(\rho v)$
where $v\in L^2(\rho;\mathbb R^d)$. For
$\xi_i=-\nabla\cdot(\rho v_i)$, define
$\langle \xi_1,\xi_2\rangle_\rho
:=
\int_{\mathbb R^d}\langle v_1,v_2\rangle\D\rho$.
If $\mathcal F$ admits a first variation,
its Wasserstein gradient at $\rho$ is $\mathrm{grad}_{\cW_2}\mathcal F(\rho)
=
-\nabla\cdot\left(
\rho\nabla\frac{\delta\mathcal F}{\delta\rho}
\right)$. For a functional $\mathcal F(\rho,\pi)$ of two measures,
$\mathrm{grad}_{\cW_2,i}\mathcal F$ denotes the Wasserstein gradient
with respect to the $i$th component.


\section{Variational Analysis of Proximal Sampler}
\label{sec:outer}

In this section, we analyze the proximal sampler for the target $\pi\propto \exp(-f)$ from the variational perspective \eqref{eq:proxsamplervariational}. 
The analysis is independent of the epigraph lifting used in Sections \ref{sec:constrained:mainpaper} and \ref{sec:composite:mainpaper}. The starting point is the dynamic variational characterizations in \eqref{eq:forwardprox} and \eqref{eq:backwardprox}, and our goal is to obtain their evolution equations (see Proposition \ref{prop:interpcontinuity}) and contraction (see Theorem \ref{theo:outer}) directly at the level of probability distributions.

The following theorem is the main result of this section. It establishes convergence
of the proximal sampler in $\chi^2$ and R\'{e}nyi divergences. The proof for $\chi^2$ divergence is presented at the end of Subsection~\ref{sec:debruijn}, while the proof for R\'{e}nyi divergence is similar and hence omitted.

\begin{theorem}
\label{theo:outer}
Let $\pi(\D x)=Z^{-1}e^{-f(x)} \D x$ be a probability measure on $\mathbb
R^d$, where $f$ may be nonsmooth or extended-valued, and assume that
$\pi$ satisfies the Poincar\'e inequality with constant
$C_{\mathrm{PI}}$. Let $m,n>1$ satisfy $1/m+1/n=1$. Let $\rho_0^X$ denote the initial law. Fix a step size $\eta>0$. Let $\rho_k^X,\rho_k^Y,\rho_{k+1}^X$ denote
one iteration of the proximal sampler satisfying the variational problems \eqref{eq:proxsamplervariational}.
\begin{enumerate}[label=(\alph*)]
\item Under the additional assumptions
\begin{equation}
\label{cond:chisquare:initialandtarget}
f_0^X:=\frac{\rho_0^X}{\pi}\in L^{2m}(\pi),
\qquad
\int_{\mathbb R^d}\|x\|^{2n}\pi(\D x)<\infty,
\end{equation}
we have, for every $k\ge0$,
\begin{equation}
\label{eq:chisquarecontractionconclu}
\chi^2(\rho_{k+1}^X\|\pi)
\le
\frac{1}{(1+\eta/C_{\mathrm{PI}})^2}
\chi^2(\rho_k^X\|\pi).
\end{equation}

\item Fix $q\ge2$. Under the additional assumptions
\begin{equation}
\label{cond:renyi:initialandtarget}
f_0^X:=\frac{\D\rho_0^X}{\D\pi}\in L^{qm}(\pi),
\qquad
\int_{\mathbb R^d}\|x\|^{2n}\pi(\D x)<\infty,
\end{equation}
we have, for every $k\ge0$,
\begin{equation}
\label{eq:renyicontractionconclu}
\mathcal R_q(\rho_{k+1}^X\|\pi)
\le
\begin{cases}
\displaystyle
\max\left\{
\mathcal R_q(\rho_k^X\|\pi)
-\frac{2}{q}\log\left(1+\frac{\eta}{C_{\mathrm{PI}}}\right),1
\right\},
&\text{if } \mathcal R_q(\rho_k^X\|\pi)>1,\\[2ex]
\displaystyle
\frac{\mathcal R_q(\rho_k^X\|\pi)}
{(1+\eta/C_{\mathrm{PI}})^{2/q}},
&\text{if } \mathcal R_q(\rho_k^X\|\pi)\le1.
\end{cases}
\end{equation}
\end{enumerate}

\end{theorem}

\begin{remark}
\label{rem:renyiqlessthan2}
The $q$-R\'{e}nyi divergence bound for $q\geq 2$ in \eqref{eq:renyicontractionconclu} extends to the case $1\leq q'\leq2$ since $\mathcal R_{q'}(\rho_k^X\|\pi)\leq\mathcal R_2(\rho_k^X\|\pi)$,  where $\mathcal R_1$ is understood as the Kullback-Leibler divergence. See for example \cite[Theorem~3]{vanerven2014renyi} for the monotonicity property of R\'{e}nyi divergence.
\end{remark}

\subsection{Continuity equations via optimal control}



As explained in Section \ref{sec:intro}, the first result of this subsection provides a dynamic presentation of the entropy-regularized Wasserstein distance in
\eqref{eq:EW2}, in analogy with the
Benamou-Brenier dynamic formulation of the Wasserstein distance. The proof connects the entropy-regularized $\cW_2$ first to a static Schr\"odinger bridge problem, then to a dynamic Schr\"odinger bridge problem, and finally to a stochastic optimal control problem. The
proof follows the idea in \cite[Section 4]{chen2021stochastic}, while we make the
statement more precise by introducing the admissible class
$\mathcal A_t(\mu,\nu)$ in \eqref{def:classAt}.

\begin{proposition}
\label{prop:regularizedW2dynamic}
Let $t>0$, and let $\mu,\nu$ be probability densities on $\mathbb R^d$
with finite second moments. Recall the entropy-regularized $\cW_2$ in \eqref{eq:EW2} and the admissible class $\mathcal A_t$ in \eqref{def:classAt}. 
Then, we have
\begin{equation}
\label{eq:regularizedW2dynamic}
\frac{1}{2t}\cW_{2,2t}^2(\mu,\nu)=\inf_{(\rho,u)\in\mathcal A_t(\mu,\nu)}
\frac12\int_0^t\int_{\mathbb R^d}\|u_s(x)\|^2 \D\rho_s(x) \D s
-H(\mu)
-
\frac d2\log(2\pi t).
\end{equation}
\end{proposition}

\begin{proof}

\underline{Step 1:} This step turns the entropy-regularized $\cW_2$ into a static Schr\"{o}dinger bridge problem.  Let $g_t(z):=(2\pi t)^{-d/2}\exp(-\|z\|^2/(2t))$.
Then the endpoint measure of the standard Brownian motion with initial
density $\mu$ is
\begin{equation*}
R^\mu_{0,t}(\D x\D y)
=
\mu(x)g_t(y-x) \D x\D y.
\end{equation*}
For any $\gamma\in\Pi(\mu,\nu)$,
\begin{align*}
\mathrm{KL}\bigl(\gamma \| R^\mu_{0,t}\bigr)
&=
\int \gamma(x,y)
\log\frac{\gamma(x,y)}{\mu(x)g_t(y-x)}
\D x\D y \notag \\
&=
\int \gamma\log\gamma \D x\D y
-
\int \gamma\log\mu(x) \D x\D y
-
\int \gamma\log g_t(y-x) \D x\D y \\
&=
\int \gamma\log\gamma \D x\D y
+H(\mu)
+
\frac{1}{2t}\int \|x-y\|^2 \D\gamma
+
\frac d2\log(2\pi t)
\end{align*}
which leads to
\begin{equation}\label{eq:staticSBandregularizedW2}
\inf_{\gamma\in\Pi(\mu,\nu)}
\mathrm{KL}\bigl(\gamma \| R^\mu_{0,t}\bigr)
=
\frac{1}{2t}\cW_{2,2t}^2(\mu,\nu)
+H(\mu)
+
\frac d2\log(2\pi t).
\end{equation}
\smallskip
\underline{Step 2:} We recall the well-known connection between the static and dynamic Schr\"{o}ndinger bridge problem in this step.  Let $R^\mu$ be Brownian path measure on $C([0,t];\mathbb R^d)$ with
initial law $\mu$, and let $P$ be a path law with initial law $\mu$.
Let $P_{0,t}(\D x\D y)$ be the joint law of the initial and terminal points
under
$P$. Similarly, $R^\mu_{0,t}(\D x\D y)
=
\mu(x)g_t(y-x) \D x\D y$
denotes the joint law of the initial and terminal points under $R^\mu$. Then we
have the decomposition
\[
P(\D\omega)=P^{xy}(\D\omega) P_{0,t}(\D x\D y),
\qquad
R^\mu(\D\omega)=R^{xy}(\D\omega) R^\mu_{0,t}(\D x\D y),\]
where $P^{xy}(\D\omega)$ is the conditional path law of $P$ given initial
point $x$ and terminal point $y$, and $R^{xy}(\D\omega)$ is the Brownian
bridge from $x$ to $y$ over time $t$. The chain rule for KL
divergence says
\begin{equation}\label{eq:entropychainrule}
\mathrm{KL}(P\|R^\mu)
=
\mathrm{KL}(P_{0,t}\|R^\mu_{0,t})
+
\int \mathrm{KL}(P^{xy}\|R^{xy}) P_{0,t}(\D x\D y).
\end{equation}
The second term in \eqref{eq:entropychainrule} is nonnegative and equal
$0$ when $P^{xy}=R^{xy}$.
Hence, we deduce that $\inf_{P: P_0=\mu,\ P_t=\nu}\mathrm{KL}(P\|R^\mu)
=
\inf_{\gamma\in\Pi(\mu,\nu)}
\mathrm{KL}(\gamma\|R^\mu_{0,t})$. 
Combining this with
\eqref{eq:staticSBandregularizedW2}
yields
\begin{equation}\label{eq:dynamicSBandregularizedW2}
\inf_{P: P_0=\mu,\ P_t=\nu}\mathrm{KL}(P\|R^\mu)
=
\frac{1}{2t}\cW_{2,2t}^2(\mu,\nu)
+H(\mu)
+
\frac d2\log(2\pi t).
\end{equation}
\smallskip
\underline{Step 3:} In this step, we recall the fact that a dynamic Schr\"{o}ngdinger bridge problem is equivalent to an optimal control problem with an SDE constraint. 
Let $\mathcal U_t(\mu,\nu)$ be the class of Borel controls $u$ for which
there exists a weak solution of
\begin{equation}\label{eq:SDEregW2}
dX_s=u_s(X_s) \D s+dB_s,
\qquad
X_0\sim\mu,
\qquad
X_t\sim\nu,
\end{equation}
such that $\E\int_0^t \|u_s(X_s)\|^2 \D s<\infty$. Denote by $P^u$ the law of this weak
solution.
From Step 2, recall $R^\mu$ as the Brownian path measure with initial law $\mu$. 
Then Girsanov's theorem \cite[Theorem 8.6.4]{oksendalbook} and \eqref{eq:SDEregW2} imply
\begin{align*}
\log\frac{dP^u}{dR^\mu}
&=
\int_0^t u_s(X_s)\cdot dX_s
-
\frac12\int_0^t \|u_s(X_s)\|^2 \D s,
\qquad P^u\text{-a.s.} \\
&\stackrel{\eqref{eq:SDEregW2}}=
\int_0^t u_s(X_s)\cdot dB_s
+
\frac12\int_0^t \|u_s(X_s)\|^2 \D s,
\qquad P^u\text{-a.s.}
\end{align*}
Taking expectation under $P^u$, the stochastic integral has mean zero, so that $\mathrm{KL}(P^u\|R^\mu) =
\frac12 \E \int_0^t \|u_s(X_s)\|^2 \D s$. 
Therefore,
\begin{equation}\label{eq:pathtocontrol}
\inf_{P: P_0=\mu,\ P_t=\nu}\mathrm{KL}(P\|R^\mu)
=
\inf_{u\in\mathcal U_t(\mu,\nu)}
\frac12
\E
\int_0^t \|u_s(X_s)\|^2 \D s.
\end{equation}

\smallskip 
\underline{Step 4:} Here we convert the SDE constraint in $\mathcal U_t(\mu,\nu)$ to the  PDE constraint $A_t(\mu,\nu)$.
Let $u\in\mathcal U_t(\mu,\nu)$ and $\rho_s(x) \D x$ be the law of $X_s$
in \eqref{eq:SDEregW2}. We will verify \eqref{def:classAt} and \eqref{eq:finiteactionandpde}.  For every bounded continuous $\psi$, we can write $\int_{\mathbb R^d}\psi(x)\rho_s(x)\D x
=
\E[\psi(X_s)]$. The right-hand side is continuous in $s$ by dominated convergence and continuity of $X_s,s\geq 0$, which further implies $\rho\in C_{\mathrm{narrow}}$ in \eqref{def:classAt}. Moreover, $\E \left[\int_0^t \|u_s(X_s)\|^2 \D s\right]
=
\int_0^t\int_{\mathbb R^d}\|u_s(x)\|^2\rho_s(x) \D x \D s<\infty$, 
so that the inequality in \eqref{eq:finiteactionandpde} holds. Next, via an application of  It\^o's formula (see for instance \cite[Exercise 8.3]{oksendalbook}, one can show $\rho_s(x) $ is a weak solution to the PDE in \eqref{eq:finiteactionandpde}. Thus $(\rho,u)\in\mathcal A_t(\mu,\nu)$.

In the opposite direction, let $(\rho,u)\in\mathcal A_t(\mu,\nu)$. Then
$\int_0^t\int_{\mathbb R^d}\|u_s(x)\|^2\rho_s(x) \D x \D s<\infty$. By the Cauchy-Schwarz inequality,
\[
\int_0^t\int_{\mathbb R^d}\|u_s(x)\|\rho_s(x) \D x \D s
\le
t^{\frac{1}{2}}
\left(
\int_0^t\int_{\mathbb R^d}\|u_s(x)\|^2\rho_s(x) \D x \D s
\right)^{\frac{1}{2}}
<\infty.
\]
Thus \cite[Theorem 2.5]{trevisan2016well} applies to the Fokker-Planck
equation in the definition of $\mathcal A_t(\mu,\nu)$ with $ b_s(x)=u_s(x), a_s(x)=I$: there exists a path law $P$
with marginals $\operatorname{Law}_{P}(X_s)=\rho_s(x) \D x, 0\le s\le t$
such that, for every $f\in C_c^\infty(\mathbb R^d)$, $f(X_s)-f(X_0)
-
\int_0^s
\left(
u_r(X_r)\cdot\nabla f(X_r)
+
\frac12\Delta f(X_r)
\right) dr$
is a $P$-martingale. Therefore, setting
$B_s
=
X_s-X_0-\int_0^s u_r(X_r) dr,$
the process $B_s$ is a Brownian motion under $P$, and $X_s$ solves the SDE in the definition of $\mathcal U_t(\mu,\nu)$ with $B_s,s\geq 0$ as the driving Brownian motion.  Moreover, $X_0\sim\mu, X_t\sim\nu$ and $\E_P\int_0^t \|u_s(X_s)\|^2 \D s
=
\int_0^t\int_{\mathbb R^d}\|u_s(x)\|^2\rho_s(x) \D x \D s
<\infty$. 
Consequently,
\begin{equation}\label{ eq:controltoFP}
\inf_{u\in\mathcal U_t(\mu,\nu)}
\frac12\E_{P^u}\int_0^t \|u_s(X_s)\|^2 \D s
=
\inf_{(\rho,u)\in\mathcal A_t(\mu,\nu)}
\frac12\int_0^t\int_{\mathbb R^d}\|u_s(x)\|^2\rho_s(x) \D x \D s.
\end{equation}
Finally, combining \eqref{eq:staticSBandregularizedW2},
\eqref{eq:dynamicSBandregularizedW2}, \eqref{eq:pathtocontrol} and
\eqref{ eq:controltoFP} leads to the desired result
\eqref{eq:regularizedW2dynamic}.
\end{proof}

Let us now motivate the next technical result and provide an intuitive derivation based on duality. 
We first note that \eqref{eq:forwardprox} and \eqref{eq:backwardprox} are both special instances of the generic optimal control problem
\begin{equation}
\label{eq:optdyncontrol}
\inf_{\substack{
\sigma\in\mathcal P_2(\mathbb R^d),\\
(\rho,u)\in\mathcal A_t(\alpha,\sigma)
}}
\left\{
\int \Phi_t \D\sigma
+
\frac12\int_0^t\int \|u_s\|^2 \D\rho_s \D s
\right\}
\end{equation}
with $\Phi_s =0$ and $\Phi_s = \Phi_s^b$ as in \eqref{eq:backwardprox}, respectively.
Moreover, it is easy to verify that both $\Phi_s$ satisfy the Hamilton-Jacobi equation
\begin{equation}
    \label{eq:PhiHJchain}
\partial_s\Phi_s+\frac12\Delta\Phi_s
=
\frac12\|\nabla\Phi_s\|^2.
\end{equation}
As a result, to solve the original variational problems \eqref{eq:forwardprox} and \eqref{eq:backwardprox}, it is sufficient to solve \eqref{eq:optdyncontrol} and \eqref{eq:PhiHJchain}.
In the next, we will show the following variational problem 
\begin{equation}
\label{eq:dualproblem}
\sup_{\phi\in C^\infty([0,t]\times\mathbb R^d)}
\left\{
\int \phi_0 \D\alpha:
\partial_s\phi_s+\frac12\Delta\phi_s-\frac12\|\nabla\phi_s\|^2\ge0,
\quad 0\le s\le t,\quad
\phi_t\le\Phi_t
\right\}
\end{equation} 
is the dual of \eqref{eq:optdyncontrol}-\eqref{eq:PhiHJchain} and provide a heuristic derivation for the optimal pair of \eqref{eq:optdyncontrol}-\eqref{eq:PhiHJchain} via this duality.  Since admissible pairs satisfy the controlled Fokker-Planck equation \eqref{eq:finiteactionandpde}, differentiating under the integral sign implies $\frac{\D}{\D s}\int\phi_s\rho_s\D x=\int(\partial_s\phi_s)\rho_s\D x+\int\phi_s\left[-\nabla\cdot(\rho_su_s)+\frac12\Delta\rho_s\right]\D x$. Integrating by parts the last two terms that are $\int-\phi_s\nabla\cdot(\rho_su_s)\D x=\int \left(u_s\cdot\nabla\phi_s\right) \rho_s\D x$ and $\int\phi_s\Delta\rho_s\D x=\int\Delta\phi_s \rho_s\D x$ leads to the formal chain rule $\frac{\D}{\D s}\int \phi_s \rho_s \D x=\int\left(\partial_s \phi_s + u_s \cdot \nabla \phi_s + \frac12 \Delta\phi_s \right) \rho_s \D x$. Therefore, we can write
\begin{equation}
\label{eq:chainrulerem}
\int \phi_t\D \rho_t-\int \phi_0 \D \alpha
=\int_0^t \int\left(\partial_s \phi_s+u_s \cdot \nabla
\phi_s+\frac{1}{2} \Delta \phi_s\right)  \rho_s \D x \D s
\end{equation}
It follows from \eqref{eq:dualproblem} that for any admissible $(\rho,u)$, 
\begin{align*}
J(\rho,u) &= \int \Phi_t\rho_t \D x + \frac12\int_0^t\int \|u_s\|^2\rho_s
\D x \D s \stackrel{\eqref{eq:dualproblem}}\ge \int \phi_t\rho_t \D x + \frac12\int_0^t\int \|u_s\|^2\rho_s \D x \D s \\
&\stackrel{\eqref{eq:chainrulerem}}= \int \phi_0 \D\alpha + \int_0^t \int\left(\partial_s \phi_s+u_s \cdot
\nabla \phi_s+\frac{1}{2} \Delta \phi_s + \frac12 \|u_s\|^2\right) 
\rho_s \D x \D s \\
&= \int \phi_0 \D\alpha + \int_0^t \int\left(\partial_s
\phi_s+\frac{1}{2} \Delta \phi_s - \frac12 \|\nabla \phi_s\|^2 + \frac12
\|u_s + \nabla \phi_s\|^2\right) \rho_s \D x \D s \stackrel{\eqref{eq:dualproblem}}\ge \int \phi_0 \D\alpha.
\end{align*}
Thus, the weak duality of \eqref{eq:dualproblem} and \eqref{eq:optdyncontrol} holds. Moreover, strong duality is achieved by taking $\phi_s=\Phi_s, u_s=-\nabla\Phi_s$ and letting the density solve $\partial_s\rho_s
=
\nabla\cdot(\rho_s\nabla\Phi_s)
+
\frac12\Delta\rho_s,
\rho_0=\alpha$. Thus,  Proposition~\ref{prop:phichainrule} below can be viewed as a rigorous verification
of this heuristic duality argument to find the optimal pair of \eqref{eq:optdyncontrol}-\eqref{eq:PhiHJchain}. 


\medskip

\begin{proposition}
\label{prop:phichainrule}
Fix $t>0$ and let $\alpha,\sigma\in\mathcal P_2(\mathbb R^d)$. Let
$\Phi\in C^\infty([0,t]\times\mathbb R^d)$ solve \eqref{eq:PhiHJchain}.
Assume that there exists a constant $C>0$, independent of $s$, such that
\begin{equation}
\label{eq:Phigradgrowth}
|\Phi_s(x)|+\|\nabla \Phi_s(x)\|^2
\le
C(1+\|x\|^2),
\qquad
0\le s\le t,\ x\in\mathbb R^d .
\end{equation}
Suppose that there exists a unique narrowly continuous curve
$(\rho_s^\star)_{0\le s\le t}$ in $\mathcal P_2(\mathbb R^d)$ solving
\begin{equation}
\label{assum:existnarrowcurve}
\partial_s\rho_s^\star
=
\nabla\cdot(\rho_s^\star\nabla\Phi_s)
+
\frac12\Delta\rho_s^\star,
\qquad
\rho_0^\star=\alpha.
\end{equation}
Set $u_s^\star=-\nabla\Phi_s$. Then $(\rho^\star,u^\star)$ is the unique optimal pair for
\eqref{eq:optdyncontrol}-\eqref{eq:PhiHJchain}.
\end{proposition}
\begin{proof}

\underline{Step 1:} We will verify that
$(\rho^\star,u^\star)\in \mathcal A_t(\alpha,\rho_t^\star)$. Clearly, $(\rho^\star,u^\star)$ satisfies the
endpoint and equation requirements of $\mathcal
A_t(\alpha,\rho_t^\star)$, so it is sufficient to check the second moment
and that $\int_0^t\int \|u_s^\star\|^2 \D\rho_s^\star \D s<\infty$.
Since $\rho_0^\star=\alpha\in\mathcal P_2(\mathbb R^d)$ and
$u_s^\star=-\nabla\Phi_s$, we use the assumed quadratic growth bound to
get
\[
\|u_s^\star(x)\|^2
=
\|\nabla\Phi_s(x)\|^2
\le
C(1+\|x\|^2).
\]
Let $\psi_R\in C_b^\infty(\mathbb R^d)$ be a nonnegative approximation
of $\|x\|^2$ such that $\psi_R\uparrow \|x\|^2$, $\psi_R=\|x\|^2$ on $B_R$,
$\psi_R$ is constant outside $B_{2R}$, and $\|\nabla\psi_R\|^2\le C(1+\psi_R)$,
$|\Delta\psi_R|\le C$,
$\|x\|^2\le C(1+\psi_R(x))$ on $\operatorname{supp}\nabla\psi_R$. 

An explicit construction of $\psi_R$ can be done as in~\cite[Section 13.1]{tu2011manifolds}.  Let 
\[
f(t):=
\begin{cases}
e^{-1/t},& t>0\\
0,& t\le0
\end{cases},
\qquad
g(t):=\frac{f(t)}{f(t)+f(1-t)}.
\]
Then $g$ is a smooth, non-decreasing step
function with $g=0$ on $(-\infty,0]$ and $g=1$ on $[1,\infty)$. Define for $r\in[0,\infty)$,
\[
\chi(r):=1-g\!\left(\frac{r-1}{3}\right), \qquad \phi(r):=\int_0^r \chi(s) \D s.
\]
Then $\chi\in C^\infty([0,\infty))$, $0\le\chi\le1$, $\chi=1$ on
$[0,1]$, $\chi=0$ on $[4,\infty)$, and $\chi$ is non-increasing.
This implies $\phi(r)=r$ on $[0,1]$, $\phi$ is constant on
$[4,\infty)$, $0\le\phi'\le1$, and $r|\phi'(r)|^2=r\chi(r)^2\le r\chi(r)\le \int_0^r\chi(s) \D s=\phi(r)$. Moreover $|\phi'(r)|+|r\phi''(r)|\le C$ for all $r$, since $\phi'=\chi$ and
$\phi''=\chi'$ is smooth and supported in $[1,4]$. Now for $R\ge1$, we can set
\[
\psi_R(x):=R^2\phi\!\left(\frac{\|x\|^2}{R^2}\right).
\]
It is straightforward to verify all the requirements for $\psi_R$ due to $
\nabla\psi_R(x)=2x\phi'(r)$ and $
\Delta\psi_R(x)=2d\phi'(r)+4r\phi''(r)$ for $r=\|x\|^2/R^2$. 

Testing the PDE \eqref{assum:existnarrowcurve} against $\psi_R$ (which is justified by compactly supported smooth
approximations), which is justified by compactly supported smooth
approximations, implies
\begin{equation}
\frac{\D}{\D s}\int \psi_R \D\rho_s^\star
=
-\int \nabla\psi_R\cdot\nabla\Phi_s \D\rho_s^\star
+
\frac12\int \Delta\psi_R \D\rho_s^\star.
\end{equation}
Using $\|\nabla\psi_R\|^2\le C\psi_R$, $|\Delta\psi_R|\le C$, 
\eqref{eq:Phigradgrowth} and the
Cauchy-Schwarz inequality, we obtain $\frac{\D}{\D s}\int \psi_R \D\rho_s^\star
\le
C\int \psi_R \D\rho_s^\star
+
C\int \|u_s\|^2\D\rho_s^\star
+
C$. 
Then via Gronwall's inequality, we get the bound $\sup_{0\le s\le t}\int \psi_R \D\rho_s^\star
\le
C_t\left[
1+\int \|x\|^2\D\alpha
+
\int_0^t\int \|u_s\|^2\D\rho_s^\star\D s
\right]$
where $C_t$ is independent of $R$. Letting $R\to\infty$ and using
monotone convergence yield 
\begin{equation}
\label{eq:finitecost}
    \sup_{0\le s\le t}\int \|x\|^2 \D\rho_s^\star<\infty. 
\end{equation}
In particular, $\rho_t^\star\in\mathcal P_2(\mathbb R^d)$. Moreover,
\[
\int_0^t\int \|u_s^\star\|^2 \D\rho_s^\star \D s
=
\int_0^t\int \|\nabla\Phi_s\|^2 \D\rho_s^\star \D s
\le
C\int_0^t\int (1+\|x\|^2)\D\rho_s^\star \D s<\infty,
\]
and we conclude $(\rho^\star,u^\star)\in\mathcal A_t(\alpha,\rho_t^\star)$.

\underline{Step 2:}
Fix $\sigma\in\mathcal P_2(\mathbb R^d)$ and
$(\rho,u)\in\mathcal A_t(\alpha,\sigma)$, then Lemma~\ref{lem:chainPhi} is applicable: adding
$\frac12\int_0^t\int \|u_s\|^2 \D\rho_s \D s$ to both sides of \eqref{eq:chainPhi} and completing the square yields
\begin{equation}
\label{eq:squarelower}
\int \Phi_t \D\sigma
+
\frac12\int_0^t\int \|u_s\|^2 \D\rho_s \D s
=
\int \Phi_0 \D\alpha
+
\frac12\int_0^t\int \|u_s+\nabla\Phi_s\|^2 \D\rho_s \D s \ge
\int \Phi_0 \D\alpha .
\end{equation}
We also know from the previous step that $ (\rho^\star,u^\star)$ belongs to $\mathcal
A_t(\alpha,\rho_t^\star)$.
Since $u^\star_s=-\nabla \Phi_s$, this combined with
\eqref{eq:squarelower}
imply $ (\rho^\star,u^\star)$ is an optimal pair for
\eqref{eq:optdyncontrol}-\eqref{eq:PhiHJchain}.

Now let us show uniqueness. Let $(\rho,u)$ be any
optimal
pair. Then equality must hold in \eqref{eq:squarelower}, and therefore $\int_0^t\int \|u_s+\nabla\Phi_s\|^2\D\rho_sds=0$. 
Thus, $u_s=-\nabla\Phi_s$ and the associated PDE in \eqref{assum:existnarrowcurve} serves an an optimal pair for the variational problem \eqref{eq:optdyncontrol}-\eqref{eq:PhiHJchain}. Combined with the assumed uniqueness of $(\rho^\star_s)_{\{s\geq 0\}}$ in \eqref{assum:existnarrowcurve}, we conclude the pair above is the unique optimal pair for \eqref{eq:optdyncontrol}-\eqref{eq:PhiHJchain}.  
\end{proof}

Next, we apply the generic optimal pair result in Proposition~\ref{prop:phichainrule} to the variational problems \eqref{eq:forwardprox} and \eqref{eq:backwardprox}.


\begin{proposition}
\label{prop:interpcontinuity}
The following statements hold.
\begin{enumerate}
\item[(a)] Let $0\leq t\le \eta$. The variational problem \eqref{eq:forwardprox} admits a unique
optimal solution, whose marginal laws $(\rho_t)_{0\le t\le\eta}$ form a
narrowly continuous curve and satisfy
\begin{equation}
\label{eq:cont:forward:actual}
\partial_t\rho_t
=
\frac12\Delta\rho_t
=
-\nabla\cdot(\rho_t v_t^\rho),
\qquad
v_t^\rho=-\frac12\nabla\log\rho_t,
\qquad
\rho_0=\rho_k^X,\quad \rho_\eta=\rho_k^Y. 
\end{equation}
The corresponding reference forward flow satisfies
\begin{equation}
    \label{eq:cont:forward:reference}
\partial_t\pi_t
=
\frac12\Delta\pi_t
=
-\nabla\cdot(\pi_t v_t^\pi),
\qquad
v_t^\pi=-\frac12\nabla\log\pi_t,
\qquad
\pi_0=\pi,\quad \pi_\eta=\pi P_\eta.
\end{equation}

\item[(b)] Let $0\leq t\le \eta$. The variational problem \eqref{eq:backwardprox} admits a unique
optimal solution, whose marginal laws $(\rho_t^{\mathrm b})_{0\le t\le\eta}$ form a
narrowly continuous curve and satisfy
\begin{equation}
\label{eq:cont:backward:actual}
\partial_t\rho_t^{\mathrm b}
=
-\nabla\cdot(\rho_t^{\mathrm b}v_t^{\rho,\mathrm b}),
\qquad
v_t^{\rho,\mathrm b}
=
\nabla\log\pi_t^{\mathrm b}
-\frac12\nabla\log\rho_t^{\mathrm b},
\quad
\rho_0^{\mathrm b}=\rho_k^Y,\quad \rho_\eta^{\mathrm b}=\rho_{k+1}^X,
\end{equation}
where $\pi_t^{\mathrm b}$ is reference backward flow and satisfies
\begin{equation}
\label{eq:cont:backward:reference}
\partial_t\pi_t^{\mathrm b}
=
-\nabla\cdot(\pi_t^{\mathrm b}v_t^{\pi,{\mathrm b}}),
\qquad
v_t^{\pi,{\mathrm b}}
=
\frac12\nabla\log\pi_t^{\mathrm b},
\quad
\pi_0^{\mathrm b}=\pi P_\eta,\quad
\pi_\eta^{\mathrm b}=\pi.
\end{equation}
\end{enumerate}
\end{proposition}

\begin{proof} 
We shall only prove (b), since the proof of (a) is similar (and simpler) and hence omitted. It follows from Proposition~\ref{prop:regularizedW2dynamic} that \eqref{eq:backwardprox} is in the form of \eqref{eq:optdyncontrol} with $\Phi_t = \Phi_t^b:=-\log\pi_t^{\mathrm b}$.

Fix $\varepsilon\in(0,\eta)$ and let $0<t\le \eta-\varepsilon$. Per Lemma~\ref{lem:Phibbounds},   all the requirements of
Proposition~\ref{prop:phichainrule} are satisfied, and we can apply Proposition~\ref{prop:phichainrule} with $\alpha=\rho_k^Y$, $\Phi=\Phi^{\mathrm b}$ to obtain the unique solution \eqref{eq:cont:backward:actual} on
$[0,\eta-\varepsilon]$ to the optimal control problem \eqref{eq:optdyncontrol} with $\Phi_t = \Phi_t^b$ and consequently to the variational problem \eqref{eq:backwardprox}. This solution has the form
\begin{equation}
\label{eq:truncatedpdecurve}
\partial_t\rho_t^{\mathrm b}
=
\nabla\cdot(\rho_t^{\mathrm b}\nabla \Phi_t^{\mathrm b})
+
\frac12\Delta\rho_t^{\mathrm b},\qquad 0\leq t\leq\eta-\varepsilon,
\end{equation}
which can be rewritten in the velocity form as stated in the proposition. For the endpoint $t=\eta$, \eqref{eq:backwardprox} and \eqref{eq:proxsamplervariational}  and $\Phi_\eta^{\mathrm b}:=-\log\pi_\eta^{\mathrm b}=f$ identify $\rho_\eta^{\mathrm b}=\rho_{k+1}^X $. This justifies \eqref{eq:cont:backward:actual}.

Next, since we have shown narrow continuity of the family $(\rho_t^{\mathrm b})_{0\le t< \eta}$,  if we can further verify 
\begin{equation}\label{eq:endpointconv}\rho_t^{\mathrm b}
\xrightarrow[t\uparrow\eta]{\mathrm{d}}
\rho_{k+1}^X,\end{equation}
then it will imply narrow continuity of $(\rho_t^{\mathrm b})_{0\le t\leq \eta}$. Define $h_k:=\frac{\rho_k^Y}{\pi P_\eta}$ then for $0\leq t<\eta$, the solution of \eqref{eq:truncatedpdecurve} has the representation $\rho_t^{\mathrm b}
=
\pi_t^{\mathrm b}P_t h_k$ (by plugging the formula into \eqref{eq:truncatedpdecurve} to check). 
Let $\varphi\in C_b(\mathbb R^d)$. By Fubini's theorem,
\begin{equation}
\label{eq:prelimit}
\int_{\mathbb R^d}\varphi(x) \rho_t^{\mathrm b}(\D x)
=
\int_{\mathbb R^d}
\frac{
\int_{\mathbb R^d}
\varphi(x)g_t(y-x) \pi_t^{\mathrm b}(\D x)
}{
\pi P_\eta(y)
}
 \rho_k^Y(\D y).
\end{equation}
Since $\pi_t^{\mathrm b}\stackrel{\mathrm{d}}\longrightarrow\pi$
and $g_t\to g_\eta$ uniformly on $\mathbb R^d$ as
$t\uparrow\eta$, we have  for
every $y\in\mathbb R^d$ that \begin{equation}\label{eq:limitmiddlestep}\int_{\mathbb R^d}
\varphi(x)g_t(y-x) \pi_t^{\mathrm b}(\D x)
\longrightarrow
\int_{\mathbb R^d}
\varphi(x)g_\eta(y-x) \pi(\D x)\end{equation} Moreover, the semigroup property and \eqref{eq:referencecurves} imply $\int_{\mathbb R^d}
g_t(y-x) \pi_t^{\mathrm b}(\D x)=
\pi_t^{\mathrm b}P_t(y)=\pi P_{\eta-t}P_t(y)=
\pi P_\eta(y)$
and hence $\left|
\frac{
\int_{\mathbb R^d}
\varphi(x)g_t(y-x) \pi_t^{\mathrm b}(\D x)
}{
\pi P_\eta(y)
}
\right|
\le \|\varphi\|_\infty$. 
Therefore, the dominated convergence theorem applied to \eqref{eq:prelimit} and $\Pi^{X\mid Y=y}(\D x)
=
\frac{g_\eta(y-x)\pi(\D x)}
{\pi P_\eta(y)}$  yield
\begin{equation*}
  \lim_{t\uparrow\eta}
\int_{\mathbb R^d}\varphi(x) \rho_t^{\mathrm b}(\D x)
\stackrel{\eqref{eq:prelimit},\eqref{eq:limitmiddlestep}}= \int_{\mathbb R^d} \frac{ \int_{\mathbb R^d}\varphi(x)g_\eta(y-x)\pi(\mathrm{d}x) }{ \pi P_\eta(y) } \rho_k^Y(\mathrm{d}y)
= \int_{\mathbb R^d} \left( \int_{\mathbb R^d} \varphi(x)\Pi^{X\mid Y}(\mathrm{d}x) \right) \rho_k^Y(\mathrm{d}y),
\end{equation*}
which becomes $\int_{\mathbb R^d}\varphi(x)\rho_{k+1}^X(\mathrm{d}x)$ in view of the fact that $\rho_{k+1}^X = \int_{\mathbb R^d} \Pi^{X\mid Y}\rho_k^Y \mathrm{d}y$.
Therefore, we prove \eqref{eq:endpointconv}.

Finally, applying Proposition~\ref{prop:phichainrule} in the same manner
with initial law $\pi_0^{\mathrm b}=\pi P_\eta$ and potential
$\Phi^{\mathrm b}=-\log\pi^{\mathrm b}$ yields \eqref{eq:cont:backward:reference}.  This completes the proof. 
\end{proof}

\subsection{De Bruijn identities via Otto calculus} \label{sec:debruijn}

The goal of this subsection is to first derive the de Bruijn identities, which characterize the rates of change of the $\chi^2$ and R\'enyi divergences, and then use these identities to establish the corresponding non-asymptotic convergence rates stated in Theorem~\ref{theo:outer}.

The following lemma establishes the de Bruijn identities.
They are in the same spirit as the simultaneous heat flow identities in \cite[Lemmas~12 and~15]{chen2022improved}, the formula in \cite[Proposition~1]{kook2026inandout}, and the generalized de Bruijn identity in \cite[Proposition~1.10]{klartagordentlich}. Here, we present an alternative proof based on a direct calculation using Otto calculus, avoiding complicated integration by parts arguments.

\begin{lemma}
\label{lem:debruijn}
 Fix $\varepsilon>0$. Consider the forward processes in \eqref{eq:cont:forward:actual}, 
\eqref{eq:cont:forward:reference} and backward processes in \eqref{eq:cont:backward:actual}, 
\eqref{eq:cont:backward:reference}. Then 
\begin{enumerate}[label=\textup{(\alph*)}]
\item under the additional assumption \eqref{cond:chisquare:initialandtarget}, it holds for respectively $t\in[\varepsilon,\eta]$ and $t\in(0,\eta-\varepsilon]$ that
\begin{equation}
\label{eq:debruijn:chisquare}
\frac{\D}{\D t} \chi^2\left(\rho_t \| \pi_t\right)
=
-\int \left\|\nabla \frac{\rho_t}{\pi_t} \right\|^2 \D \pi_t,  \quad  \frac{\D}{\D t}\chi^2(\rho_t^{\mathrm b}\|\pi_t^{\mathrm b})
=
-\int \left\| \nabla \frac{\rho_t^{\mathrm b}}{\pi_t^{\mathrm b}}
\right\|^2 \D \pi_t^{\mathrm b};
\end{equation}

\item under the additional assumption \eqref{cond:renyi:initialandtarget} for some $m,n>1$ such that $1/m+1/n=1$ and $q\geq 2$, it holds for respectively $t\in[\varepsilon,\eta]$ and $t\in(0,\eta-\varepsilon]$ that
\[
\frac{\D}{\D t}\mathcal R_q(\rho_t\|\pi_t)
=
-\frac{2}{q}
\frac{
\displaystyle \int
\left\|
\nabla\left(\frac{\rho_t}{\pi_t}\right)^{q/2}
\right\|^2
\D\pi_t
}{
\displaystyle \int
\left(\frac{\rho_t}{\pi_t}\right)^q
\D\pi_t
},
\quad\frac{\D}{\D t}
\mathcal R_q(\rho_t^{\mathrm b}\|\pi_t^{\mathrm b})
=
-\frac{2}{q}
\frac{
\displaystyle \int
\left\|
\nabla\left(
\frac{\rho_t^{\mathrm b}}{\pi_t^{\mathrm b}}
\right)^{q/2}
\right\|^2
\D\pi_t^{\mathrm b}
}{
\displaystyle \int
\left(
\frac{\rho_t^{\mathrm b}}{\pi_t^{\mathrm b}}
\right)^q
\D\pi_t^{\mathrm b}
}.
\]
\end{enumerate}
\end{lemma}

\begin{proof}
We will only derive the de Bruijn identities for the $\chi^2$ divergence in (a). The derivation in the R\'enyi divergence case in (b) is similar and hence omitted. 

 Let us derive the first equation in \eqref{eq:debruijn:chisquare}. In view of  Lemmas~\ref{lem:forward:chisquare} and \ref{lem:propagation-chisquare-regularity}, condition
\eqref{cond:chisquare:initialandtarget} ensures
for every $t>0$ and writing $f_t:=\rho_t/\pi_t$,
we have 
\begin{equation}
\label{eq:chisquare:forward:conditions}
f_t\in L^2(\pi_t), \qquad
\chi^2(\rho_t\|\pi_t)<\infty, \qquad \nabla f_t\in
L^2(\pi_t).
\end{equation}
With the technical guarantees in \eqref{eq:chisquare:forward:conditions}, the following de Bruijn calculation via Otto calculus is valid for
$t\in[\varepsilon,\eta]$. The first variations are $\frac{\delta}{\delta \mu}\chi^2(\mu\|\nu) = 2\frac{\mu}{\nu}$ and $\frac{\delta}{\delta \nu}\chi^2(\mu\|\nu) =
-\left(\frac{\mu}{\nu}\right)^2$, which imply the Wasserstein gradient computation
\begin{equation}\label{eq:wass-grad}
\mathrm{grad}_{\cW_2, 1} \chi^2\left(\rho_t \| \pi_t\right) = -\nabla
\cdot \left(\rho_t \nabla \left(2f_t\right)\right), \quad
\mathrm{grad}_{\cW_2,2} \chi^2\left(\rho_t \| \pi_t\right) = \nabla
\cdot \left(\pi_t \nabla f_t^2\right).  
\end{equation}
Viewing $\chi^2\left(\rho_t \| \pi_t\right)$ as a function of both
$\rho_t$ and $\pi_t$ and taking the time derivative, we obtain
\begin{align*}
\frac{\D}{\D t} \chi^2\left(\rho_t \|
\pi_t\right)&=\inner{\mathrm{grad}_{\cW_2, 1} \chi^2\left(\rho_t \|
\pi_t\right)}{ \partial_t\rho_t}_{\rho_t}+\inner{\mathrm{grad}_{\cW_2,
2} \chi^2\left(\rho_t \| \pi_t\right)}{\partial_t \pi_t}_{\pi_t} \nn\\
&\stackrel{\eqref{eq:wass-grad}}= \int \inner{\nabla \left(2f_t\right)}{v^\rho_t} \D \rho_t - \int
\inner{\nabla \left(f_t\right)^2}{v^\pi_t} \D \pi_t,
\end{align*}
where we also use $\partial_t\rho_t = -\nabla\cdot(\rho_t v_t^\rho)$ and $\partial_t\pi_t = -\nabla\cdot(\pi_t v_t^\pi)$ from \eqref{eq:cont:forward:actual} and \eqref{eq:cont:forward:reference}, respectively.
Finally, using the facts that $\D \rho_t = f_t\D\pi_t$ and $\nabla f_t^2 = 2f_t\nabla f_t$, \eqref{eq:cont:forward:actual}, and \eqref{eq:cont:forward:reference}, we obtain
\begin{equation}\label{eq:ddt-KL}
\frac{\D}{\D t} \chi^2\left(\rho_t \| \pi_t\right)
= 2\int f_t \inner{\nabla f_t }{v^\rho_t - v^\pi_t} \D \pi_t
= -\int f_t \inner{\nabla f_t }{\nabla \log f_t} \D \pi_t.
\end{equation}
Therefore, the first de Bruijn identity in \eqref{eq:debruijn:chisquare} follows immediately. 
The right-hand side of the first de Bruijn identity is well-defined in view of the last identity of \eqref{eq:chisquare:forward:conditions}.

 We derive the second equation in \eqref{eq:debruijn:chisquare} next. Per  Lemmas~\ref{lem:backward:chisquare} and~\ref{lem:propagation-chisquare-regularity}, condition
\eqref{cond:renyi:initialandtarget} ensures that for every $0<t\le\eta$ and $h_t=\rho_t^{\mathrm b}/\pi_t^{\mathrm b}$, 
we have $\chi^2(\rho_t^{\mathrm b}\|\pi_t^{\mathrm b})<\infty$ and $\nabla h_t\in L^2(\pi_t^{\mathrm b})$. 
Repeating the calculation leading to \eqref{eq:ddt-KL} with
$\rho_t^{\mathrm b}$ and $\pi_t^{\mathrm b}$, then using
\eqref{eq:cont:backward:actual} and
\eqref{eq:cont:backward:reference}, we obtain the second equation in \eqref{eq:debruijn:chisquare}, for which the right-hand side is finite as $\nabla h_t\in L^2(\pi_t^{\mathrm b})$. 
\end{proof}

Note that existing rigorous de Bruijn arguments, such as
\cite{klartagordentlich,kook2026inandout} for the heat flow, require both a finite initial
divergence condition and a fourth-moment condition on the target. We are able to relax the fourth-moment assumption at the expense of a stronger integrability assumption on the initial density ratio (see \eqref{cond:chisquare:initialandtarget} and \eqref{cond:renyi:initialandtarget}).


We are now ready to present the proof of Theorem~\ref{theo:outer}(a). The proof of (b) is similar and hence omitted.  

\textbf{Proof of Theorem~\ref{theo:outer}(a):}
We write the proof in two halves.

\textbf{The forward step:} 
Consider the interpolation processes $(\rho_t)_{0\le t\le\eta}$ and $(\pi_t)_{0\le t\le\eta}$ in \eqref{eq:cont:forward:actual} and \eqref{eq:cont:forward:reference}, respectively.
Let $\varepsilon\in(0,\eta)$ and $\displaystyle f_t=\rho_t/\pi_t$. 
Since $\pi$ satisfies PI with constant
$C_{\mathrm{PI}}$, the Gaussian convolution
$\pi_t=\pi * {\cal N}(0,tI)$ satisfies PI with constant $C_{\mathrm{PI}}+t$ per
\cite[Corollary 3.1]{Chafai}. Hence,
\[
\chi^2(\rho_t\|\pi_t)= \int (f_t-1)^2 \D\pi_t = \Var_{\pi_t}(f_t)
\le
(C_{\mathrm{PI}}+t)\int \|\nabla f_t\|^2 \D\pi_t .
\]
Combining the above PI with the first de Bruijn
identity in \eqref{eq:debruijn:chisquare} leads to
\begin{equation}\label{eq:chisquare:forward:debruijn-contraction}
\frac{\D}{\D t}\chi^2(\rho_t\|\pi_t)
\stackrel{\eqref{eq:debruijn:chisquare}}=
-\int \|\nabla f_t\|^2 \D\pi_t
\le
-\frac{1}{C_{\mathrm{PI}}+t}\chi^2(\rho_t\|\pi_t),
\qquad t\in[\varepsilon,\eta].
\end{equation}
Integrating \eqref{eq:chisquare:forward:debruijn-contraction} from
$\varepsilon$ to $\eta$, we obtain
\[
\chi^2(\rho_k^Y\|\pi_\eta)
= \chi^2(\rho_\eta\|\pi_\eta)
\le
\frac{C_{\mathrm{PI}}+\varepsilon}{C_{\mathrm{PI}}+\eta}
\chi^2(\rho_\varepsilon\|\pi_\varepsilon)
\stackrel{\mathrm{DPI}}{\le}
\frac{C_{\mathrm{PI}}+\varepsilon}{C_{\mathrm{PI}}+\eta}
\chi^2(\rho_k^X\|\pi),\]
where the last inequality is by the data processing inequality $\chi^2(\rho_\varepsilon\|\pi_\varepsilon)
\le
\chi^2(\rho_0\|\pi_0)$ as $(\rho_\varepsilon,\pi_\varepsilon)$ evolve from $(\rho_0,\pi_0)$ by following the same dynamics \eqref{eq:cont:forward:actual} and \eqref{eq:cont:forward:reference}.
Finally, letting $\varepsilon\downarrow0$, we obtain
\begin{equation}
\label{eq:chisquareproof:forwardconcl}
\chi^2(\rho_k^Y\|\pi_\eta)
\le
\frac{C_{\mathrm{PI}}}{C_{\mathrm{PI}}+\eta}
\chi^2(\rho_k^X\|\pi)
=
\frac{1}{1+\eta/C_{\mathrm{PI}}}\chi^2(\rho_k^X\|\pi).
\end{equation}
\textbf{The backward step:} 
Similarly, consider the interpolation processes $(\rho_t^b)_{0\le t\le\eta}$ and $(\pi_t^b)_{0\le t\le\eta}$ in \eqref{eq:cont:backward:actual} and \eqref{eq:cont:backward:reference}, respectively.
Let $\varepsilon\in(0,\eta)$ and
$h_t=\rho_t^{\mathrm b}/\pi_t^{\mathrm b}$. 
It follows from \cite[Corollary 3.1]{Chafai} that the PI constant of $\pi_t^{\mathrm b}=\pi* {\cal N}(0,(\eta-t) I)$ is $C_{\mathrm{PI}}+\eta-t$, which together with the second de Bruijn identity in \eqref{eq:debruijn:chisquare} implies that
\[
\frac{\D}{\D t}\chi^2(\rho_t^{\mathrm b}\|\pi_t^{\mathrm b})
=
-\int \|\nabla h_t\|^2 \D\pi_t^{\mathrm b}
\le
-\frac{1}{C_{\mathrm{PI}}+\eta-t}
\chi^2(\rho_t^{\mathrm b}\|\pi_t^{\mathrm b}), \qquad t\in(0,\eta-\varepsilon].
\]
Integrating from $0$ to $\eta-\varepsilon$, we obtain 
\[
\chi^2(\rho_\eta^{\mathrm b}\|\pi_\eta^{\mathrm b})\le
\chi^2(\rho_{\eta-\varepsilon}^{\mathrm b}
\|\pi_{\eta-\varepsilon}^{\mathrm b})
\le
\frac{C_{\mathrm{PI}}+\varepsilon}{C_{\mathrm{PI}}+\eta}
\chi^2(\rho_0^{\mathrm b}\|\pi_0^{\mathrm b}).
\]
The first inequality above is due to the data processing inequality. 
Since
$\rho_0^{\mathrm b}=\rho_k^Y$,
$\pi_0^{\mathrm b}=\pi_\eta$,
$\rho_\eta^{\mathrm b}=\rho_{k+1}^X$, and
$\pi_\eta^{\mathrm b}=\pi$, sending $\varepsilon\downarrow0$ yields
\begin{equation}
\label{eq:chisquareproof:backwardconcl}
\chi^2(\rho_{k+1}^X\|\pi)
\le
\frac{C_{\mathrm{PI}}}{C_{\mathrm{PI}}+\eta}
\chi^2(\rho_k^Y\|\pi_\eta)
=
\frac{1}{1+\eta/C_{\mathrm{PI}}}
\chi^2(\rho_k^Y\|\pi_\eta).
\end{equation}
Finally, combining \eqref{eq:chisquareproof:forwardconcl} and
\eqref{eq:chisquareproof:backwardconcl} yields \eqref{eq:chisquarecontractionconclu}. 
\QEDA

\section{Constrained Sampling}
\label{sec:constrained:mainpaper}
Recall from Section \ref{sec:intro} that, starting from the constrained target $\nu(x)$ as in \eqref{silif:def:nu}, we perform an epigraph lifting and obtain the lifted measure $\pi$ as in \eqref{silif:def:pi}.
If one can efficiently sample $(x,t)\sim\pi$ on the lifted set $Q$, then discarding the $t$-coordinate yields samples from the desired distribution $\nu$. In this way, the original general log-concave sampling on $K\subseteq \R^d$ reduces to sampling from a uniform distribution up to a one-dimensional exponential tilt on the lifted set $Q \subseteq \R^{d+1}$. We will employ the proximal sampler to develop an algorithm for sampling from $\pi$ in this section.

For the remainder of Section~\ref{sec:constrained:mainpaper}, we will make use of the following conditions:
\begin{itemize}
\item[(A1)] $K\subset \mathbb{R}^d$ is nonempty, closed, convex, and $R_\mathrm{in} \ball_d(0) \subseteq K\subseteq R_\mathrm{out} \ball_d(0)$;
\item[(A2)] $f$ is convex and Lipschitz continuous on $K$ with constant $L>0$;
\item[(A3)] there exists an initial distribution $\nu_0$ that is $M$-warm w.r.t.\ $\nu$;
\item[(A4)] there exists evaluation and subgradient oracles for $f$, and a separation oracle for $K$.
\end{itemize} 

We will use the notation $\norm{f}_\infty=\sup_{x\in K}|f(x)|<\infty$. Recall the notation $R_\mathrm{in}$ and $R_\mathrm{out}$ for the inradius and outradius of $K$ respectively. Furthermore, for $a>0$, set
\begin{equation} \label{silif:def:notation}
w =(x,t), \,\,
z=(y,s-a\eta),\quad
C_f(a)=\sqrt{a^2+L^2}, \,\,
C_K=\frac{LR_\mathrm{out}+d}{R_\mathrm{in}}, \,\, C_L(a)
=C_f(a)+C_K.
\end{equation}
 For a measurable set $S\in \R^{d+1}$ and $c>0$, define
\begin{equation}
\label{def:ZandB}
Z_{S}(c)=\int_{S}e^{-ct}\D w,
\qquad
B_{S}(c)=\int_{\partial S}e^{-ct}\D S,
\end{equation}
where $t$ is the last coordinate, $\D w$ is Lebesgue measure, and $\D S$ is the surface element on $\partial S$.

\subsection{Proximal sampler for the lifted distribution $\pi$}
\label{silif:sec:asf}
Per the idea of the proximal sampler, to sample $\pi(x,t)$ in \eqref{silif:def:pi}, we consider the augmented distribution
\begin{equation}
\label{silif:def:bigpi}
    \Pi((x,t),(y,s))\propto \exp\left(-I_Q(x,t) - at - \frac{1}{2\eta}\norm{(x,t)-(y,s)}^2 \right).
\end{equation}
For fixed $(y,s)\in \R^{d+1}$, the above potential (i.e., negative log-density) function becomes
\begin{equation}\label{silif:def:Theta}
\Theta^{\eta,Q}_{y,s}(x,t)
= I_Q(x,t)+at+\frac{\|(x,t)-(y,s)\|^2}{2\eta}
= I_Q(x,t)+\frac{\|(x,t)-(y,s-a\eta)\|^2}{2\eta}-\frac{a^2\eta}{2}+as.
\end{equation} 
Hence, $ \Pi^{Y,S|X,T}(y,s|x,t) = \mathcal{N}\brac{(x,t),\eta I_{d+1}}$ and 
$\Pi^{X,T|Y,S}(x,t|y,s) = \mathcal{N}\brac{(y,s-a\eta),\eta I_{d+1}}|_Q$, i.e., a Gaussian distribution restricted to $Q$. 
Thus, the proximal sampler for the target $\pi(x,t)$ in \eqref{silif:def:pi} is as follows.
\begin{algorithm}[H]
	\caption{Proximal sampler for the target $\pi$ in \eqref{silif:def:pi}}
	\label{silif:alg:ASF_pi}
	\begin{algorithmic}[1]
		\State  \textbf{Gaussian}: generate $(y_k,s_k)\sim  \mathcal{N}\brac{(x_k,t_k),\eta I_{d+1}}$;
		\State \textbf{RGO}: generate $(x_{k+1},t_{k+1})\sim {\cal N}\brac{(y_k,s_k-a\eta),\eta I_{d+1}}|_Q. $
	\end{algorithmic}
\end{algorithm}

Algorithm \ref{silif:alg:ASF_pi} describes the iterative steps starting from $k = 0$.
To initialize, we generate $x_0\sim \nu_0$, an $M$-warm start for $\nu$ in~\eqref{silif:def:nu}, which guarantees $x_0\in K$ almost surely by the definition of warmness (see Section ~\ref{sec:prelim}).
Under $\pi$, the $X$-marginal is $\nu$, and for $x\in K$,
\begin{equation}
\label{silif:eq:conditionalT}
\pi^{T\mid X=x}(t)
=
a e^{f(x)-at}\mathbf 1_{\{t\geq f(x)/a\}}
=
a e^{-(at-f(x))}\mathbf 1_{\{at-f(x)\geq 0\}}.
\end{equation}
Thus, $aT-f(x)\mid X=x\sim\operatorname{Exp}(1)$,
which further implies for a random variable $E\sim\operatorname{Exp}(1)$ independent of $X$, 
\begin{equation}
\label{eq:Tmarginal}
(X,T)\stackrel{\text{d}}{=}
\left(X,\frac{f(X)+E}{a}\right).
\end{equation}
Therefore, we initialize $t_0$ by drawing $e_0\sim\operatorname{Exp}(1)$ independent of $x_0$ and setting $t_0=(f(x_0)+e_0)/a$. Equation~\eqref{silif:eq:conditionalT} then yields $t_0\sim\pi^{T\mid X=x_0}$ and $(x_0,t_0)\sim\pi_0(x,t)=\nu_0(x)\pi^{T\mid X=x}(t)$.
Starting from $(x_0,t_0) \in Q$, Algorithm \ref{silif:alg:ASF_pi} keeps generating $(x_k,t_k)\sim \pi_k$ at Step 2 and maintains feasibility $(x_k,t_k) \in Q$. Moreover, $\pi_0$ is $M$-warm for $\pi$ in~\eqref{silif:def:pi} since 
\begin{equation}\label{silif:eq:warmstart}\frac{\D \pi_0}{\D\pi}(x,t)
=\frac{\frac{\D \nu_0}{\D x}(x)  \pi^{T\mid X=x}(t)}
{\frac{\D \nu}{\D x}(x)  \pi^{T\mid X=x}(t)}
=\frac{\D\nu_0}{\D\nu}(x)
\le M.\end{equation} Lemma~\ref{silif:lem:warmstart} further shows that $M$-warmness of $\pi_k$ with respect to $\pi$ holds for every $k\ge 1$.

Assuming the RGO (i.e., Step 2) is available, we obtain the iteration complexity of Algorithm~\ref{silif:alg:ASF_pi} in the next theorem. The proof follows directly from the one-step contraction in $\chi^2$ and R\'enyi divergences given by Theorem~\ref{theo:outer}.
We also apply Lemma~\ref{silif:lem:PIconstant}, which bounds the PI constant of $\pi$ for every $a>0$. The proof is deferred to Appendix~\ref{silif:appen:proofouter}

\begin{theorem}\label{silif:theo:outer}
Assume (A1)-(A3) hold and let $a>0$ in~\eqref{silif:def:pi}. Denote by $C_{\mathrm{PI}}$ the Poincar\'e constant of $\pi$. Then
$C_{\mathrm{PI}}=\mathcal{O}\brac{\log(d+1)\brac{\opnorm{\cov{\nu}}+\frac{d+1}{a^2}}}$.
Given $\varepsilon>0$, the iteration complexity in $\chi^2$ divergence to reach $\varepsilon$-accuracy is
\[\mathcal{O}\left(\left(1+\frac{C_{\mathrm{PI}}}{\eta}\right)\log\frac{M}{\varepsilon}\right).\]
Moreover, for $\mathcal{R}_q$ with $q\ge1$, the iteration complexity in $\mathcal{R}_q$ to reach $\varepsilon$-accuracy is
\[
{\mathcal{O}}\left(q\left(1+\frac{C_{\mathrm{PI}}}{\eta}\right) \log\frac{M}{\varepsilon}\right).
\]
\end{theorem}

\subsection{RGO implementation}\label{subsec:RGO:constrained}
The complexity bounds in Theorem~\ref{silif:theo:outer} quantify how many calls to the RGO are needed to reach a target accuracy under conditions~(A1)-(A3) and the idealization that the conditional update in Step~2 is an exact RGO.
To turn these iteration counts into end-to-end oracle complexities, we must consider how to implement this RGO. The goal of this subsection is to implement the RGO step in Algorithm~\ref{silif:alg:ASF_pi} via rejection sampling, assuming access to a subgradient oracle for $f$ and a separation oracle for $K$. We begin by explaining why these oracles are needed.

In rejection sampling, one typically draws samples from a proposal distribution that is easier to sample from than the target, while being close to the latter to keep the rejection rate low. 
In our setting, the target distribution in RGO is $\Pi^{X,T|Y,S}(x,t|y,s) = \mathcal{N} \left((y, s - a\eta),   \eta I_{d+1}\right)|_Q$, which is concentrated around
\begin{equation}
\label{silif:optimizingThetaasprojection}
\underset{(x,t) \in \R^{d+1}}{\argmin}   \Theta^{\eta,Q}_{y,s}(x,t)=\underset{w \in Q}{\argmin}  \big\|w - z\big\|^2 
= \proj_Q(z),
\end{equation}
where $w$ and $z$ are as in \eqref{silif:def:notation}.
This suggests the ideal proposal $\mathcal{N}\left(\proj_Q(z),\eta I_{d+1}\right)$ assuming we know the exact projection $\proj_Q(z)$. In our setting, we do not assume projection onto $Q \subset \R^{d+1}$, and we therefore aim to compute an approximate solution to \eqref{silif:optimizingThetaasprojection}. Via the cutting-plane method by~\cite{vaidya,lee2015fastercutting}, more specifically Lemma~\ref{silif:lem:equivopt}(c) and Theorem~\ref{silif:theo:cuttingplane}(a), we are able to construct an approximate solution $\tilde w$ satisfying 
\begin{equation}
\label{silif:difference:tildextildet:mainpaper}
    \norm{\tilde{w}- \proj_Q(z)}\leq \Delta:=\brac{1+\frac{L}{a}}\sqrt{\frac{2\eta}{d+1}}. 
\end{equation}
We set $\tilde{w}$ to be the center of the proposal in our rejection sampler for the RGO, i.e., the proposal density is proportional to $\exp\brac{-\mathcal{P}_1(w)}$ where  
\begin{equation}\label{silif:def:P1}
\mathcal{P}_1(w)
=\frac{\|w-\tilde{w}\|^2+\|\tilde{w}-z\|^2}{2\eta}
-\frac{\Delta}{\eta}
\left(\|w-\tilde{w}\|+\|\tilde{w}-z\|\right)
+as-\frac{a^2\eta}{2}.
\end{equation}
Notably, the cutting-plane method of~\cite{vaidya,lee2015fastercutting} needs (i) (obvious) gradient of $\|w-z\|^2$ in~\eqref{silif:optimizingThetaasprojection} and (ii) a separation oracle for $Q$, which we construct below from oracle access to $f$ and $K$.
\begin{lemma}\label{silif:lem:sep} (proof in Appendix~\ref{appen:oracles})
    Evaluation and subgradient oracles of $f$, and membership and separation oracles of $K$ give a valid separation oracle for $Q$ in \eqref{silif:def:pi}.
\end{lemma}

It is worth noting that the above claim also holds when a proximal oracle for $f$ is available (see Lemma~\ref{silif:lem:sep2}). However, for simplicity of exposition, we assume access to a subgradient oracle for $f$ throughout this section. Now we state the RGO implementation of Algorithm~\ref{silif:alg:ASF_pi}. Recall the notation $w=(x,t)$ and $z=(y,s-a\eta)$. 
\begin{algorithm}[H]
	\caption{RGO implementation  of Algorithm~\ref{silif:alg:ASF_pi}} 
	\label{silif:RGO}
	\begin{algorithmic}[1]
		\State  Compute an approximate solution $\tilde{w}$ to \eqref{silif:optimizingThetaasprojection} satisfying~\eqref{silif:difference:tildextildet:mainpaper} via the cutting-plane method by \cite{vaidya,lee2015fastercutting}. 
        \State  Generate $U\sim {\cal U}[0,1]$ and $w\sim \exp\brac{-\mathcal{P}_1(w)}$ via Algorithm \ref{silif:alg:samplingP1} in Appendix \ref{silif:appen:specialalgo}. 
		\State  If 
		\begin{equation}\label{silif:event:separation}
			U \leq  \exp\brac{-\Theta_{y,s}^{\eta,Q}(w)+ \mathcal{P}_1(w)},
		\end{equation}
		then accept $w$; otherwise, reject $w$ and go to step 2.
	\end{algorithmic}
\end{algorithm}

The following lemma shows that $\mathcal{P}_1\le \Theta_{y,s}^{\eta,Q}$ on $\R^{d+1}$, so the acceptance test~\eqref{silif:event:separation} is well-defined. \cite[Lemma B.3]{dangliang2025oracle} then implies Algorithm~\ref{silif:RGO} is unbiased, i.e., $w\sim \Pi^{X,T\mid Y,S}$. Any accepted sample is also feasible since $w\notin Q$ means $I_Q(w)=\infty$ and the RHS of~\eqref{silif:event:separation} equals $0$.

\begin{lemma}\label{silif:lem:compareThetaP1}  
Assuming~\eqref{silif:difference:tildextildet:mainpaper} holds, then for every $w\in \R^{d+1}$, $\mathcal{P}_1(w)\leq \Theta_{y,s}^{\eta,Q}(w)$.
\end{lemma}

\begin{proof}
Let $p_*=\proj_Q(z)$.
By
\eqref{silif:difference:tildextildet:mainpaper},
$\|\tilde w-p_*\|\leq\Delta$.
For any $u,v\in\R^{d+1}$ satisfying
$\|u-v\|\leq\Delta$, we have
\[
\|u\|^2
=\|v\|^2+2\langle v,u-v\rangle+\|u-v\|^2
\geq\|v\|^2-2\Delta\|v\|.
\]
Applying this inequality twice yields
\[
\|w-p_*\|^2
\geq\|w-\tilde w\|^2
-2\Delta\|w-\tilde w\|, \quad
\|p_*-z\|^2
\geq\|\tilde w-z\|^2
-2\Delta\|\tilde w-z\|.
\]
Combining the above two inequalities and applying the Pythagorean inequality for projection
onto the closed convex set $Q$ give
\begin{align*}
\Theta_{y,s}^{\eta,Q}(w)
&\geq
\frac{\|w-p_*\|^2+\|p_*-z\|^2}{2\eta}
+as-\frac{a^2\eta}{2}\\
&\geq
\frac{
\|w-\tilde w\|^2+\|\tilde w-z\|^2
-2\Delta\left(\|w-\tilde w\|+\|\tilde w-z\|\right)
}{2\eta}
+as-\frac{a^2\eta}{2} \stackrel{\eqref{silif:def:P1}}=\mathcal P_1(w).
\end{align*}
The first inequality is immediate when $w\notin Q$,
since then $\Theta_{y,s}^{\eta,Q}(w)=+\infty$.
\end{proof}

The following theorem about the RGO (Algorithm~\ref{silif:RGO}) is the main result of this section, with the proof in the upcoming subsection. 

\begin{theorem}
\label{silif:theo:separation}  
Assume conditions~(A1) through (A4) hold, and set
\begin{equation}\label{silif:chooseaandeta}
a=\frac{LR_\mathrm{out}+d}{R_\mathrm{in}},\qquad \eta=\frac{1}{a^2}. 
\end{equation}
\begin{itemize}
    \item[(a)] \textbf{(oracle complexity)} There are at most ${\mathcal{O}}\left(d\log\frac{d\gamma}{\alpha}\right)$ expected calls to the separation oracle of $K$ and to the subgradient oracle of $f$. Here $\gamma=R_\mathrm{out}/\mbox{minwidth}(K)$ and the parameter $\alpha$ is specified in~\eqref{silif:def:alpha}. We have
    \begin{equation}
    \label{silif:eq:concent}
    \Pr\left(
    \alpha\leq
    \min\left\{
    \frac{R_\mathrm{in}^2}
    {36(d+1)R_\mathrm{out}^2(LR_\mathrm{out}+d)^2},
    \frac14
    \right\}
    \right)
    \leq
    2\exp\left(-\frac d2\right).
    \end{equation}

    \item[(b)] \textbf{(proposal complexity)} The average number of proposals is $\mathcal{O}(M)$. 
\end{itemize}
\end{theorem}




\subsection{Proof of Theorem \ref{silif:theo:separation}}

In this subsection, for simplicity, we will use the notation $C_L=C_L(a)$ and $C_f=C_f(a)$ in \eqref{silif:def:notation} and choose $a$ at the end in the proof of Theorem~\ref{silif:theo:separation}.
The proposal complexity bound in Theorem~\ref{silif:theo:separation}(b)
is a central and technically challenging part of our analysis.
Its proof requires a careful study of the weighted geometry of
the parallel sets $Q_r$.
We develop the key estimates in Lemmas~\ref{lem:coarea}-\ref{silif:lem:keytechnical},
which form the technical foundation of the proposal complexity analysis.


A key ingredient in the proposal complexity analysis of
Theorem~\ref{silif:theo:separation} is a weighted analogue
of the classical volume-surface area relationship:
the derivative of the volume of a parallel set equals
its surface area. 

\begin{lemma} 
\label{lem:coarea}    
Let $Q\subset\R^{m}$ be a closed and convex set, let $c>0$,
$Z_{Q_r}(c):=\int_{Q_r} e^{-ct} \D y$, and 
$B_{Q_r}(c):=\int_{\partial Q_r} e^{-ct} \D S^{m-1}(y)$ where $t=t(y)$ is the last coordinate in $y$.
Suppose $Z_{Q_r}(c)<\infty$ for every $r\geq0$.
Then, $Z_{Q_r}(c)$ is absolutely continuous in $r$ and for a.e.\ $r>0$,
\begin{equation}\label{eq:IrderivativeofQr}
\frac{\D}{\D r}Z_{Q_r}(c)=B_{Q_r}(c).
\end{equation}
\end{lemma}

\begin{proof}
The co-area formula (e.g.,\ \cite[Thm.~3.2.12]{federer2014geometric}) says: if $\varphi:\R^{m}\to\R$ is Lipschitz and
$h\in L^{1}(\R^{m})$, then
\[
\int_{\R^{m}} h(y) |\nabla\varphi(y)| \D y
\;=\;
\int_{\R}\left(\int_{\{y:\varphi(y)= p\}} h(y) \D S^{m-1}(y)\right)\D p.
\]
For $y=(x,t)\in\R^{m-1}\times\R$, set $\varphi(y)=\dist (y,Q)$ and
$h(y)=\frac{e^{-ct}}{\|\nabla\varphi(y)\|}\mathbf 1_{\{y:\varphi(y)\le r\}}$.
Then, noting that $\{y:\dist (y,Q)\le r\}=Q_r$ and $\|\nabla \varphi(y)\| = \|\nabla \dist (y,Q)\|=1$ for $\mathcal{L}^{m}$ a.e.\ $y$, we have
\[
Z_{Q_r}(c)-Z_Q(c)
=
\int_{0}^{r}\left(\int_{\{y:\varphi(y)= p\}}
e^{-ct} \D S^{m-1}(y)\right)\D p.\]
Thus, $Z_{Q_r}(c)$ is absolutely continuous in $r$ and for a.e.\ $r>0$,
\[
\frac{d}{\D r}Z_{Q_r}(c)
=
\int_{\{y:\varphi(y)= r\}}
e^{-ct} \D S^{m-1}(y). 
\]
Moreover, we know $\partial Q_r=\{y:\dist (y,Q)= r\}$, so that~\eqref{eq:IrderivativeofQr} holds. 
\end{proof}

In the next result, we use Lemma \ref{lem:coarea} to perform integration by parts and bound a Gaussian-weighted integral of $B_{Q_r}$. 

\begin{lemma}
\label{silif:lem:envelope}
Assume conditions (A1) and (A2) hold. For every $\tau\geq 0$, we have 
\[
\int_{0}^{\infty} \exp \left(-\frac{(r-\tau)^2}{2\eta}\right)  
B_{Q_r}  \D r
\ \le\
 Z_Q C_L   \sqrt{2\pi\eta}\;
\exp \left(\tau C_L + \frac{\eta C_L^2}{2}\right)+Z_Q\left(e^{C_L\tau}-1\right).
\]
\end{lemma}
\begin{proof}
     The proof consists of two steps. In Step 1, we construct an envelope set which contains $Q_r$ and is much easier to integrate on compared to $Q_r$; while in Step 2, we perform integration by parts (IBP) to obtain the desired bound. The reason we need IBP to bound the quantity $\int_{0}^{\infty} \exp \left(-\frac{(r-\tau)^2}{2\eta}\right)  
B_{Q_r}  \D r$ is that $B_{Q_r}$ as an integral on the boundary $\partial Q_r$ is quite hard to deal with; so via IBP, $B_{Q_r}$ is replaced by $Z_{Q_r}$, which is easier to handle.  

\smallskip
\underline{Step 1:}
Since $f$ is only assumed on $K$, we extend it to $\R^d$ as follows.
Let $\proj_K:\R^d\to K$ be the Euclidean projection onto the closed convex set $K$,
and define $\bar f(x)=f(\proj_K(x)),x\in\R^d$. Then $\bar f=f$ on $K$ and since the projection map onto the closed convex set $K$ is $1$-Lipschitz, $\bar f$ is $L$-Lipschitz on $\R^d$: 
\begin{equation}\label{ineq:phi-Lip}
    \abs{\bar f(x)-\bar f(y)} = \abs{f(\proj_K(x))-f(\proj_K(y))}\leq  L\norm{\proj_K(x)-\proj_K(y)}\leq L\norm{x-y}. 
\end{equation}
Recall the definition of $Q$ in \eqref{silif:def:pi}, we write $Q=S_0\cap S_1 \subset \R^{d+1}$ where $S_0=K\times \R, S_1=\{(x,t)\in\R^{d+1}:\ t\ge \bar f(x)/a\}$. Fix $r>0$. Since $(A\cap B)+H\subseteq (A+H)\cap(B+H)$ for any sets $A,B,H$, applying this observation to $A=S_0,B=S_1$ and $H=r\ball_{d+1}(0)$, we get
\begin{equation}
\label{silif:setQr}
Q_r=(S_0\cap S_1)_r \subseteq (S_0)_r\cap (S_1)_r,
\end{equation}
where $(S_0)_r = K_r\times \R$ and $K_r=K+r\ball_d(0)\subset\R^d$.
Considering $\phi=\bar f/a$, we note that $\phi$ is $(L/a)$-Lipschitz continuous in view of \eqref{ineq:phi-Lip} and $S_1=\mathrm{epi}(\phi)$.
Hence, applying Lemma~\ref{lem:parallelsetcontainment} with $\phi=\bar f/a$, we obtain 
\[
(S_1)_r \subseteq \left\{(x,t):\ t\ge \frac{\bar f(x)}{a}-r\sqrt{1+\frac{L^2}{a^2}}\right\}=\left\{(x,t):\ t\ge \frac{\bar f(x)}{a}-\frac{r}{a}C_f\right\},
\]
where $C_f$ is as in \eqref{silif:def:notation}.
Therefore, via \eqref{silif:setQr}, we conclude 
\begin{equation}
\label{silif:envelopeinclusion}
Q_r\subseteq \mathcal E_r=\left\{(x,t):x\in K_r, t\ge \frac{\bar f(x)}{a}-\frac{r}{a}C_f \right\}.
\end{equation}
The definition of $Z_S$ for a set $S$ in \eqref{def:ZandB}, the inclusion in~\eqref{silif:envelopeinclusion}, and Lemma~\ref{lem:basicfacts}(b) (under the specification $F(x)=\mathbf{1}_{K_r}(x)$, $\psi(x)=\bar f(x)/a$, $\gamma=rC_f/a$) imply that
\begin{equation}\label{silif:firstbound_Zer}
    Z_{Q_r}\stackrel{\eqref{def:ZandB},\eqref{silif:envelopeinclusion} }{\le} Z_{\mathcal E_r}\stackrel{\eqref{def:ZandB}}{=}\int_{K_r}\int_{t\ge (\bar f(x)-rC_f)/a} e^{-a t}  \D t  \D x
\stackrel{\eqref{eq:intoverQ}}{=}\frac{e^{rC_f}}{a}\int_{K_r} e^{-\bar f(x)}  \D x. 
\end{equation}
By the definition of $R_\mathrm{in}$ and the assumption $0\in K$,
$R_\mathrm{in}\ball_d(0)\subseteq K$. Hence
$r\ball_d(0)\subseteq (r/R_\mathrm{in})K$, and therefore
\[
K_r=K+r\ball_d(0)
\subseteq \left(1+\frac{r}{R_\mathrm{in}}\right)K.
\]
Let $\lambda=1+r/R_\mathrm{in}$ and $y=x/\lam$. Then
\begin{equation}
\label{silif:firstboundintKr}
\int_{K_r} e^{-\bar f(x)}  \D x
\le \int_{\lambda K} e^{-\bar f(x)}  \D x
=\lambda^d\int_{K} e^{-\bar f(\lambda y)}  \D y.
\end{equation}
It follows from \eqref{ineq:phi-Lip} with $(x,y)=(\lam y, y)$ and $\|y\|\le R_\mathrm{out}$ for $y\in K$ that 
\[
\bar f(\lambda y) \ge \bar f(y)-\frac{Lr}{R_\mathrm{in}}\|y\|
\ge f(y)-\frac{LR_\mathrm{out}}{R_\mathrm{in}}r,
\]
where we also use the fact that $\bar f(y) = f(y)$ for $y \in K$.
Plugging the above inequality into \eqref{silif:firstboundintKr} and using
$\lambda^d=(1+r/R_\mathrm{in})^d\leq
\exp(dr/R_\mathrm{in})$, we have
\begin{equation}
\label{silif:secondboundintKr}
\int_{K_r} e^{-\bar f(x)}  \D x
\le \exp \left(\frac{LR_\mathrm{out}+d}{R_\mathrm{in}}r\right)
\int_{K} e^{-f(y)}  \D y
= e^{rC_K}\int_K e^{-f(y)}  \D y,
\end{equation}
where $C_K$ is as in \eqref{silif:def:notation}.
Using the definition of $Q$ in \eqref{silif:def:pi} and Lemma~\ref{lem:basicfacts}(b) with specification $F(x)=\mathbf{1}_{K}(x)$, $\psi(x)=f(x)/a$, and $\gamma=0$, we have
\[
\frac{1}{a} \int_K e^{-f(y)} \D y \stackrel{\eqref{eq:intoverQ}}= \int_K \int_{t\ge f(y)/a} e^{-at} \D t  \D y \stackrel{\eqref{silif:def:pi}}= \int_{Q} e^{-a t} \D w \stackrel{\eqref{def:ZandB}}= Z_Q,
\]
where the last identity is due to \eqref{def:ZandB}.
Therefore, plugging the above relation and \eqref{silif:secondboundintKr} into \eqref{silif:firstbound_Zer}, we obtain 
\begin{equation}
\label{silif:lemmafirststep}
Z_{Q_r} \stackrel{\eqref{silif:firstbound_Zer},\eqref{silif:secondboundintKr}}\le \frac{e^{r(C_f+C_K)}}{a} \int_K e^{-f(y)}  \D y
=e^{r(C_f+C_K)}  Z_Q
\stackrel{\eqref{silif:def:notation}}=e^{rC_L}  Z_Q,
\end{equation}
where $C_L$ is defined in \eqref{silif:def:notation}.

\underline{Step 2:} Set $w_{\eta,\tau}(r)=\exp \left(-(r-\tau)^2/(2\eta)\right)$.
Since $w_{\eta,\tau}(r)$ decays exponentially, the conclusion~\eqref{silif:lemmafirststep} of Step 1 implies that $Z_{Q_r}w_{\eta,\tau}(r)\to0$ as $r\to\infty$.
Applying Lemma~\ref{lem:coarea} with $c=a$ to get
$\frac{\D}{\D r}Z_{Q_r}=B_{Q_r}$ a.e.\ $r$.
Hence, integrating by parts yields
\begin{align}\label{silif:eq:IBP-master}
&\int_{0}^{\infty} B_{Q_r}   w_{\eta,\tau}(r)  \D r
\stackrel{\eqref{eq:IrderivativeofQr}}=\int_{0}^{\infty} Z_{Q_r}'   w_{\eta,\tau}(r)  \D r
= Z_{Q_r}w_{\eta,\tau}(r)\bigg|_{0}^{\infty}
-\int_{0}^{\infty} Z_{Q_r}   w_{\eta,\tau}'(r)  \D r\nonumber\\
\quad&= -Z_Q e^{-\tau^2/(2\eta)}
+ \int_{0}^{\infty} Z_{Q_r}  \frac{r-\tau}{\eta}  w_{\eta,\tau}(r)  \D r. 
\end{align}
We will bound the second term on the right-hand side of~\eqref{silif:eq:IBP-master} using the fact that $Z_Q\leq Z_{Q_r}$ and \eqref{silif:lemmafirststep},  
\begin{align}
    \int_{0}^{\infty} Z_{Q_r}  \frac{r-\tau}{\eta}  w_{\eta,\tau}(r)  \D r &= \int_0^\tau Z_{Q_r}  \frac{r-\tau}{\eta}  w_{\eta,\tau}(r)  \D r + \int_\tau^\infty Z_{Q_r}  \frac{r-\tau}{\eta}  w_{\eta,\tau}(r)  \D r \nn \\
    &\stackrel{\eqref{silif:lemmafirststep}}\leq Z_Q\int_0^\tau \frac{r-\tau}{\eta}    w_{\eta,\tau}(r)  \D r+ Z_Q\int_\tau^\infty \frac{r-\tau}{\eta}  e^{rC_L}  w_{\eta,\tau}(r)  \D r.\label{ineq:ZQ}
\end{align}
In view of the definition of $w_{\eta,\tau}(r)$, it is easy to compute
\[
\int_0^\tau \frac{r-\tau}{\eta}    w_{\eta,\tau}(r) \D r = - w_{\eta,\tau}(r) \bigg|_{0}^{\tau} 
    = - 1 + e^{-\tau^2/(2\eta)}.
\]
Combining the above observation, \eqref{silif:eq:IBP-master}, and \eqref{ineq:ZQ}, and using Lemma~\ref{lem:basicfacts} with $C=C_L$, we have
\begin{align}
     \int_{0}^{\infty} B_{Q_r}   w_{\eta,\tau}(r)  \D r & \stackrel{\eqref{silif:eq:IBP-master},\eqref{ineq:ZQ}}\leq -Z_Q + Z_Q\int_\tau^\infty \frac{r-\tau}{\eta}  e^{rC_L}  w_{\eta,\tau}(r)  \D r \nn \\
     &\stackrel{\eqref{eq:kindofibp}} \le Z_Q\left(e^{C_L\tau}-1\right)+C_L Z_Q\int_{\tau}^{\infty} e^{rC_L}w_{\eta,\tau}(r)  \D r. \label{finalboundintBQr}
\end{align}
Finally, using the definition of $w_{\eta,\tau}(r)$ and noting 
\[
C_L r-\frac{(r-\tau)^2}{2\eta}
=-\frac{1}{2\eta}\left(r-(\tau+\eta C_L)\right)^2+\tau C_L+\frac{\eta C_L^2}{2},
\]
we obtain 
\[
\int_{\tau}^{\infty} e^{rC_L}w_{\eta,\tau}(r)  \D r
\le e^{\tau C_L+\frac{\eta C_L^2}{2}}
\int_{-\infty}^{\infty}\exp \left(-\frac{(r-(\tau+\eta C_L))^2}{2\eta}\right)  \D r=\sqrt{2\pi\eta}\;\exp \left(\tau C_L+\frac{\eta C_L^2}{2}\right),
\] 
where the identity is due to \cite[Lemma B.1(a)]{dangliang2025oracle}. Therefore, the conclusion immediately follows from \eqref{finalboundintBQr}.
\end{proof}


We next combine Lemma~\ref{silif:lem:envelope} with the co-area formula
to bound an integral involving the distance
$\dist(w,Q)$. This estimate plays a key role in controlling the
expected number of rejection proposals in
Theorem~\ref{silif:theo:separation}.

\begin{lemma}
\label{silif:lem:keytechnical}
Assume conditions (A1) and (A2). For any $\tau>0$ and $w=(x,t)\in \R^{d+1}$,  we have 
\begin{equation}\label{ineq:lem5}
    \int_{\R^{d+1}} e^{-a t}  
\exp \left(-\frac{(\dist(w,Q)-\tau)^2}{2\eta}\right)  \D w
\ \le\
Z_Q \exp(\tau C_L) \left[
 C_L\sqrt{2\pi\eta}\exp \left(\frac{\eta C_L^2}{2}\right)+1\right].
\end{equation}
\end{lemma}
\begin{proof}
    Split the integral over $Q$ and $  Q^c$. For $w=(x,t) \in Q$, $\dist(w,Q)=0$, hence we have
\[
\int_Q e^{-a t}\exp \left(-\frac{(\dist(w,Q)-\tau)^2}{2\eta}\right)  \D w
=\exp \left(-\frac{\tau^2}{2\eta}\right)\int_Q e^{-a t}  \D w
\stackrel{\eqref{def:ZandB}}=Z_Q\exp \left(-\frac{\tau^2}{2\eta}\right),
\]
where the second relation is due to the definition of $Z_Q$ in \eqref{def:ZandB}.
On $Q^c$, applying the co-area formula, we obtain
\begin{align*}
&\int_{  Q^c} e^{-a t}\exp \left(-\frac{(\dist(w,Q)-\tau)^2}{2\eta}\right)  \D w
=\int_{0}^{\infty}\left(\int_{\{y:\dist(y,Q)=r\}}
e^{-a t}\exp \left(-\frac{(r-\tau)^2}{2\eta}\right)  \D S(y)\right)\D r\\
&\quad=\int_0^\infty \exp \left(-\frac{(r-\tau)^2}{2\eta}\right)  
\left(\int_{\{y:\dist (y,Q)=r\}} e^{-a t}  \D S(y)\right)\D r\stackrel{\eqref{def:ZandB}}=\int_0^\infty \exp \left(-\frac{(r-\tau)^2}{2\eta}\right)  B_{Q_r}  \D r,
\end{align*}
where the last identity is due to $\partial Q_r=\{y:\dist(y,Q)= r\}$ and the definition of $B_{Q_r}$ in \eqref{def:ZandB}.
Combining the above two relations yields 
\[ \int_{\R^{d+1}} e^{-a t}  
\exp \left(-\frac{(\dist(w,Q)-\tau)^2}{2\eta}\right)  \D w
=
Z_Q\exp \left(-\frac{\tau^2}{2\eta}\right)
+ \int_0^\infty \exp \left(-\frac{(r-\tau)^2}{2\eta}\right)  B_{Q_r}  \D r.\]
Then we can bound the integral on the right-hand side by Lemma~\ref{silif:lem:envelope} to get
\begin{align*}
& \int_{\R^{d+1}} e^{-a t}  
\exp \left(-\frac{(\dist(w,Q)-\tau)^2}{2\eta}\right)  \D w\\
&\leq  Z_Q\left[\exp \left(-\frac{\tau^2}{2\eta}\right)+C_L   \sqrt{2\pi\eta}
\exp \left(\tau C_L + \frac{\eta C_L^2}{2}\right)+e^{C_L\tau}-1\right]. 
\end{align*}
 Therefore, \eqref{ineq:lem5}  follows from the fact that $\exp(-\tau^2/(2\eta))\le 1$. 
\end{proof}

To apply the preceding bound in Lemma \ref{silif:lem:keytechnical} to the rejection sampler,
we introduce a simpler lower bound $\mathcal P_2$
for the proposal potential $\mathcal P_1$.

\begin{lemma}
\label{silif:lem:compareP1P2}
Recall $z=(y,s-a\eta)$ as in~\eqref{silif:def:notation}
and $\tilde w$ as in Algorithm~\ref{silif:RGO}.
For $w\in\R^{d+1}$, define
\begin{equation}\label{silif:def:P2}
\mathcal P_2(w)
=
\frac{1}{2\eta}\left[
\left(\|w-\tilde w\|-\Delta\right)^2
+\left(\operatorname{dist}(z,Q)-2\Delta\right)^2
-6\Delta^2
\right]
+as-\frac{a^2\eta}{2},
\end{equation}
where $\Delta$ in as in \eqref{silif:difference:tildextildet:mainpaper}.
Assume \eqref{silif:difference:tildextildet:mainpaper}
holds, then for every $w\in\R^{d+1}$, $\mathcal P_2(w)\leq\mathcal P_1(w)$.
\end{lemma}

\begin{proof}
Set $D=\dist(z,Q)$.
It follows from \eqref{silif:difference:tildextildet:mainpaper}
and the triangle inequality that
$|\|\tilde w-z\|-D|\leq\Delta$.
Hence, we have the following basic inequalities
\[
(\|\tilde w-z\|-\Delta)^2
\geq[\|\tilde w-z\|-\Delta]_+^2
\geq[D-2\Delta]_+^2
=(D-2\Delta)^2-[2\Delta-D]_+^2
\geq(D-2\Delta)^2-4\Delta^2,
\]
where the last inequality is due to $D\geq0$.
Completing the squares in~\eqref{silif:def:P1},
we obtain
\begin{align*}
\mathcal P_1(w)
&\stackrel{\eqref{silif:def:P1}}=
\frac{1}{2\eta}\left[
(\|w-\tilde w\|-\Delta)^2
+(\|\tilde w-z\|-\Delta)^2-2\Delta^2
\right]
+as-\frac{a^2\eta}{2}\\
&\geq
\frac{1}{2\eta}\left[
(\|w-\tilde w\|-\Delta)^2
+(D-2\Delta)^2-6\Delta^2
\right]
+as-\frac{a^2\eta}{2}\stackrel{\eqref{silif:def:P2}}=\mathcal P_2(w),
\end{align*}
which completes the proof.
\end{proof}


We are now ready to prove Theorem~\ref{silif:theo:separation}.

\textbf{Proof of Theorem~\ref{silif:theo:separation}:} The proof of Part a is quite technical and applies the cutting-plane method by~\cite{vaidya,lee2015fastercutting}.  Consequently, we review the cutting-plane method guarantee and present the proof of Part a together in Appendix~\ref{silif:appen:cp}. 

Regarding part b, denote $\mu_k$ the distribution of $(y_k,s_k)$ for the first step of Algorithm \ref{silif:alg:ASF_pi} and $n_{y,s}$ the expected number of proposals conditioned on the RGO input $(y_k,s_k)$. The fact that $d\mu_k/d\Pi^{Y,S}\leq M$ from Lemma~\ref{silif:lem:warmstart} implies 
\begin{equation}
\label{silif:sepawarmstartimplication}
	\E_{\mu_k}[n_{y,s}] \le M \E_{\Pi^{Y,S}}[n_{y,s}], 
\end{equation}
so our focus will be on bounding $\E_{\Pi^{Y,S}}[n_{y,s}]$.
 
In Algorithm~\ref{silif:RGO}, Steps~2 and~3 are a rejection sampler where the true potential function $\Theta$ and its proposal $\mathcal{P}$ are respectively 
\begin{equation}\label{eq:corr2}
      \Theta = \Theta^{\eta,Q}_{y,s}, \quad \mathcal{P} = \mathcal{P}_1
\end{equation}
in view of \cite[Lemma B.3]{dangliang2025oracle}, so that the latter result applies. It follows from $n_{y,s}=\E[F]$ in~\cite[(16)]{dangliang2025oracle} with the specification \eqref{eq:corr2} and Lemma~\ref{silif:lem:compareP1P2} that
\begin{align}
\label{silif:for:averagerejectionnotwarmsepa}
 \E_{\Pi^{Y,S}} [n_{y,s}]
 &= \int_{\R^{d+1}} \frac{\int_{\R^{d+1}} \exp(-\mathcal{P}_1(x,t))  \D x  \D t}{\int_{\R^{d+1}} \exp(-\Theta^{\eta,Q}_{y,s}(x,t))  \D x  \D t} \Pi^{Y,S}(y,s) \D y  \D s\nn\\
 & \leq \int_{\R^{d+1}} \frac{\int_{\R^{d+1}} \exp(-\mathcal{P}_2(x,t))  \D x  \D t}{\int_{\R^{d+1}} \exp(-\Theta^{\eta,Q}_{y,s}(x,t))  \D x  \D t} \Pi^{Y,S}(y,s) \D y  \D s .
\end{align}
Now via the definition of $\mathcal{P}_2$ at~\eqref{silif:def:P2}, \cite[Lemma B.1(b)]{dangliang2025oracle} and the fact that $L/a\leq 1$ for $a=(LR_\mathrm{out}+d)/R_\mathrm{in}$, 
\begin{align}
\label{silif:sepa:boundintP3}
&\int_{\R^{d+1}}\exp(-\mathcal P_2(w))\D w
\leq \exp\left(\frac{33}{4}\right)(2\pi\eta)^{(d+1)/2}\nn\\
&\quad\cdot
\exp\left(
-\frac{(\operatorname{dist}(z,Q)-2\Delta)^2}{2\eta}
\right)
\exp\left(
\frac{a^2\eta}{2}-as+\frac{24}{d+1}
\right),
\end{align}
where $z=(y,s-a\eta)$. 
Applying \cite[Lemma B.1(a)]{dangliang2025oracle} and the definition of $\Theta^{\eta,Q}_{y,s}$ in~\eqref{silif:def:Theta}, we have
\begin{align}
    \Pi^{Y,S}(y,s)\stackrel{\eqref{silif:def:bigpi}}=\frac{1}{(2\pi\eta)^{(d+1)/2}\int_Q\exp(-at)\D x\D t}\int_Q\exp\left(-at-\frac{\norm{(x,t)-(y,s)}^2}{2\eta}\right)\D x\D t,
    \label{silif:for:piYS}\\
    \int_{\R^{d+1}}\exp(-\Theta^{\eta,Q}_{y,s}(x,t))\D x\D t
    =\int_Q\exp\left(-at-\frac{\norm{(x,t)-(y,s)}^2}{2\eta}\right)\D x\D t.
    \label{eq:int-theta}
\end{align}
Plugging \eqref{silif:sepa:boundintP3}, \eqref{silif:for:piYS}, and \eqref{eq:int-theta} into \eqref{silif:for:averagerejectionnotwarmsepa}, we obtain
\begin{align*}
     &\E_{\Pi^{Y,S}} [n_{y,s}]
     \stackrel{\eqref{silif:for:averagerejectionnotwarmsepa}}\leq \frac{1}{\int_Q\exp(-at) \D x \D t}\exp\left(\frac{33}{4}+\frac{24}{d+1}\right)\nonumber\\
&\quad\cdot\int_{\R^{d+1}}
\exp\left(\frac{a^2\eta}{2}-as\right)
\exp\left(-\frac{(\operatorname{dist}(z,Q)-2\Delta)^2}{2\eta}\right)
\D y\D s.
\end{align*}
Recall the notation $z_{d+1}=(y,s-a\eta)_{d+1}=s-a\eta$. Then
\begin{align}
&\int_{\R^{d+1}}
\exp\left(\frac{a^2\eta}{2}-as\right)
\exp\left(-\frac{(\operatorname{dist}(z,Q)-2\Delta)^2}{2\eta}\right)
\D y\D s\nn\\
&\quad=
\exp\left(-\frac{1}{2}\right)
\int_{\R^{d+1}}e^{-az_{d+1}}
\exp\left(-\frac{(\operatorname{dist}(z,Q)-2\Delta)^2}{2\eta}\right)\D z,
\label{obs}
\end{align}
where we use $\eta=1/a^2$ as specified in \eqref{silif:chooseaandeta}.
Applying Lemma~\ref{silif:lem:keytechnical} to~\eqref{obs} with $\tau=2\Delta$, and then using~\eqref{silif:sepawarmstartimplication} and $\Delta$ in \eqref{silif:difference:tildextildet:mainpaper}, we obtain
\begin{equation}
\label{silif:proofboundsepa}
\E_{\mu_k}[n_{y,s}]
\leq
M\exp\left(
\frac{31}{4}
+\frac{24}{d+1}
+4\sqrt{2}C_L\sqrt{\frac{\eta}{d+1}}
\right)
\left[
\sqrt{2\pi\eta}C_L
\exp\left(\frac{\eta C_L^2}{2}\right)+1
\right].
\end{equation}
The definitions in~\eqref{silif:def:notation} provide us with $
C_K=a,
C_f(a)\leq\sqrt{2}a$ and $
C_L(a)\leq(1+\sqrt{2})a
$. Since $\eta=a^{-2}$, we also have $\sqrt{\eta}C_L(a)\leq1+\sqrt{2}$. Substituting these bounds into~\eqref{silif:proofboundsepa} shows that the right-hand side is at most $CM$ for a universal constant $C$.
\QEDA

\section{Composite Sampling}
\label{sec:composite:mainpaper}
 Regarding the composite target $\tilde \nu$ at~\eqref{dolif:def:nu}, we will use the following conditions:
\begin{itemize}
\item[(B1)] $f$ and $h$ are closed and convex functions on $\R^d$;

\item[(B2)] $f$ and $h$ are Lipschitz continuous with constants $L_f$ and $L_h$, respectively;

    \item[(B3)] there exists an initial distribution $\tilde{\nu}_0$ that is $M$-warm with respect to $\tilde{\nu}$ i.e., $\frac{\D\tilde{\nu}_0}{\D \tilde{\nu}}\leq M$;
    
\item[(B4)] there exists a separation oracle for $\Q$.
\end{itemize}

Now we elaborate on the two-step epigraph lifting (\emph{double lifting}) that was briefly mentioned in Section \ref{sec:intro}. We first apply the epigraph lifting to $h$:
\begin{align}
\label{dolif:def:gammaandK}
  \gamma(x,s)&\propto \exp\left(-\tilde f(x,s)-I_{\K}(x,s)\right),\\
  \mbox{where } \K&:=\{(x,s)\in\mathbb R^d\times\mathbb R:\ h(x)\le as\},
  \qquad \tilde f(x,s):=f(x)+as.\notag
\end{align}
Since $\gamma$ is supported on the convex set $\K$, the first lifting reduces the composite problem to constrained sampling. We then apply epigraph lifting again to $\tilde{f}$ and obtain a lifted distribution in $\R^{d+2}$:
\begin{equation}
\label{dolif:def:piandQ}
\tilde{\pi}(x,s,t)\ \propto\ \exp \left(-bt-I_{\Q}(x,s,t)\right),
\qquad
\Q=\{(x,s,t):\ (x,s)\in\K,\ \tilde f(x,s)\le bt\}.
\end{equation}
Integrating out $(s,t)$ shows that $\tilde{\pi}^X=\tilde{\nu}$, since $\int_{h(x)/a}^{\infty}\int_{\tilde f(x,s)/b}^{\infty} e^{-bt} \D t \D s=\frac{1}{ab}\exp\!\left(-f(x)-h(x)\right)$; thus sampling $(x,s,t)\sim\tilde{\pi}$ and discarding $(s,t)$ yields samples from $\tilde{\nu}$.  Thus, the double lifting turns sampling the composite target $\tilde{\nu}$ on $\R^d$ at \eqref{dolif:def:nu} into sampling a uniform
distribution up to an one-dimensional exponential tilt over a convex set in $\R^{d+2}$. In the upcoming parts, we will first apply the proximal sampler to the double-lifted target $\tilde \pi$ and then present our own RGO implementation.

Here and in Appendix \ref{appen:composite}, $Z_S$ and $B_S$ are the quantities $Z_S(b)$ and $B_S(b)$ in~\eqref{def:ZandB}. Also, denote
\begin{equation}
\label{dolif:def:notation}
    p=(x,s,t), q=(y,u,v-b\eta);\,\, \widetilde{C}_f(a)=\sqrt{b^2+a^2+L_f^2},    \widetilde{C}_h(a)=\sqrt{a^2+L_h^2},    \widetilde{C}_L(a)=\widetilde{C}_f+\widetilde{C}_h. 
\end{equation}

\subsection{Proximal sampler for the double-lifted distribution $\tilde \pi$}
\label{dolif:sec:ASF}

We follow the idea of the proximal sampler for the target $\tilde{\pi}$ at \eqref{dolif:def:piandQ} and start with the augmented distribution
\[
    \widetilde{\Pi}\brac{(x,s,t),(y,u,v)}\propto \exp\brac{-bt-I_{\Q}(x,s,t)-\frac{1}{2\eta}\norm{(x,s,t)-(y,u,v) }^2}.
\]
For a fixed point $(y,u,v)\in \R^{d+2}$,  set
\begin{align}
\label{dolif:def:Theta}
    \widetilde{\Theta}^{\eta,\Q}_{y,u,v}(x,s,t)&=I_{\Q}(x,s,t)+bt+\frac{1}{2\eta}\norm{(x,s,t)-(y,u,v) }^2\nn\\
    &=I_{\Q}(x,s,t)+\frac{1}{2\eta}\|(x,s,t)-(y,u,v-\eta b)\|^2
+bv-\frac{\eta b^2}{2}. 
\end{align}

Then $\widetilde{\Pi}^{Y,U,V|X,S,T}(y,u,v|x,s,t)=\mathcal{N}\brac{p,\eta I_{d+2}}$ and $ \widetilde{\Pi}^{X,S,T|Y,U,V}(x,s,t|y,u,v)=\mathcal{N}\brac{q,\eta I_{d+2}}|_{\Q}$. The proximal sampler for $\tilde{\pi}$ is given below.

\begin{algorithm}[H]
	\caption{Proximal sampler for the target $\tilde{\pi}$ at \eqref{dolif:def:piandQ}}
	\label{dolif:alg:ASF_pi}
	\begin{algorithmic}[1]
		\State \textbf{Gaussian}: Generate $(y_k,u_k,v_k)\sim  \mathcal{N}\brac{(x_k,s_k,t_k),\eta I_{d+2}}$; 
		\State \textbf{RGO}: Generate $(x_{k+1},s_{k+1},t_{k+1})\sim {\cal N}\brac{(y_k,u_k,v_k-b\eta),\eta I_{d+2}}|_{\Q}.$ 
	\end{algorithmic}
\end{algorithm}


To initialize Algorithm \ref{dolif:alg:ASF_pi}, we generate $x_0\sim \tilde{\nu}_0$, an $M$-warm start for $\tilde{\nu}$ in~\eqref{dolif:def:nu}.
The $X$-marginal of $\tilde{\pi}$ is $\tilde{\nu}$ and for $x\in\R^d$,
\begin{equation}
\label{dolif:eq:conditionalS}
\tilde{\pi}^{S\mid X=x}(s)
=
a e^{h(x)-as}\mathbf{1}_{\{s\geq h(x)/a\}}
=
a e^{-(as-h(x))}\mathbf{1}_{\{as-h(x)\geq 0\}}.
\end{equation}
Thus, $aS-h(x)\mid (X=x)\sim\operatorname{Exp}(1)$,
and for $x\in\R^d$ and $s\geq h(x)/a$,
\begin{equation}
\label{dolif:eq:conditionalT}
\tilde{\pi}^{T\mid X=x,S=s}(t)
=
b e^{f(x)+as-bt}\mathbf{1}_{\{t\geq (f(x)+as)/b\}}
=
b e^{-(bt-f(x)-as)}\mathbf{1}_{\{bt-f(x)-as\geq 0\}}.
\end{equation}
Hence, $bT-f(x)-as\mid (X=x,S=s)\sim\operatorname{Exp}(1)$,
which further implies for i.i.d. $E_1,E_2\sim\operatorname{Exp}(1)$ that are independent of $X$,
\[
(X,S,T)\stackrel{\text{d}}{=}
\left(
X,
\frac{h(X)+E_1}{a},
\frac{f(X)+h(X)+E_1+E_2}{b}
\right).\]
Therefore, we initialize $s_0$ by drawing $e_{1,0}\sim\operatorname{Exp}(1)$ independent of $x_0$ and setting $s_0=(h(x_0)+e_{1,0})/a$. Equation~\eqref{dolif:eq:conditionalS} then implies $s_0\sim\tilde{\pi}^{S\mid X=x_0}$. We next initialize $t_0$ by drawing $e_{2,0}\sim\operatorname{Exp}(1)$ independent of $(x_0,s_0)$ and setting $t_0=(f(x_0)+as_0+e_{2,0})/b$. Equation~\eqref{dolif:eq:conditionalT} then gives $t_0\sim\tilde{\pi}^{T\mid X=x_0,S=s_0}$ and $(x_0,s_0,t_0)\sim\tilde{\pi}_0(x,s,t)=\tilde{\nu}_0(x)\tilde{\pi}^{S\mid X=x}(s)\tilde{\pi}^{T\mid X=x,S=s}(t)$.
Starting from $(x_0,s_0,t_0)\in\Q$, Algorithm \ref{dolif:alg:ASF_pi} keeps generating $(x_k,s_k,t_k)\sim\tilde{\pi}^k$ at Step 2 and maintains feasibility $(x_k,s_k,t_k)\in\Q$. Furthermore, $\tilde{\pi}_0$ is $M$-warm for $\tilde{\pi}$ in~\eqref{dolif:def:piandQ} since 
\begin{equation}\label{dolif:eq:warmstart}\frac{\D\tilde{\pi}_0}{\D\tilde{\pi}}(x,s,t)=
\frac{\frac{\D\tilde{\nu}_0}{\D x}(x)  \tilde{\pi}^{S,T \mid X=x}(s,t)}
{\frac{\D \tilde{\nu}}{\D x}(x)  \tilde{\pi}^{S,T \mid X=x}(s,t)}
=
\frac{\D \tilde{\nu}_0}{\D \tilde{\nu}}(x)
\le M .\end{equation}
Lemma~\ref{dolif:lem:warmstart} further shows that $\tilde{\pi}^k$ is $M$-warm with respect to $\tilde \pi$ for every $k\geq 1$.


Next,  assuming an exact RGO, we bound the iteration complexity of Algorithm~\ref{dolif:alg:ASF_pi} in $\chi^2$ and R\'enyi divergences via Theorems~\ref{theo:outer}. The proof is deferred to Appendix~\ref{dolif:appen:proofouter}. 
\begin{theorem}\label{dolif:theo:outer}
Assume (B1)-(B3) hold, let $a,b>0$ in~\eqref{dolif:def:piandQ}, and assume $a^2\geq C_0L_h^2\log(d+1)$. Let $C_{\mathrm{PI}}$ be the Poincar\'e constant of $\tilde\pi$, for which $C_{\mathrm{PI}}
\leq
C'\log(d+2)
\left(
\opnorm{\cov{\tilde\nu}}
+\frac{d+3}{a^2}
+\frac{d+3}{b^2}
\right)$
for a universal constant $C'$. Given $\varepsilon>0$, the iteration complexity in $\chi^2$ divergence for $\varepsilon$-accuracy is
\[\mathcal{O}\left(\left(1+\frac{C_{\mathrm{PI}}}{\eta}\right)\log\frac{M}{\varepsilon}\right).\]
Moreover, for $\mathcal{R}_q$ with $q\ge1$, the iteration complexity in $\mathcal{R}_q$ to reach $\varepsilon$-accuracy is 
\[
O\left(q\left(1+\frac{C_{\mathrm{PI}}}{\eta}\right)\log\frac{M}{\varepsilon}\right).
\]
\end{theorem}


\subsection{RGO implementation}\label{subsec:RGO:composite}

In Theorem~\ref{dolif:theo:outer}, the analysis of Algorithm~\ref{dolif:alg:ASF_pi} treats Step~2 as an exact draw from a Gaussian distribution truncated to $\Q$ at~\eqref{dolif:def:piandQ}. The goal of this section is to provide a realization of this truncated-Gaussian sampler via rejection sampling and to quantify the resulting acceptance probabilities and oracle usage.

Similar to the explanation given in the beginning of Section~\ref{sec:constrained:mainpaper}, we need to approximately solve for \[\argmin_{(x,s,t)\in\R^{d+2}}\widetilde{\Theta}^{\eta,\Q}_{y,u,v}(x,s,t)
=\proj_{\Q}( q),\] where $\widetilde{\Theta}^{\eta,\Q}_{y,u,v}$ is defined at~\eqref{dolif:def:Theta}, in order to center our proposal in the rejection sampler. To do so, we use the cutting-plane method by~\cite{vaidya,lee2015fastercutting}, for which the iteration complexity and running time are restated in Theorem~\ref{theo:vaidya}. Since this method requires an initial enclosing polytope, we restrict the search to a local bounded region inside a ball that still contains the unique minimizer of $\widetilde{\Theta}^{\eta,\Q}_{y,u,v}$. In particular, the reformulation in Lemma~\ref{silif:lem:equivopt} for a bounded convex body $K$ does not extend to the composite setting since $f$ and $h$ are defined on the unbounded domain $\R^d$.

Let the current input of Algorithm~\ref{dolif:alg:ASF_pi} be $(y,u,v)$ with $q=(y,u,v-b\eta)$. Let the previous RGO output be $ p_{k-1}=(x_{k-1},s_{k-1},t_{k-1})\in\Q$ a.s., set $r_{\local}=\| p_{k-1}- q\|$ and define the local ball and local region as
\begin{equation}
\label{dolif:def:localball}
    \ball_{\local}=q+2r_{\local}\ball_{d+2}(0), \qquad  \Q_{\local}=\Q\cap \ball_{\local}.
\end{equation}
First, to ensure that restricting the cutting-plane method to $\Q_{\local}$ is valid, we verify in Lemma~\ref{dolif:lem:minimizerinlocalball} in Appendix~\ref{dolif:appen:cuttingplane} that the minimizer $ p_*=\argmin_{ p\in \R^{d+2}}\widetilde{\Theta}^{\eta,\Q}_{y,u,v}( p)$
belongs to $\Q_{\local}$ a.s. Second, we identify in the following our oracle assumptions on $f$ and $h$ that yield a separation oracle for $\Q$, and hence the cutting-planes required by the cutting-plane method on $\Q_{\mathrm{loc}}$.


If both subgradient oracles for $f$ and $h$ are available, then we need not apply the lifting technique and can simply treat $f+h$ as one potential. Hence, any log-concave sampling algorithm using the subgradient of $f+h$ (e.g., Langevin Monte Carlo) applies to this composite sampling problem. Next, we discuss two concrete implementations of separation oracles for $\Q$.
For simplicity, we denote $\prox_f$ and $f'$ the proximal and subgradient oracles for a function $f$, respectively. Moreover, we assume the evaluation oracles for $f$ and $h$ are available.

\textit{Case 1: $\prox_f$ and $\prox_h$ are available.} Indeed, $\prox_h$ gives a projection oracle for $\epi(h)$ by Lemma~\ref{lem:proj-epi}, and hence a separation oracle for $\epi(h)$ by Lemma~\ref{lem:proj-sep}. By a simple scaling, this yields a separation oracle for $\K$ in \eqref{dolif:def:gammaandK}. In addition, Lemma \ref{silif:lem:proxfas} shows that $\prox_f$ yields a proximal oracle for $\tilde f(x,s)$ in \eqref{dolif:def:gammaandK}.
     Now, applying Lemma \ref{silif:lem:sep2} to $\K$ and $\tilde f$, we have a separation oracle for $\Q$ in \eqref{dolif:def:piandQ}.

\textit{Case 2: $f'$ and $\prox_h$ are available.} The same argument as in the above case produces a separation oracle for $\K$. 
    In addition, it is straightforward to see that $f'$ yields a subgradient oracle for $\tilde f(x,s)$ in \eqref{dolif:def:gammaandK}.
    Now, applying Lemma \ref{silif:lem:sep} to $\K$ and $\tilde f$, we have a separation oracle for $\Q$ in \eqref{dolif:def:piandQ}.

Having verified the requirements, we can solve for a $(d+2)^{-1}$-solution $\tilde{ p}$ of $\argmin_{ p\in \Q_\local} \widetilde{\Theta}^{\eta,\Q}_{y,u,v}( p)$ via the cutting-plane method by \cite{vaidya,lee2015fastercutting} and center at $\tilde{p}$ the proposal in our rejection sampler for the RGO, i.e., the proposal density is proportional to $\exp(-\widetilde{\mathcal{P}}_1(p))$ where 
\begin{equation}
\label{dolif:def:tildeP1}
    \widetilde{\mathcal{P}}_1( p)
  = \frac{\| p-\tilde p\|^2 + \|\tilde p- q\|^2}{2\eta}
     - \frac{\tilde\Delta}{\eta} \left(\| p-\tilde p\|+\|\tilde p- q\|\right) -\frac{b^2\eta}{2} + b v,\qquad \tilde\Delta = \sqrt{2\eta/(d+2)}.  
\end{equation}
The following is the RGO implementation of Algorithm~\ref{dolif:alg:ASF_pi} with a separation oracle on $\Q$.
\begin{algorithm}[H]
	\caption{RGO implementation of Algorithm~\ref{dolif:alg:ASF_pi}}
	\label{dolif:RGO}
	\begin{algorithmic}[1]
		\State  Compute a $(d+2)^{-1}$-solution $\tilde{ p}$ of
        $\argmin_{ p\in \Q_\local} \widetilde{\Theta}^{\eta,\Q}_{y,u,v}( p)$ via the cutting-plane method by \cite{vaidya,lee2015fastercutting}.
		\State Generate $U\sim {\cal U}[0,1]$ and $ p\sim \exp\brac{-\widetilde{\mathcal{P}}_1( p)}$ via Algorithm \ref{dolif:alg:samplingP1} in Appendix \ref{dolif:appen:specialalgo}. 
		\State  If
        \begin{equation}\label{dolif:event:separation}
			U \leq \exp\brac{-\widetilde{\Theta}_{y,u,v}^{\eta,\Q}( p)+\widetilde{\mathcal{P}}_1( p) }
		\end{equation}
		then accept $ p$; otherwise, reject $ p$ and go to step 2.
	\end{algorithmic}
\end{algorithm}
The next result ensures that the acceptance test~\eqref{dolif:event:separation} is well-defined. The proof is similar to the proof of Lemma~\ref{silif:lem:compareThetaP1} and hence omitted. 
\begin{lemma}\label{dolif:lem:compareP1P2}
 Under condition~(B1), we have $\widetilde{\mathcal P}_1(p)\leq\widetilde{\Theta}_{y,u,v}^{\eta,\Q}(p)$ for all $p\in\mathbb R^{d+2}$.
\end{lemma}

    Here we motivate the specific choice of radius $2r_{\local}$ in \eqref{dolif:def:localball}. 
Per Theorem~\ref{theo:vaidya}, bounding the oracle complexity of the cutting-plane method requires a lower bound on $\mbox{minwidth}(\Q_{\local})$. We obtain this bound by explicitly constructing an inscribed ball $\ball_{\mathrm{in}}$ of radius $\rin$ inside $\Q_{\local}$, which allows us to use $\mbox{minwidth}(\Q_{\local})\ge \mbox{minwidth}(\ball_{\mathrm{in}})=2\rin$. 
However, the construction of this inscribed ball (see Lemma~\ref{dolif:lem:localset}) is delicate and relies on a strict enlargement of the local set. Specifically, we require the local set to be $\Q\cap (q+c r_{\local}\ball_{d+2}(0))$ with $c>1$ to ensure sufficient interior volume (the argument for the existence of $\ball_{\mathrm{in}}$ breaks down when $c=1$). Therefore, we work with the choice of $c=2$.

The following is our main result for the RGO implementation.



\begin{theorem}
\label{dolif:theo:separation}
Assume conditions~(B1) through (B4) hold, set \begin{equation} \label{dolif:eq:chooseabeta}a=b=A:=\max\left\{1,L_f,L_h,\sqrt{C_0}L_h\sqrt{\log(d+1)}\right\}\quad\text{and}\quad \eta=\frac{1}{A^2}.\end{equation} 
\begin{itemize}
\item[(a)] \textbf{(oracle complexity)} there are
${\mathcal{O}}\left(d\log\frac{(d+1)\gamma}{\alpha}\right)$ expected calls to the
separation oracle of $\Q$. Here $\gamma=2r_{\local}/\mbox{minwidth}(\Q_{\local})\leq2+6\sqrt{3}$ and $\alpha$ is defined in~\eqref{dolif:def:deltaandalpha}. Moreover, it holds that
\begin{equation}
\label{dolif:eq:concent}
\Pr\left(
\alpha\leq
\frac{1}{18(d+2)^2}
\right)
\leq
2\exp\left(-\frac{d+2}{2}\right).
\end{equation}
\item[(b)] \textbf{(proposal complexity)} the average number of proposals is ${\mathcal{O}}(M)$.
\end{itemize}
\end{theorem}

\subsection{Proof of Theorem~\ref{dolif:theo:separation}}

Similar to Lemma~\ref{silif:lem:compareP1P2}, we construct an auxiliary function $\widetilde{\mathcal{P}}_2$ that is useful for the rejection analysis. 
The following lemma follows a similar argument as in Lemma~\ref{silif:lem:compareP1P2}. Hence, the proof is omitted.

\begin{lemma}
\label{dolif:lem:compareP2P3}
 For $p=(x,s,t)\in \R^{d+2}$, define 
\begin{equation}\label{dolif:def:P2}
  \widetilde{\mathcal{P}}_2( p)
  = \frac{1}{2\eta}\left[
        (\|p-\tilde p\| - \tilde\Delta)^2
        + (\dist(q,\Q) - 2\tilde\Delta)^2
        - 6\tilde\Delta^2
     \right] + b v -\frac{b^2\eta}{2},
\end{equation}
where $q=(y,u,v-b\eta)$ is as in~\eqref{dolif:def:notation} and $\tilde p$ is as in Algorithm~\ref{dolif:RGO}.
Then, $\widetilde{\mathcal{P}}_2( p) \le \widetilde{\mathcal{P}}_1( p)$.
\end{lemma}

We are now ready to prove Theorem~\ref{dolif:theo:separation}.

\textbf{Proof of Theorem~\ref{dolif:theo:separation}:} 
The proof of Part~a is a non-trivial application of Vaidya's cutting-plane method recorded in Theorem~\ref{theo:vaidya}.  Due to how involved these arguments are, we defer the proof of Part~a and its supporting lemmas to Appendix~\ref{dolif:appen:cuttingplane}

 We will verify below the proposal complexity claim in Part~b, which largely follows the argument of Part~b of
Theorem~\ref{silif:theo:separation}. Let $\mu_k$ denote the
distribution of $(y_k,u_k,v_k)$ after the first step of
Algorithm~\ref{dolif:alg:ASF_pi}, and let $n_{y,u,v}$ denote the
expected number of proposals conditioned on the RGO input
$(y_k,u_k,v_k)$. By Lemma~\ref{dolif:lem:warmstart},
\begin{equation}
\label{dolif:warmstartimplicationsepa}
\E_{\mu_k}[n_{y,u,v}]
\leq
M\E_{\widetilde{\Pi}^{Y,U,V}}[n_{y,u,v}].
\end{equation}
It therefore suffices to bound
$\E_{\widetilde{\Pi}^{Y,U,V}}[n_{y,u,v}]$.

In Steps~2 and~3 of Algorithm~\ref{dolif:RGO}, the true and
proposal potentials are
$\widetilde{\Theta}^{\eta,\Q}_{y,u,v}$ and
$\widetilde{\mathcal P}_1$, respectively. Using
$n_{y,u,v}=\E[F]$ in~\cite[(16)]{dangliang2025oracle}, 
Lemma~\ref{dolif:lem:compareP2P3}, the definition of
$\widetilde{\mathcal P}_2$ in~\eqref{dolif:def:P2}, and
\cite[Lemma B.1(b)]{dangliang2025oracle}, we obtain
\[
\int_{\R^{d+2}}
e^{-\widetilde{\mathcal P}_2(x,s,t)}
\D x\D s\D t
\leq
e^{\frac94}(2\pi\eta)^{\frac{d+2}2}
\exp\left(
-\frac1{2\eta}
\left(
\dist(q,\Q)-2\tilde\Delta
\right)^2
\right)
\exp\left(
\frac{b^2\eta}{2}-bv+\frac{6}{d+2}
\right),
\]
where $\tilde\Delta=\sqrt{2\eta/(d+2)}$. On the other hand, the
definition~\eqref{dolif:def:Theta} and
\cite[Lemma B.1(a)]{dangliang2025oracle} show that the normalizing integral of
$\widetilde{\Theta}^{\eta,\Q}_{y,u,v}$ cancels with the corresponding
integral in $\widetilde{\Pi}^{Y,U,V}$. Thus, writing
$Z_{\Q}=\int_{\Q}e^{-bt}\D x\D s\D t$, the same calculation as in
Part~b of Theorem~\ref{silif:theo:separation} yields
\begin{align}
\label{dolif:separation:intermediate}
&\E_{\widetilde{\Pi}^{Y,U,V}}[n_{y,u,v}]
\leq
\frac{1}{Z_{\Q}}
\exp\left(\frac94+\frac{6}{d+2}\right)
\int_{\R^{d+2}}
\exp\left(\frac{b^2\eta}{2}-bv\right)
\exp\left(
-\frac{(\dist(q,\Q)-2\tilde\Delta)^2}{2\eta}
\right)
\D q \nn \\
&= \frac{1}{Z_{\Q}}
\exp\left(\frac94+\frac{6}{d+2}-\frac{b^2\eta}{2}\right)
\int_{\R^{d+2}}
e^{-bq_{d+2}}
\exp\left(
-\frac{(\dist(q,\Q)-2\tilde\Delta)^2}{2\eta}
\right)
\D q,
\end{align}
where we use  $q=(y,u,v-b\eta)$ in the equality.
Applying Lemma~\ref{dolif:lem:keytechnical} with
$\tau=2\tilde\Delta$, the integral in \eqref{dolif:separation:intermediate} becomes
\[
Z_{\Q}
\exp\left(2\tilde\Delta\widetilde C_L
\right)\left[
\widetilde C_L\sqrt{2\pi\eta}
\exp\left(\frac{\eta\widetilde C_L^2}{2}\right)+1
\right].
\]
Substituting this bound into
\eqref{dolif:separation:intermediate}, using
$2\tilde\Delta=2\sqrt{2\eta/(d+2)}$, and then applying
\eqref{dolif:warmstartimplicationsepa}, we obtain
\[
\E_{\mu_k}[n_{y,u,v}] \le  M\exp\left(
\frac94+\frac{6}{d+2}-\frac{b^2\eta}{2}
+2\widetilde C_L\sqrt{\frac{2\eta}{d+2}}
\right)
\left[
\widetilde C_L\sqrt{2\pi\eta}
\exp\left(\frac{\eta\widetilde C_L^2}{2}\right)+1
\right].
\]
By the definition of $A$, we have $a=b=A\geq L_f,L_h$, so $\widetilde C_L\leq(\sqrt{3}+\sqrt{2})A$. Since $\eta=1/A^2$, we have $\eta\widetilde C_L^2={\mathcal{O}}(1)$ and $\sqrt{\eta}\widetilde C_L={\mathcal{O}}(1)$. Moreover, $-b^2\eta/2\leq0$. Therefore, the proof is completed.
\QEDA

\section{Conclusions} \label{sec:conclusion}

In this paper, we develop proximal sampling algorithms for constrained and composite log-concave distributions through epigraph liftings. 
By combining the cutting-plane method with rejection sampling, we obtain exact RGO implementations under different oracle assumptions. 
We establish mixing time and oracle complexity bounds for both sampling problems. 
We also provide a variational convergence analysis of the proximal sampler, deriving the evolution PDEs of its two half-steps and the corresponding contraction bounds in Rényi and $\chi^2$ divergences.

On a conceptual level, our work highlights two complementary connections between optimization and sampling. 
At the algorithmic level, epigraph liftings and oracle-based optimization enable the construction and implementation of proximal samplers for constrained and composite targets. 
At the analytical level, variational optimization over probability measures provides a direct approach to understanding the evolution and convergence of the proximal sampler. 
These perspectives demonstrate how optimization can serve both as a tool for developing sampling algorithms and as a framework for analyzing their convergence.

\bibliographystyle{plain}
\bibliography{refs}


\appendix

\section{General Results} \label{appen:general}

\subsection{Function and set oracles}
\label{appen:oracles}

This subsection collects elementary results about oracles for functions (evaluation, subgradient, and proximal) and sets (membership, separation, and projection). 
Throughout this subsection, we assume that both $f:\R^d\to \R$ and $K\subset \R^d$ are closed and convex.

\begin{lemma}\label{silif:lem:proxfas}
The proximal oracle for $f(x)$ gives a proximal oracle for $\tilde f(x,s)$ defined in \eqref{dolif:def:gammaandK}.
\end{lemma}

\begin{proof}
    Let $(x_0,s_0) \in \R^d \times \R$ and $\lam>0$ be given. By definitions of proximal mapping and $\tilde f(x,s)$ in \eqref{dolif:def:gammaandK}, we have $ \prox_{\lam \tilde f}(x_0, s_0)=\underset{x,s}\argmin\left\{f(x) + as + \frac{1}{2\lam}\|x-x_0\|^2 + \frac{1}{2\lam}(s-s_0)^2\right\}$.
    This objective splits into an $x$-term and an $s$-term, so the proximal mapping can be easily obtained as $\prox_{\lam \tilde f}(x_0, s_0) = (\prox_{\tilde f}(x_0), s_0 - \lam a)$. 
    Hence, the lemma is proved.
\end{proof}

\begin{lemma}\label{lem:proj-sep}
    For a closed and convex set $K$, a projection oracle for $K$ implies a separation oracle.
\end{lemma}

\begin{proof}
    Given $x$, it suffices to discuss the case $x\notin K$. For $p=\proj_K(x)$, $x\ne p$ so $g=x-p \ne 0$ satisfies $\inner{g}{x} > \inner{g}{p} \ge \inner{g}{y}$ for every $y \in K$. Hence, $g$ is a separator for $x$ and $K$.
\end{proof}

\begin{lemma}\label{lem:sep}
    Evaluation and subgradient oracles of $f$ give a valid separation oracle for $\epi(f)$.
\end{lemma}

\begin{proof}
    Consider a query point $(\hat x,\hat t) \in \R^d \times \R$. First, evaluating $f(\hat x)\le \hat t$ checks whether $(\hat x,\hat t)\in \epi(f)$. If infeasible (i.e., $f(\hat x) > \hat t$), then pick a subgradient $g \in \partial f(\hat x)$, and convexity of $f$ gives $f(y) \ge f(\hat x)+\inner{g}{y-\hat x}$ for all $y$. For $(y,s) \in \epi(f)$, $s\ge f(y) \ge f(\hat x)+\inner{g}{y-\hat x}$
    and hence $ \inner{(g,-1)}{(y,s)} \le \inner{g}{\hat x} - f(\hat x) < \inner{g}{\hat x} - \hat t = \inner{(g,-1)}{(\hat x,\hat t)}$. 
    Therefore, $(g,-1)$ is a valid separator for $(\hat x,\hat t)$ and $\epi(f)$.
\end{proof}

\begin{lemma}\label{lem:proj-epi}
    Evaluation and proximal oracles of $f$ give a valid projection oracle for $\epi(f)$. 
\end{lemma}

\begin{proof}
    Consider a query point $(\hat x,\hat t) \in \R^d \times \R$. First, evaluating $f(\hat x)\le \hat t$ checks whether $(\hat x,\hat t)\in \epi(f)$. If infeasible (i.e., $f(\hat x) > \hat t$), we consider
    \[
    \min_{x, t} \left\{\frac{1}{2}\|x-\hat x\|^2+\frac{1}{2}(t-\hat t)^2: f(x) \le t\right\}.
    \]
    Optimality conditions give $\lambda \geq 0$ such that $ \hat x-x \in \lam \partial f(x), t-\hat t - \lam =0$ and $\lam (f(x) - t) = 0$.  
    So, we have $f(x) = t= \hat t + \lam$ and $\hat x-x \in \lam \partial f(x) \Longleftrightarrow x=\prox_{\lam f}(\hat x)$. Define  $\phi(\lam) = f(\prox_{\lam f}(\hat x)) -\hat t - \lam$. 
    Then $\phi$ is continuous and strictly decreasing on $[0, \infty)$, with $\phi(0^{+})=f(\hat x)-\hat t>0$ and $\phi(\lam) \rightarrow-\infty$ as $\lambda \rightarrow \infty$.
    Therefore, there is a unique $\lam_*>0$ such that $\phi(\lam_*)=0$, and the exact projection is
$x_* = \prox_{\lam_* f}(\hat x)$ and   $t_* = \hat t + \lam_*$.  Finally, we verify the monotonicity of $\phi(\lam)$. For $0< \lam_1 < \lam_2$, let $x_1= \prox_{\lam_1 f}(\hat x)$ and $x_2=\prox_{\lam_2 f}(\hat x)$, by optimality
    \begin{align*} 
    & \frac{1}{2}\|x_2-x_1\|^2 + \frac{1}{2}\|x_2-\hat x\|^2+\lam_2 f(x_2) \le \frac{1}{2}\|x_1-\hat x\|^2+\lam_2 f(x_1), \\ 
    & \frac{1}{2}\|x_2-x_1\|^2 + \frac{1}{2}\|x_1-\hat x\|^2+\lam_1 f(x_1) \le \frac{1}{2}\|x_2-\hat x\|^2+\lam_1 f(x_2) .
    \end{align*}
    Summing the two inequalities and rearranging, we have $(\lambda_2-\lambda_1)(f(x_2)-f(x_1)) \le -\|x_2-x_1\|^2 \leq0$
    and hence $f(x_2)\leq f(x_1)$, i.e., $\phi$ is strictly decreasing. Therefore, $\lam_*$ can be solved analytically or numerically using a bisection search.
\end{proof}

We are ready to prove Lemma~\ref{silif:lem:sep}.

\noindent
\textbf{Proof of Lemma~\ref{silif:lem:sep}:} 
Without loss of generality, we assume $a=1$ in \eqref{silif:def:pi}, i.e., $Q= (K \times \R) \cap \epi(f)$. Consider a query point $(\hat x,\hat t) \in \R^d \times \R$. There are two infeasible cases: i) $\hat x \notin K$ and ii) $(\hat x, \hat t)\notin \epi(f)$. 

Regarding Case i, suppose $p\in \R^d$ is a separator for $\hat x$ and $K$, i.e., $\inner{p}{y} < \inner{p}{\hat x}$ for every $y \in K$ 
    Then the hyperplane given by $(p,0)$ separates $(\hat x,\hat t)$ and $Q$. Indeed, $\inner{(p,0)}{(y,s)} \le \inner{(p,0)}{(\hat x,\hat t)}$ for all $(y,s) \in Q$. Meanwhile, Case ii directly follows from Lemma \ref{lem:sep}.
\QEDA

\begin{lemma}\label{silif:lem:sep2}
    Evaluation and proximal oracles of $f$, and membership and separation oracles of $K$ give a valid separation oracle for $Q$ as in \eqref{silif:def:pi}.
\end{lemma}

\begin{proof}
    We consider a query point $(\hat x,\hat t) \in \R^d \times \R$ and the same two cases as in the proof of Lemma~\ref{silif:lem:sep}.
    The same argument applies to case i) here. For case ii), Lemma \ref{lem:proj-epi} gives a projection oracle for $\epi(f)$. Since $\epi(f)$ is closed and convex, it further implies a separation oracle for $\epi(f)$ through Lemma \ref{lem:proj-sep}.
\end{proof}

\subsection{Technical results}
\label{appen:generaltechnical}

\begin{lemma}  
\label{lem:parallelsetcontainment} Let $\phi:\R^{m}\to\R$ be $L$-Lipschitz continuous and $\mathrm{epi}(\phi)$ be the epigraph of $\phi$, while $(\mathrm{epi}(\phi))_r$ is the associated parallel set. Then $(\mathrm{epi}(\phi))_r\ \subseteq\ \left\{(u,v):\ v\ \ge\ \phi(u)-r\sqrt{1+L^2}\right\}.$
\end{lemma}

\begin{proof}
If $(x,s)\in(\mathrm{epi}(\phi))_r$ then there exists $(y,t)\in\mathrm{epi}(\phi)$ with
$\|(x,s)-(y,t)\|\le r$ and $t\ge \phi(y)$. Put $\alpha:=\|x-y\|$, $\beta:=|t-s|$ so $\alpha^2+\beta^2\le r^2$.
Then
\[
\phi(x)-s \le |\phi(x)-\phi(y)| + (\phi(y)-s) \le L\alpha + (t-s) \le L\alpha + \beta
\le \sqrt{1+L^2} \sqrt{\alpha^2+\beta^2}\le r\sqrt{1+L^2},
\]
which is equivalent to $s\ge \phi(x)-r\sqrt{1+L^2}$.
\end{proof}

\begin{lemma} \label{lem:basicfacts}     
The following facts hold. 
\begin{enumerate}[label=\alph*)]
    \item Given $\eta>0$ and $\tau>0$, set $w_{\eta,\tau}(r):=\exp \left(-\frac{(r-\tau)^2}{2\eta}\right)$ and
let $C\in\R$. Then,
\begin{equation}\label{eq:kindofibp}
\int_{\tau}^{\infty} \frac{r-\tau}{\eta}e^{Cr}w_{\eta,\tau}(r) \D r
= C\int_{\tau}^{\infty} e^{Cr}w_{\eta,\tau}(r) \D r + e^{C\tau}.
\end{equation}
\item Let $a>0, \gamma\in\R$, $\psi:\R^d\to\R\cup\{+\infty\}$, and $F:\R^d\to[0,\infty]$ be measurable functions. Then,
\begin{equation}
    \label{eq:intoverQ}
   \int_{\R^d} F(x)\left(\int_{t\ge \psi(x)-\gamma} e^{-a t} \D t\right)\D x
=\frac{e^{a\gamma}}{a}\int_{\R^d} F(x) e^{-a\psi(x)} \D x.  
\end{equation}
\end{enumerate}

\end{lemma}

\begin{proof}
Part~a is due to integration by parts. Define $F_C(r):=e^{Cr}w_{\eta,\tau}(r)$. Then $F_C'(r)=Ce^{Cr}w_{\eta,\tau}(r)+e^{Cr}w_{\eta,\tau}'(r)
=\left(C-\tfrac{r-\tau}{\eta}\right)F_C(r)$
since $w_{\eta,\tau}'(r)=-(r-\tau)/\eta\;w_{\eta,\tau}(r)$.
Rearranging gives $\frac{r-\tau}{\eta}F_C(r)=C F_C(r)-F_C'(r)$.
Integrating over $[\tau,\infty)$ and using $F_C(r)\to 0$ as $r\to\infty$ and
$F_C(\tau)=e^{C\tau}$, we obtain \eqref{eq:kindofibp}. Part~b is due to $  \int_{\R^d} F(x)\left(\int_{t\ge \psi(x)-\gamma} e^{-a t} \D t\right)\D x
=\int_{\R^d} F(x) \left( \frac{1}{a}e^{-a(\psi(x)-\gamma)} \right) \D x \nonumber=\frac{e^{a\gamma}}{a}\int_{\R^d} F(x) e^{-a\psi(x)} \D x$. This completes the proof.
\end{proof}

\section{Supporting Results and Proofs for Section~\ref{sec:outer}}
\label{appen:outer}

\subsection{Technical lemmas}

We first establish the fixed-point claim of the proximal sampler.
\begin{lemma}
\label{lem:fixedpoint}
    $\pi$ is a fixed point of the proximal sampler \eqref{eq:proxsamplervariational}, in the sense that if $\rho_0^X=\pi$ then $\rho_k^X=\pi$ for all $k\geq 1$. Moreover in that case, we also have $\rho_k^Y=\pi P_\eta$ for all $k\geq 1$
\end{lemma}

\begin{proof}

Applying \eqref{eq:staticSBandregularizedW2} with
$\mu=\rho_0^X$ and $t=\eta$ gives $\frac{1}{2\eta}\cW_{2,2\eta}^2(\rho_0^X,\nu)
=
\inf_{\gamma\in\Pi(\rho_0^X,\nu)}
\KL\bigl(\gamma\|R^{\rho_0^X}_{0,\eta}\bigr)
-H(\rho_0^X)-\frac d2\log(2\pi\eta)$,
where $R^{\rho_0^X}_{0,\eta}(\D x\D y)
=
\rho_0^X(\D x)g_\eta(y-x)\D y$ and $g_t(z):=(2\pi t)^{-d/2}\exp(-\|z\|^2/(2t))$. 
Taking infimum over
$\nu$ yields the unique minimizer $\gamma=R^{\rho_0^X}_{0,\eta}$
whose second marginal is $\rho_k^X P_\eta$. This implies if we take $\rho_0^X=\pi$ then $\rho_1^Y=\pi P_\eta.$ Next, let us compute $\rho_1^X$. Writing $\pi(\D x)=Z^{-1}e^{-f(x)}\D x$, a similar calculation to \eqref{eq:staticSBandregularizedW2} will give $\int f \D\mu
+\frac{1}{2\eta}\cW_{2,2\eta}^2(\rho_k^Y,\mu)=
\inf_{\gamma\in\Pi(\rho_k^Y,\mu)}
\KL(\gamma\|\overline R_\eta)
-\log Z-\frac d2\log(2\pi\eta)$
where $\overline R_\eta(\D y\D x)
:=
\pi(\D x)g_\eta(y-x)\D y$. Minimizing over $\mu$, the unique minimizer is
$\gamma=\overline R_\eta$, whose second marginal is $\pi$. Hence, we arrive at $\rho_{1}^X=\pi$. Repeat this argument inductively to conclude $\rho_k^X=\pi$ for all $k\geq 1$. 
\end{proof}

In the next result, we verify the assumptions of Proposition~\ref{prop:phichainrule} for the backward potential $\Phi^{\mathrm b}$. 

\begin{lemma}
\label{lem:Phibbounds} 
Assume \eqref{cond:chisquare:initialandtarget}. Fix
$\varepsilon\in(0,\eta)$ and $0<t\le \eta-\varepsilon$. Recall the quantities $\pi_s^{\mathrm b}=P_{\eta-s}\pi$ and $\Phi_s^{\mathrm b}:=-\log\pi_s^{\mathrm b}$  for $0\le s\le t $ at respectively \eqref{eq:referencecurves} and \eqref{eq:backwardprox}. 
Then the following statements hold:
\begin{enumerate}
    \item[a)] $\Phi^{\mathrm b}\in C^\infty([0,t]\times\mathbb R^d)$, satisfies \eqref{eq:PhiHJchain} and \eqref{eq:Phigradgrowth}; 
\item[b)] there exists a narrowly continuous probability curve
$(\rho_s^{\star,\mathrm b})_{0\le s\le t}$ such that \eqref{assum:existnarrowcurve} holds in the distributional sense with $\alpha=\rho_k^Y$.  
\end{enumerate}

\end{lemma}
\begin{proof}
\underline{Part a:} The fact that $0\le s\le t\le\eta-\varepsilon$ implies $\eta-s\geq
\varepsilon$. Then
$\pi_s^{\mathrm b}=P_{\eta-s}\pi$ is strictly positive and smooth in $(s,x)$ on
$[0,t]\times\mathbb R^d$. Since $\pi_s^{\mathrm b}=P_{\eta-s}\pi$, it solves the
backward heat
equation $\partial_s\pi_s^{\mathrm b}+\frac12\Delta\pi_s^{\mathrm b}=0$.
Moreover, since $\pi_s^{\mathrm b}=e^{-\Phi_s^{\mathrm b}}$, we have $
\frac{\partial_s\pi_s^{\mathrm b}}{\pi_s^{\mathrm b}}=-\partial_s\Phi_s^{\mathrm b}$ and
$\frac{\Delta\pi_s^{\mathrm b}}{\pi_s^{\mathrm b}}
=
-\Delta\Phi_s^{\mathrm b}+\|\nabla\Phi_s^{\mathrm b}\|^2$. The equation
$\partial_s\pi_s^{\mathrm b}+\frac12\Delta\pi_s^{\mathrm b}=0$ therefore implies
\eqref{eq:PhiHJchain}.
Next, let us show \eqref{eq:Phigradgrowth}. Let $\tau_s:=\eta-s$.
Since $0\le s\le t\le\eta-\varepsilon$, we have
$\varepsilon\le\tau_s\le\eta$. Choose $R>0$ such that $a_R:=\int_{B_R}\pi(z)\D z>0$.
For $z\in B_R$, we have $\|x-z\|\le \|x\|+R$. Also since
$\varepsilon\le\tau_s\le\eta$, we have
$(2\pi\tau_s)^{-d/2}\ge(2\pi\eta)^{-d/2}$ and
$\exp(-\|x-z\|^2/(2\tau_s))\ge
\exp(-(\|x\|+R)^2/(2\varepsilon))$. Therefore, for every
$x\in\mathbb R^d$ and $0\le s\le t$,
\begin{align}
\label{eq:nulower}
\pi_s^{\mathrm b}(x)
&=
P_{\eta-s}\pi(x)
=
\int_{\mathbb R^d} g_{\tau_s}(x-z)\pi(z)\D z
\ge
\int_{B_R}g_{\tau_s}(x-z)\pi(z)\D z \nonumber \\
&\ge
(2\pi\eta)^{-d/2}
\exp\left(-\frac{(\|x\|+R)^2}{2\varepsilon}\right)
\int_{B_R}\pi(z)\D z
=
(2\pi\eta)^{-d/2}
a_R
\exp\left(-\frac{(\|x\|+R)^2}{2\varepsilon}\right).
\end{align}
Combining \eqref{eq:nulower} and $\Phi_s^{\mathrm b}:=-\log\pi_s^{\mathrm b}$ to get
\[
\Phi_s^{\mathrm b}(x)
=
-\log\pi_s^{\mathrm b}(x)
\le
\frac d2\log(2\pi\eta)
-
\log a_R
+
\frac{(\|x\|+R)^2}{2\varepsilon}
\le
C(1+\|x\|^2).
\]
On the other hand, since $\tau_s\ge\varepsilon$, we have $\pi_s^{\mathrm b}(x)
=
P_{\tau_s}\pi(x)
\le
\|g_{\tau_s}\|_\infty\int \pi
\le
(2\pi\varepsilon)^{-d/2}$ and thus $\Phi_s^{\mathrm b}(x)
=
-\log\pi_s^{\mathrm b}(x)
\ge
d\log\bigl(2\pi\varepsilon\bigr)/2$. 
Combining the upper and lower bounds on $\Phi_s^{\mathrm b}$, there
exists a constant $C$ independent of $s$ such that
\begin{equation}
\label{eq:Phibbound}
|\Phi_s^{\mathrm b}(x)|
\le
C(1+\|x\|^2),
\qquad
0\le s\le t,\ x\in\mathbb R^d .
\end{equation}
Next, for a fixed $\varepsilon>0$, define $\widetilde\pi_\theta:=P_{\varepsilon/2+\theta}\pi$ for $\theta\ge0$. 
Then $\widetilde\pi_\theta$ is a positive bounded solution of the heat
equation $ \partial_\theta\widetilde\pi_\theta
=
\frac12\Delta\widetilde\pi_\theta$
such that $0<\widetilde\pi_\theta(x)\le
M_\varepsilon:=(\pi\varepsilon)^{-d/2}$. Per
\cite[Theorem~1]{kotschwar2007hamilton}, we have the gradient estimate $\|\nabla\log \widetilde\pi_\theta(x)\|^2
\le
\frac{2}{\theta}
\log\frac{M_\varepsilon}{\widetilde\pi_\theta(x)}$ for $\theta>0$. Taking $\theta=\eta-s-\varepsilon/2$, we have
$\theta\ge\varepsilon/2$ since $s\le t\le\eta-\varepsilon$. Then
$\widetilde\pi_\theta=P_{\eta-s}\pi=\pi_s^{\mathrm b}$. Before applying the gradient
estimate, let us bound the logarithmic term. By \eqref{eq:nulower},
\[
\log\frac{M_\varepsilon}{\pi_s^{\mathrm b}(x)}
\le
\log M_\varepsilon
+
\frac d2\log(2\pi\eta)
-
\log a_R
+
\frac{(\|x\|+R)^2}{2\varepsilon}
\le
C(1+\|x\|^2).
\]
Therefore, the gradient estimate implies
$\|\nabla\Phi_s^{\mathrm b}(x)\|^2
=
\|\nabla\log \pi_s^{\mathrm b}(x)\|^2
\le
\frac{4}{\varepsilon}
\log\frac{M_\varepsilon}{\pi_s^{\mathrm b}(x)}
\le
C(1+\|x\|^2)$. 
Combining this with \eqref{eq:Phibbound}, we get the desired growth bound \eqref{eq:Phigradgrowth}.

\underline{Part b:}  Since $0<t\le\eta-\varepsilon$, we have $\eta-s\ge\varepsilon$ for $0\le
s\le t$. Hence $\pi_s^{\mathrm b}=P_{\eta-s}\pi$ is strictly positive and smooth in
$(s,x)$ on $[0,t]\times\mathbb R^d$, and so $\Phi_s^{\mathrm
b}=-\log\pi_s^{\mathrm b}$ is smooth on $[0,t]\times\mathbb R^d$. Set $b_s(x):=-\nabla\Phi_s^{\mathrm b}(x)$. 
By Part a, $ \|b_s(x)\|^2
=
\|\nabla\Phi_s^{\mathrm b}(x)\|^2
\le
C(1+\|x\|^2)$ for $ 0\le s\le t$.
Thus, $b_s$ is locally Lipschitz in $x$ and has at most linear growth.
Therefore, the SDE
\[
\D X_s^{\mathrm b}=b_s(X_s^{\mathrm b}) \D s+\D B_s,
\qquad
X_0^{\mathrm b}\sim\rho_k^Y,
\]
admits a non-explosive weak solution on $[0,t]$
(see \cite{karatzas2014brownian}). Let $P^{\star,\mathrm b}$ be the law of
$X^{\mathrm b}$ on $C([0,t];\mathbb R^d)$, and define $\rho_s^{\star,\mathrm b}(\cdot)
=P^{\star,\mathrm b}(X_s^{\mathrm b}\in \cdot),u_s^{\star,\mathrm b}:=b_s=-\nabla\Phi_s^{\mathrm b} $. 
Then $\rho_0^{\star,\mathrm b}=\rho_k^Y$. The curve
$s\mapsto\rho_s^{\star,\mathrm b}$ is narrowly continuous: if $f\in
C_b(\mathbb R^d)$ and $s_n\to s$, then path continuity implies
$X_{s_n}^{\mathrm b}\to X_s^{\mathrm b}$ $P^{\star,\mathrm b}$-a.s., and
the dominated convergence theorem gives
\[
\int f \D\rho_{s_n}^{\star,\mathrm b}
=
\int f(X_{s_n}^{\mathrm b}) \D P^{\star,\mathrm b}
\to
\int f(X_s^{\mathrm b}) \D P^{\star,\mathrm b}
=
\int f \D\rho_s^{\star,\mathrm b}.
\]
Finally, by It\^o's formula, for every $\zeta\in C_c^\infty(\mathbb
R^d)$ and every $0\le r\le s\le t$,
\[
\int \zeta \D\rho_s^{\star,\mathrm b}
-
\int \zeta \D\rho_r^{\star,\mathrm b}
=
\int_r^s\int
\left(
b_\ell\cdot\nabla\zeta+\frac12\Delta\zeta
\right)\D\rho_\ell^{\star,\mathrm b}\D \ell .
\]
Since $b_\ell=-\nabla\Phi_\ell^{\mathrm b}$, this is the weak
formulation of \eqref{assum:existnarrowcurve}. This completes the proof.
\end{proof}

The proof of Proposition~\ref{prop:phichainrule} needs the following technical formula. 
\begin{lemma}
\label{lem:chainPhi}
    For $(\rho,u)\in\mathcal A_t(\alpha,\sigma)$ defined at \eqref{def:classAt} and $(\Phi_s)_{\{s\geq 0\}}$ satisfying \eqref{eq:PhiHJchain} and \eqref{eq:Phigradgrowth},  it holds that 
\begin{equation}
\label{eq:chainPhi}
\int \Phi_t \D\sigma
-
\int \Phi_0 \D\alpha
=
\int_0^t\int
\left(
u_s\cdot\nabla\Phi_s
+
\frac12\|\nabla\Phi_s\|^2
\right)\D\rho_s \D s.
\end{equation}
\end{lemma}

\begin{proof}
\underline{Step 1:} The same argument showing \eqref{eq:finitecost} also leads to $\sup_{0\le s\le t}\int \|x\|^2 \D\rho_s(x)<\infty$, 
which combined with \eqref{eq:Phigradgrowth},\eqref{eq:finiteactionandpde} and the Cauchy-Schwarz inequality implies
\begin{equation}
\label{eq:gradint}
\int_0^t\int \|\nabla\Phi_s\|^2 \D\rho_s \D s<\infty,\qquad \int_0^t\int |u_s\cdot\nabla\Phi_s| \D\rho_s \D s<\infty. 
\end{equation}
%
%
%
Now let us set up some truncation
definitions. Assume $\chi_R\in C_c^\infty(\mathbb R^d)$ satisfy
\[
\chi_R=1 \text{ on } B_R,
\qquad
\chi_R=0 \text{ outside } B_{2R},
\qquad
\|\nabla\chi_R\|\le C/R,
\qquad
|\Delta\chi_R|\le C/R^2 .
\]
Let $h_N\in C^\infty(\mathbb R)$ satisfy
\begin{equation}
\label{def:hN}
h_N(r)=r \text{ for } |r|\le N,
\qquad
|h_N(r)|\le |r|,
\qquad
\|h_N\|_\infty\le C_N,
\end{equation}
\[
|h_N'(r)|\le1,
\qquad
|h_N''(r)|\le C/N,
\]
and
\[
h_N(r)\to r,
\qquad
h_N'(r)\to1
\qquad
\text{pointwise as }N\to\infty.
\]
For instance, for $\chi_R$, choose compactly supported and $C^\infty$
function $\chi\in C_c^\infty(\mathbb R^d)$ with $\chi=1$ on $B_1$ and
$\chi=0$ on $B_2$ then set $\chi_R(x):=\chi(x/R)$. Meanwhile, for $h_N$,
choose an odd and $C^\infty$ function $h$ such that $h(r)=r$ for
$|r|\le1$ and $h(r)=2 \operatorname{sgn}(r)$ for $|r|\ge3$; also that
$|h(r)|\le |r|, |h'(r)|\le1$; then set $h_N(r):=N h(r/N)$.
Next, set
\[
\psi_{R,N}(s,x):=\chi_R(x)h_N(\Phi_s(x)).
\]
Since $\Phi_s$ is neither compactly supported nor bounded, the factor
$\chi_R$ ensures that $\psi_{R,N}(s,x)$ is compactly supported in $x$,
while $h_N$ truncates $\Phi_s$ so that $\psi_{R,N}$ is bounded with
controlled derivatives.

Next, $\Phi_s$ is smooth on
$[0,t]\times\mathbb R^d$ and $\psi_{R,N}$ is an admissible compactly
supported test function in $x$. Then the weak form of the Fokker-Planck equation in \eqref{eq:finiteactionandpde} (using the test function $\psi_{R,N}$) says
\begin{align*}
&\int \chi_R h_N(\Phi_t) \D\sigma
-
\int \chi_R h_N(\Phi_0) \D\alpha \\
&=
\int_0^t\int
\left[
\partial_s\bigl(\chi_R h_N(\Phi_s)\bigr)
+
u_s\cdot\nabla\bigl(\chi_R h_N(\Phi_s)\bigr)
+
\frac12\Delta\bigl(\chi_R h_N(\Phi_s)\bigr)
\right]\D\rho_s \D s \\
&=
\int_0^t\int
\chi_R h_N'(\Phi_s)
\left(
\partial_s\Phi_s
+
u_s\cdot\nabla\Phi_s
+
\frac12\Delta\Phi_s
\right)\D\rho_s \D s
+
\frac12\int_0^t\int
\chi_R h_N''(\Phi_s)\|\nabla\Phi_s\|^2 \D\rho_s \D s \\
&\quad
+
\int_0^t\int
h_N(\Phi_s)u_s\cdot\nabla\chi_R \D\rho_s \D s
+
\int_0^t\int
h_N'(\Phi_s)\nabla\chi_R\cdot\nabla\Phi_s \D\rho_s \D s
+
\frac12\int_0^t\int
h_N(\Phi_s)\Delta\chi_R \D\rho_s \D s.
\end{align*}
It follows from \eqref{eq:PhiHJchain} that the above equation  becomes
\begin{align}
\label{eq:cutoffchainHJ}
&\int \chi_R h_N(\Phi_t) \D\sigma
-
\int \chi_R h_N(\Phi_0) \D\alpha \\
&=
\int_0^t\int
\chi_R h_N'(\Phi_s)
\left(
u_s\cdot\nabla\Phi_s
+
\frac12\|\nabla\Phi_s\|^2
\right)\D\rho_s \D s
+
\frac12\int_0^t\int
\chi_R h_N''(\Phi_s)\|\nabla\Phi_s\|^2 \D\rho_s \D s
+
\operatorname{Err}_{R,N}, \nn
\end{align}
where
\[
\operatorname{Err}_{R,N}
:=
\int_0^t\int
h_N(\Phi_s)u_s\cdot\nabla\chi_R \D\rho_s \D s
+
\int_0^t\int
h_N'(\Phi_s)\nabla\chi_R\cdot\nabla\Phi_s \D\rho_s \D s 
+
\frac12\int_0^t\int
h_N(\Phi_s)\Delta\chi_R \D\rho_s \D s.
\]
\underline{Step 2:} For fixed $N$, we will show
\[
\lim_{R\to\infty} \operatorname{Err}_{R,N}=0\]
by bounding the three terms in
$\operatorname{Err}_{R,N}$. From \eqref{def:hN}, we have $\left|
\int_0^t\int
h_N(\Phi_s)u_s\cdot\nabla\chi_R \D\rho_s \D s
\right|
\le
\frac{C C_N}{R}
\int_0^t\int \|u_s\| \D\rho_s \D s$, so by the Cauchy-Schwarz inequality,
\[
\left|
\int_0^t\int
h_N(\Phi_s)u_s\cdot\nabla\chi_R \D\rho_s \D s
\right|
\le
\frac{C C_N\sqrt t}{R}
\left(
\int_0^t\int \|u_s\|^2 \D\rho_s \D s
\right)^{1/2}
\to0.
\]
Next, since $\left|
\int_0^t\int
h_N'(\Phi_s)\nabla\chi_R\cdot\nabla\Phi_s \D\rho_s \D s
\right|
\le
\frac{C}{R}
\int_0^t\int_{B_{2R}\setminus B_R}
\|\nabla\Phi_s\| \D\rho_s \D s $, we have via the Cauchy-Schwarz's inequality
and the first part of \eqref{eq:gradint},
\[
\left|
\int_0^t\int
h_N'(\Phi_s)\nabla\chi_R\cdot\nabla\Phi_s \D\rho_s \D s
\right|
\le
\frac{C\sqrt t}{R}
\left(
\int_0^t\int_{B_{2R}\setminus B_R}
\|\nabla\Phi_s\|^2 \D\rho_s \D s
\right)^{1/2}
\to0. 
\]
Finally,
\[
\left|
\frac12\int_0^t\int
h_N(\Phi_s)\Delta\chi_R \D\rho_s \D s
\right|
\le
\frac{C C_N}{R^2}\int_0^t\rho_s(B_{2R}\setminus B_R) \D s
\le
\frac{C C_Nt}{R^2}
\to0.
\]
Thus, $\lim_{R\to\infty} \operatorname{Err}_{R,N}=0$. For fixed $N$, we
now let $R\to\infty$ in \eqref{eq:cutoffchainHJ}. Given that $|h_N'|\le
1, \|h_N\|_\infty<\infty, \|h_N''\|_\infty<\infty$ from \eqref{def:hN},
$\int_0^t\int \|\nabla\Phi_s\|^2 \D\rho_s \D s<\infty$ from the first part of 
\eqref{eq:gradint}, $\int_0^t\int \|u_s\|^2 \D\rho_s \D s<\infty$ from the
assumption \eqref{eq:finiteactionandpde},
the terms on both sides of \eqref{eq:cutoffchainHJ} beside
$\operatorname{Err}_{R,N}$ converge by dominated convergence. For the
remainder term, we know $\lim_{R\to\infty} \operatorname{Err}_{R,N}=0$.
Then the limit of \eqref{eq:cutoffchainHJ} is
\begin{align}
\label{eq:truncchainR}
&\int h_N(\Phi_t) \D\sigma
-
\int h_N(\Phi_0) \D\alpha \nn \\
&=
\int_0^t\int
h_N'(\Phi_s)
\left(
u_s\cdot\nabla\Phi_s
+
\frac12\|\nabla\Phi_s\|^2
\right)\D\rho_s \D s
+
\frac12\int_0^t\int
h_N''(\Phi_s)\|\nabla\Phi_s\|^2 \D\rho_s \D s .
\end{align}
Now let $N\to\infty$ in \eqref{eq:truncchainR}. By
\eqref{eq:Phigradgrowth} and $\sigma,\alpha\in\mathcal
P_2(\mathbb R^d)$, the dominated convergence theorem says $ \int h_N(\Phi_t)
\D\sigma\to \int \Phi_t \D\sigma$ and $\int h_N(\Phi_0) \D\alpha\to \int
\Phi_0 \D\alpha $.
Moreover, $h_N'(\Phi_s)
\left(
u_s\cdot\nabla\Phi_s+\frac12\|\nabla\Phi_s\|^2
\right)
\to
u_s\cdot\nabla\Phi_s+\frac12\|\nabla\Phi_s\|^2$
pointwise, and the absolute value is bounded by $
|u_s\cdot\nabla\Phi_s|+\frac12\|\nabla\Phi_s\|^2$
which is integrable by
\eqref{eq:gradint}. Hence, the dominated convergence theorem applies to the
$h_N'$-term on the right-hand side of \eqref{eq:truncchainR}. Finally,
by the first part of \eqref{eq:gradint}, $\left|
\frac12\int_0^t\int
h_N''(\Phi_s)\|\nabla\Phi_s\|^2 \D\rho_s \D s
\right|
\le
\frac{C}{2N}
\int_0^t\int \|\nabla\Phi_s\|^2 \D\rho_s \D s
\to0$. 
Therefore \eqref{eq:chainPhi} follows from taking limit $N\to\infty$
of \eqref{eq:truncchainR}. This completes the proof. 
\end{proof}

\subsection{Supporting propagation results for integrability and finiteness} 
\label{sec:propagation}

This subsection collects supporting results on the propagation of integrability and finiteness for $\chi^2$ and R\'{e}nyi divergences.

\begin{lemma}
\label{lem:forward:chisquare}
Let $\pi$ be a probability measure on $\mathbb R^d$, and set $\rho=f\pi,
f=\frac{\D\rho}{\D\pi}.$ Let $m,n>1$ satisfy $\frac1m+\frac1n=1$ and assume
\[
f\in L^{2m}(\pi),
\qquad
\int_{\mathbb R^d}\|x\|^{2n}\pi(\D x)<\infty.
\]
Define the heat-flow marginals $\rho_t=\rho P_t,
\pi_t=\pi P_t$ and also $f_t=\frac{\rho_t}{\pi_t}$ for $t>0$. 
Then
\begin{enumerate}[label=\alph*)]
\item $f_t\in L^2(\pi_t)$ and $\chi^2(\rho_t\|\pi_t)<\infty$;
\item $\nabla f_t\in L^2(\pi_t)$.
\end{enumerate}
\end{lemma}
\begin{proof}
Let $g_t$ denote the density of $N(0,tI)$ independent of $\pi$ and $\rho$.
Then the heat flows are $\rho_t(y)
=
\int_{\mathbb R^d}g_t(y-x)f(x)\pi(\D x)$ and $
\pi_t(y)
=
\int_{\mathbb R^d}g_t(y-x)\pi(\D x).$
Let $Z\sim\mathcal N(0,I)$ be independent of $X\sim\pi$, and set $Y_t:=X+\sqrt t Z$. Following~\cite{Klartagputterman2024}, define the
posterior operator $Q_t$ by
\[
Q_t h(X+\sqrt t Z)=\E[h(X)\mid Y_t].
\]
Bayes' formula tells us
\begin{equation}
\label{consequence:bayes:chisquare}
f_t(y)=Q_t f(y)=\E_{\pi}[f(X)\mid Y_t=y]=\frac{
\int_{\mathbb R^d}g_t(y-x)f(x) \pi(\D x)
}{
\int_{\mathbb R^d}g_t(y-x) \pi(\D x)
}.
\end{equation}
\underline{Proof of a:} Since $m>1$, the assumption $f\in L^{2m}(\pi)$
implies $f\in L^2(\pi)$. By Jensen's inequality for conditional
expectation and the fact that $x\mapsto x^2$ is convex, $\int f_t^2 \D\pi_t
=
\E\!\left[\E[f(X)\mid Y_t]^2\right]
\le
\E[f(X)^2]
=
\int f^2 \D\pi
<\infty$. 
Thus, $\chi^2(\rho_t\|\pi_t)
=
\int f_t^2 \D\pi_t-1
<\infty$. 

\smallskip
\underline{Proof of b:} Using \eqref{consequence:bayes:chisquare} and the fact that $\nabla_y g_t(y-x)
=
\frac{x-y}{t}g_t(y-x)$, we have
\begin{align*}
\nabla f_t(y)
&=
\frac{1}{t}
\left(
\E_\pi[f(X)X\mid Y_t=y]
-
f_t(y)\E_\pi[X\mid Y_t=y]
\right)\nn\\
& =
\frac{1}{t}
\E_\pi
\left[
\bigl(f(X)-f_t(y)\bigr)
\bigl(X-\E_\pi[X\mid Y_t=y]\bigr)
\mid Y_t=y
\right].
\end{align*}
By Cauchy-Schwarz inequality, $\|\nabla f_t(y)\|^2
\le
\frac{1}{t^2}
\E_\pi[(f(X)-f_t(y))^2\mid Y_t=y]
\cdot\E_\pi[\|X-\E_\pi[X\mid Y_t=y]\|^2\mid Y_t=y]$. 
Combining with the fact that $\E\bigl[(A-\E A)^2\bigr]
=
\E[A^2]-(\E A)^2
\le
\E[A^2]$, we obtain $\|\nabla f_t(y)\|^2
\le
\frac{1}{t^2}
\E_\pi[f(X)^2\mid Y_t=y]\cdot
\E_\pi[\|X\|^2\mid Y_t=y]$. 
By H\"older's inequality with exponents $m,n>1$ satisfying $1/m+1/n=1$,
\[
\begin{aligned}
\int_{\mathbb R^d} \|\nabla f_t\|^2(y) \D\pi_t(y)
&\le
\frac{1}{t^2}
\int
\E_\pi[f(X)^2\mid Y_t]
\cdot\E_\pi[\|X\|^2\mid Y_t] \D\pi_t \\
&\le
\frac{1}{t^2}
\left(
\int
\E_\pi[f(X)^2\mid Y_t]^m \D\pi_t
\right)^{1/m}
\left(
\int
\E_\pi[\|X\|^2\mid Y_t]^n \D\pi_t
\right)^{1/n}.
\end{aligned}
\]
Jensen's inequality and our assumptions give $\int
\E_\pi[f(X)^2\mid Y_t]^m \D\pi_t
\le
\int f^{2m}\D\pi<\infty$ and $\int
\E_\pi[\|X\|^2\mid Y_t]^n \D\pi_t
\le
\int \|x\|^{2n}\D\pi<\infty$. 
Thus, we get $\int \|\nabla f_t\|^2 \D\pi_t<\infty$.
This completes the proof.
\end{proof}
\begin{lemma}
\label{lem:backward:chisquare}
Denote $P_t,t\geq 0$ the heat semigroups. Let $\eta>0$ and for
$t\in[0,\eta]$, $\pi_t^{\mathrm b}:=\pi P_{\eta-t}, \pi_0^{\mathrm b}=\pi P_\eta=:\pi_\eta,$ and also $h=\frac{\rho_0^{\mathrm b}}{\pi_0^{\mathrm b}}$. 
Let $m,n>1$ satisfy $\frac1m+\frac1n=1$ and assume
\[
h\in L^{2m}(\pi_0^{\mathrm b}),
\qquad
\int_{\mathbb R^d}\|x\|^{2n} \pi_0^{\mathrm b}(\D x)<\infty.
\]
Define $\rho_t^{\mathrm b}(x)=\int_{\mathbb R^d}\rho_0^{\mathrm b}(y)q_t(x\mid y) \D y$ and $
h_t:=\frac{\rho_t^{\mathrm b}}{\pi_t^{\mathrm b}}$ for $0<t\le\eta$,
where $ q_t(x\mid y)
=
\frac{\pi_t^{\mathrm b}(x)g_t(y-x)}{\pi_0^{\mathrm b}(y)}$ and $g_t$ denotes the density of $N(0,tI)$. Then, for $0<t\le\eta$,
\begin{enumerate}[label=\alph*)]
\item $h_t\in L^2(\pi_t^{\mathrm b})$ and $\chi^2(\rho_t^{\mathrm b}\|\pi_t^{\mathrm b})<\infty$,
\item $\nabla h_t\in L^2(\pi_t^{\mathrm b})$.
\end{enumerate}
\end{lemma}


\begin{proof}
It follows from $\rho_0^{\mathrm b}=h\pi_0^{\mathrm b}$ that $\rho_t^{\mathrm b}(x)
=
\int_{\mathbb R^d}h(y)\pi_0^{\mathrm b}(y)
\frac{\pi_t^{\mathrm b}(x)g_t(y-x)}{\pi_0^{\mathrm b}(y)} \D y=
\pi_t^{\mathrm b}(x)\int_{\mathbb R^d}g_t(y-x)h(y) \D y$.
Thus, 
\begin{equation}
\label{eq:pushratiobackward}
h_t(x)=\frac{\rho_t^{\mathrm b}(x)}{\pi_t^{\mathrm b}(x)}
=
\int_{\mathbb R^d}g_t(y-x)h(y) \D y
=
(P_t h)(x).
\end{equation} 
Moreover, if $X_t\sim\pi_t^{\mathrm b}=\pi P_{\eta-t}$, $Z\sim\mathcal N(0,I)$ is
independent of
$X_t$, and set $Y_t:=X_t+\sqrt t Z$
then $Y_t\sim\pi_0^{\mathrm b}$. Combining with \eqref{eq:pushratiobackward} to get $h_t(X_t)
=
P_t h(X_t)
=
\E[h(Y_t)\mid X_t]$. 
Now Parts (a) and (b) follow by arguments similar to those in the proof of Lemma~\ref{lem:forward:chisquare}.
\end{proof}

\begin{lemma}
\label{lem:propagation-chisquare-regularity}
Let $m,n>1$ satisfy $\frac1m+\frac1n=1$ and assume
\[
f_0^X:=\frac{\rho_0^X}{\pi}\in L^{2m}(\pi),
\qquad
\int_{\mathbb R^d}\|x\|^{2n}\pi(\D x)<\infty.
\]
Let $P_\eta$ denote the heat semigroups, and define $\rho_k^Y:=\rho_k^X P_\eta,
\pi_0^{\mathrm b}:=\pi P_\eta$. Then, for every $k\ge0$, it holds that
\[
f_k^X:=\frac{\rho_k^X}{\pi}\in L^{2m}(\pi),
\qquad
f_k^Y:=\frac{\rho_k^Y}{\pi_0^{\mathrm b}}\in L^{2m}(\pi_0^{\mathrm b}) \quad\text{and}\quad \int_{\mathbb R^d}\|x\|^{2n}\pi_0^{\mathrm b}(\D x)<\infty.
\]
\end{lemma}

\begin{proof}
Since $\pi_0^{\mathrm b}=\pi
P_\eta$, if $X\sim\pi$ and $Z\sim N(0,I)$ is independent of $X$, then $X+\sqrt\eta Z\sim \pi_0^{\mathrm b}$.
Combining $\int_{\mathbb R^d}\|x\|^{2n}\pi_0^{\mathrm b}(\D x)
=
\E\|X+\sqrt\eta Z\|^{2n}$ and $\|X+\sqrt\eta Z\|^{2n}
\le
C_n\left(\|X\|^{2n}+\eta^n\|Z\|^{2n}\right)$ and the moment assumption on $\pi$, we obtain $\int_{\mathbb R^d}\|x\|^{2n}\pi_0^{\mathrm b}(\D x)
\le
C_n\left(
\int_{\mathbb R^d}\|x\|^{2n}\pi(\D x)+\eta^n\E\|Z\|^{2n}
\right)
<\infty$. 
Finally, we propagate the $L^{2m}$-integrability of the density
ratios. We argue by induction. The base case $f_0^X\in L^{2m}(\pi)$ is our
assumption. Now suppose as our induction hypothesis that for some $k\ge0$, $f_k^X\in L^{2m}(\pi)$. Regarding the forward heat step, the same argument as
in Lemma~\ref{lem:forward:chisquare} leads to
\begin{equation}
\label{eq:chisquare-holder-forwardstep}
f_k^Y\in L^{2m}(\pi_0^{\mathrm b}).
\end{equation}
Indeed, writing $Y_\eta=X+\sqrt\eta Z$, the posterior representation
gives $f_k^Y(Y_\eta)
=
\E_\pi[f_k^X(X)\mid Y_\eta]$ and thus $\int (f_k^Y)^{2m} \D\pi_0^{\mathrm b}
=
\E\!\left[
\E_\pi[f_k^X(X)\mid Y_\eta]^{2m}
\right]
\le
\E_\pi[(f_k^X)^{2m}]
=
\int (f_k^X)^{2m} \D\pi
<\infty$. We consider next the backward step. Set $h_k:=f_k^Y=\frac{\rho_k^Y}{\pi_0^{\mathrm b}}$ and 
recall the interpolation $\pi_t^{\mathrm b}, \rho_t^{\mathrm b}$ from Lemma~\ref{lem:backward:chisquare}.  Then
\[
\rho_t^{\mathrm b}(x)
=
\int_{\mathbb R^d}
\rho_k^Y(y)
\frac{\pi_t^{\mathrm b}(x)g_t(y-x)}{\pi_0^{\mathrm b}(y)}
\D y=
\pi_t^{\mathrm b}(x)
\int_{\mathbb R^d}
g_t(y-x)
\frac{\rho_k^Y(y)}{\pi_0^{\mathrm b}(y)}
\D y 
=
\pi_t^{\mathrm b}(x)(P_t h_k)(x).
\]
In particular,  $f_{k+1}^X
=
\frac{\rho_{k+1}^X}{\pi}
=
\frac{\rho_\eta^{\mathrm b}}{\pi_\eta^{\mathrm b}}
=
P_\eta h_k$
for $h_k:=f_k^Y=\frac{\rho_k^Y}{\pi_0^{\mathrm b}}$. Applying conditional Jensen's inequality and
\eqref{eq:chisquare-holder-forwardstep}, we obtain
\begin{equation*}
\int (f_{k+1}^X)^{2m} \D\pi
=
\int
\left(
\frac{\rho_\eta^{\mathrm b}}{\pi_\eta^{\mathrm b}}
\right)^{2m}
\D\pi_\eta^{\mathrm b} =
\int (P_\eta h_k)^{2m} \D\pi
\le
\int h_k^{2m} \D\pi_0^{\mathrm b}
=
\int (f_k^Y)^{2m} \D\pi_0^{\mathrm b}
<\infty.
\end{equation*}
Hence, $f_{k+1}^X\in L^{2m}(\pi)$. This closes the induction and completes the proof.
\end{proof}


The following three results for R\'{e}nyi divergence are similar to Lemmas \ref{lem:forward:chisquare}-\ref{lem:propagation-chisquare-regularity} for $\chi^2$ divergence.
Hence, their proofs are omitted.

\begin{lemma}
\label{lem:forward:renyi}
Let $\pi$ be a probability measure on $\mathbb R^d$, and let $\rho=f\pi,
f=\frac{\D\rho}{\D\pi}$. Fix $q\ge 2$. Let $m,n>1$ satisfy $\frac1m+\frac1n=1$. Assume
\[
f\in L^{qm}(\pi),
\qquad
\int_{\mathbb R^d}\|x\|^{2n}\pi(\D x)<\infty.
\]
Define the heat-flow marginals $\rho_t,\pi_t$, and $f_t$ as in Lemma~\ref{lem:forward:chisquare}. 
Then 
\begin{enumerate}[label=\alph*)]
\item $f_t\in L^q(\pi_t)$ and $\mathcal R_q(\rho_t\|\pi_t)<\infty$;
\item $\nabla f_t^{q/2}\in L^2(\pi_t)$.
\end{enumerate}
\end{lemma}

\begin{lemma}
\label{lem:backward:renyi}
Denote $P_t,t\geq 0$ the heat semigroups. Let $\eta>0$ and for
$t\in[0,\eta]$, $\pi_t^{\mathrm b}:=\pi P_{\eta-t},  \pi_0^{\mathrm b}=\pi P_\eta=:\pi_\eta$ and $h=\frac{\rho_0^{\mathrm b}}{\pi_0^{\mathrm b}}$. 
Fix $q\ge 2$. Let $m,n>1$ satisfy $\frac1m+\frac1n=1$. 
Assume
\[
h\in L^{qm}(\pi_0^{\mathrm b}),
\qquad
\int_{\mathbb R^d}\|x\|^{2n} \pi_0^{\mathrm b}(\D x)<\infty.
\]
Define $\rho_t^{\mathrm b}$ and $h_t$ as in Lemma~\ref{lem:backward:chisquare}.
Then for $0<t\le\eta$,
\begin{enumerate}[label=\alph*)]
\item $h_t\in L^q(\pi_t^{\mathrm b})$ and $\mathcal R_q(\rho_t^{\mathrm b}\|\pi_t^{\mathrm b})<\infty$;
\item $\nabla h_t^{q/2}\in L^2(\pi_t^{\mathrm b})$.
\end{enumerate}
\end{lemma}

\begin{lemma}
\label{lem:propagation-renyi-moment}
Fix $q\ge2$. Let $m,n>1$ satisfy $\frac1m+\frac1n=1$. 
Assume
\[
f_0^X:=\frac{\rho_0^X}{\pi}\in L^{qm}(\pi),
\qquad
\int_{\mathbb R^d}\|x\|^{2n}\pi(\D x)<\infty.
\]
Let $P_\eta,\rho_k^Y$ and $\pi_0^{\mathrm b}$  as in Lemma~\ref{lem:propagation-chisquare-regularity}.
Then, for every $k\ge0$,
\[
f_k^X:=\frac{\rho_k^X}{\pi}\in L^{qm}(\pi),
\qquad
f_k^Y:=\frac{\rho_k^Y}{\pi_0^{\mathrm b}}\in L^{qm}(\pi_0^{\mathrm b})\quad\text{and}\quad \int_{\mathbb R^d}\|x\|^{2n}\pi_0^{\mathrm b}(\D x)<\infty.
\]
Consequently, the hypotheses of Lemmas~\ref{lem:forward:renyi} and \ref{lem:backward:renyi} are available at every iteration.
\end{lemma}



\section{Supporting Results and Proofs for Section \ref{sec:constrained:mainpaper}}
\label{appen:constrained}

\subsection{Proof of Theorem~\ref{silif:theo:outer}}

Let us provide some important lemmas for the proof of Theorem~\ref{silif:theo:outer}. First, we show how to bound the PI constant of the lifted distribution~\eqref{silif:def:pi}.

  \begin{lemma}
\label{silif:lem:PIconstant}
 With respect to $\pi$ at~\eqref{silif:def:pi}, assume $a>0$ and $f$ is convex on a convex set $K$. Then $\pi$  satisfies a PI with  $  C_{\mathrm{PI}}(\pi)
    \leq C\log(d+1)
    \left(\opnorm{\cov{\nu}}+\frac{d+1}{a^2}\right)$. 
\end{lemma}
\begin{proof}
Let $(X,T)\sim\pi$ and the $X$-marginal is $\nu$. The following calculation follows the argument of ~\cite[Eq.~(1.1) and Lemma~1.3]{kook2026unified}, adapted here to a general lifting parameter $a>0$. For $\nu(x)=Z_\nu^{-1}e^{-f(x)}\mathbf 1_K(x)$, the varentropy of $\nu$ is  $\operatorname{Var}_\nu(f(X))=\operatorname{Var}\brac{-\log\nu(X)}$. 
Since $\nu$ is log-concave on $\R^d$, it follows from \cite[Theorem~2.3]{fradelizi2016optimal} (see also \cite{nguyenthesis,wang2014heatcapacity,bolley2018dimensional}) that
\begin{equation}
\label{silif:eq:varentropyinequal}
\operatorname{Var}_\nu(f(X))\leq d.
\end{equation}
We note that the sharp varentropy bound for log-concave densities was obtained earlier and independently by Nguyen~\cite{nguyenthesis} and Wang~\cite{wang2014heatcapacity}, while an alternative shorter proof appear in~\cite[Theorem~2.3]{fradelizi2016optimal}. See also \cite[Remark 4.2]{bolley2018dimensional}.

We further recall  \eqref{eq:Tmarginal} that $T\stackrel{\text{d}}=\frac{f(X)+E}{a}$ for $E\sim\operatorname{Exp}(1)$ independent of $X$. Combining with  \eqref{silif:eq:varentropyinequal} to get
\begin{equation}
\label{silif:eq:variance-T}
\operatorname{Var}_\pi(T)
=\frac{\operatorname{Var}_\nu(f(X))+1}{a^2}
\leq\frac{d+1}{a^2}.
\end{equation}
For $(u,s)\in\R^d\times\R$ with $\|u\|^2+s^2=1$, the triangle inequality and Cauchy-Schwarz inequality give
\begin{align*}
\sqrt{\operatorname{Var}_\pi(u^\top X+sT)}
&=
\left\|
u^\top \bigl(X-\mathbb E_\pi X\bigr)
+s\bigl(T-\mathbb E_\pi T\bigr)
\right\|_{L^2(\pi)}\notag
\\
&\leq
\left\|
u^\top \bigl(X-\mathbb E_\pi X\bigr)
\right\|_{L^2(\pi)}
+
|s|
\left\|
T-\mathbb E_\pi T
\right\|_{L^2(\pi)}\notag
\\
&=
\sqrt{u^\top \cov{\nu}u}
+
|s|\sqrt{\operatorname{Var}_\pi(T)}\notag
\\
&\leq
\sqrt{\opnorm{\cov{\nu}}}\|u\|
+\sqrt{\operatorname{Var}_\pi(T)}|s|\notag
\\
&\leq
\sqrt{\opnorm{\cov{\nu}}+\operatorname{Var}_\pi(T)}
\sqrt{\|u\|^2+s^2}=
\sqrt{\opnorm{\cov{\nu}}+\operatorname{Var}_\pi(T)}.
\end{align*}
Then via~\eqref{silif:eq:variance-T}, we have $\opnorm{\cov{(X,T)}}=\sup_{\|u\|^2+s^2=1}\operatorname{Var}_\pi(u^\top X+sT)
\leq
\opnorm{\cov{\nu}}+\frac{d+1}{a^2}$. 
Finally, under the assumption of $f$ and $K$, $\pi$ is log-concave. The bound in \cite[Remark~7.12]{klartag2025lecturenoteisoperimetric} leads to  $C_{\mathrm{PI}}(\pi)
\leq C\log(d+1)\opnorm{\cov{(X,T)}}$. 
The proof is completed after applying the previous covariance norm bound.
\end{proof}

\begin{remark}
    A long line of works \cite{cheeger2015lower,kannan1995isoperimetric,lee2024eldan,chen2021almost,klartag2022bourgain,klartag2023logarithmic} establish $C_{\mathrm{PI}}(\nu)=  \mathcal{O}\brac{\opnorm{\cov {\nu}}\log d}$ for a general log-concave distribution $\nu$. This matches the order of the bound on $C_{\mathrm{PI}}$ in Lemma~\ref{silif:lem:PIconstant}.
    \end{remark}

The next lemma shows that warm starts are preserved by the lifting and remain valid throughout the iterations of the proximal sampler~\eqref{silif:alg:ASF_pi}.

\begin{lemma}\label{silif:lem:warmstart} 
Under condition~(A3), if $\pi^{Y,S}$ denotes the law of $(Y,S)$ after Step~1 of Algorithm~\ref{silif:alg:ASF_pi}, and $\pi_k$ denotes the law of $(x_k,t_k)$, then $\D\pi^{Y,S}/d\Pi^{Y,S}\le M$ and $\D\pi_k/d\Pi^{X,T}\le M$ for $k\geq 0$. In other words, warmness holds for every step of Algorithm~\ref{silif:alg:ASF_pi}.
\end{lemma}

\begin{proof}  
The fact that $\D\pi_0/\D\pi\le M$ was shown in \eqref{silif:eq:warmstart}. 
Next, we assume any $U\subseteq \R^{d+1}$. For $(y,s)\in \R^{d+1}$, set $U-(y,s)=\{(x,t)\in \R^{d+1}:(x,t)+(y,s)\in U \}$. Denote $g$ the density of $\mathcal{N}(0,\eta I_{d+1})$. Then, using $\D \pi_0/\D \pi=\D \pi_0/\D \Pi^{X,T}\leq M$, $\pi^{Y,S} = \pi_0 * g$, and $\Pi^{Y,S} = \Pi^{X,T} * g$, we have
\begin{align*}
    \pi^{Y,S}(U)= \int_{\R^{d+1}} \pi_0(U-y) g(y,s)\D y\D s&=\int_{\R^{d+1}} \left( \int_{U-(y,s)} \frac{\D \pi_0}{\D \Pi^{X,T}}(x,t)  \Pi^{X,T}(\D x,\D t) \right) g(y,s) \D y\D s\\
    &\leq M\int_{\R^{d+1}} \Pi^{X,T}(U-(y,s))g(y,s) \D y\D s = M\Pi^{Y,S}(U).
\end{align*}
A similar argument yields $\frac{\D\pi_k}{\D\Pi^{X,T}}\leq M$. This completes the proof. 
\end{proof}

\textbf{Proof of Theorem~\ref{silif:theo:outer}:}
\label{silif:appen:proofouter}
\underline{Step 1:} Lemma~\ref{silif:lem:PIconstant} with $\nu$ in \eqref{silif:def:nu} provides the stated PI constant bound. 

\underline{Step 2:} In order to apply Theorem~\ref{theo:outer}, we first verify that their hypotheses \eqref{cond:chisquare:initialandtarget} and \eqref{cond:renyi:initialandtarget} hold for the lifted initial law $\pi_0$ and lifted target $\pi$. 
Per \eqref{silif:eq:warmstart}, we have $ \frac{\D\pi_0}{\D\pi}\le M$, so that $\D\pi_0/\D\pi\in L^p(\pi)$ for every $p<\infty$. Thus, the first assumptions in \eqref{cond:chisquare:initialandtarget} and \eqref{cond:renyi:initialandtarget} are satisfied. Next, we check the lifted moment condition. Let $(X,T)\sim\pi$ for $\pi$ in \eqref{silif:def:pi}, then \eqref{eq:Tmarginal} holds. By (A1), $K\subseteq R_\mathrm{out}\ball_d(0)$, so $\|X\|\le R_\mathrm{out}$ $\nu$-a.s. By (A2), $f$ is Lipschitz on the bounded set $K$, hence $f$ is bounded on $K$. Therefore $T$ has finite moments of all orders, and consequently $ \int_Q \|(x,t)\|^p\D \pi(x,t)<\infty$ for every $p<\infty$. 
Thus the hypotheses \eqref{cond:chisquare:initialandtarget} and \eqref{cond:renyi:initialandtarget} hold.

\underline{Step 3:} We apply Theorem~\ref{theo:outer} in this step. Step 2 already has $ \frac{\D\pi_0}{\D\pi}\le M$ which implies $\chi^2\brac{\pi_0||\pi}\leq M-1$. Then by applying Theorem~\ref{theo:outer}(a) iteratively,  we can solve for $k$ in  $ \frac{M}{\brac{1+\eta/C_{\mathrm{PI}}}^{2k}} \leq \varepsilon$  to get the stated iteration complexity with respect to $\chi^2$ divergence.

For $q\ge2$, the calculation for $q$-R\'{e}nyi divergence is along the same line with the use of Theorem~\ref{theo:outer}(b) and Lemma~\ref{silif:lem:PIconstant}. In particular, we see in Step 2 that $ \frac{\D\pi_0}{\D\pi}\le M$, so $\mathcal{R}_q\brac{\pi_0||\pi}=\frac{1}{q-1}\log\int(\D\pi_0/\D\pi)^q\D\pi\leq\frac{1}{q-1}\log\left(M^{q-1}\int(\D\pi_0/\D\pi)\D\pi\right)=\log M$.  Hence, by iteratively applying the first part of  \eqref{eq:renyicontractionconclu} in Theorem~\ref{theo:outer}(b), we can first solve for $k_0$ such that $\mathcal R_q(\pi_{k_0}\|\pi)\leq1$. This allows us to iteratively apply the second case of \eqref{eq:renyicontractionconclu}. The stated complexity bound in R\'{e}nyi divergence is a consequence of solving  $\brac{1+\eta/C_{\mathrm{PI}}}^{-2(k-k_0)/q}\leq\varepsilon$. This completes the proof.
\QEDA

\subsection{Proof of Theorem \ref{silif:theo:separation}(a)}
\label{silif:appen:cp}

We first record the cutting-plane method of \cite{lee2015fastercutting}, which builds on Vaidya's breakthrough volumetric center method \cite{vaidya}.  However, the output of the cutting-plane method in \cite{lee2015fastercutting} need not belong to $K$.  To guarantee feasibility, we combine it with an argument from \cite{sidford}. Specifically, we apply \cite[Theorem 42]{lee2015fastercutting} to the extended objective $\bar f=f$ on $K$ and $\bar f=+\infty$ outside $K$.  The objective- bound \eqref{ineq:opt} below applied to $\bar f$ implies $\bar f(\tilde x)<+\infty$, so $\tilde x\in K$. By \cite[Lemma 7]{sidford} and convexity of $f$, a separation-oracle call for $\bar f$ requires at most one separation-oracle call for $K$ and one subgradient-oracle call for $f$.

\begin{theorem}(\cite{vaidya,lee2015fastercutting})\label{theo:vaidya} 
Let $f$ be a convex function on a convex body $K$. Moreover, $K$ contains a minimizer
of $f$ and belongs to $B_\infty (0,R)$, where $B_\infty (0,R)$ denotes a ball of radius $R$ in $\ell_\infty$ norm, i.e., $\norm{x}_\infty=\sup_{1\leq i\leq d}\abs{x_i}$. 

Suppose we have a subgradient oracle for $f$ with cost $T$ and a separation oracle for $K$ with cost $S$. Using $B_\infty (0,R)$ as the initial polytope for our cutting-plane method, for any $0< \alpha<1$, we can compute $\tilde x\in K$ such that 
\begin{equation}\label{ineq:opt}
f(\tilde x)-\min_{x \in K} f(x) \le \alpha \left(\max_{x \in K} f(x)-\min_{x \in K} f(x)\right).    
\end{equation}
with an expected running time of $ {\mathcal{O}}\left(T\cdot d \log \frac{d \gamma}{\alpha}+S\cdot d \log \frac{d \gamma}{\alpha}+d^3\operatorname{polylog} \frac{d \gamma}{\alpha}\right)$. In particular, the expected number of subgradient oracle calls and separation oracle calls are
    $\mathcal{O}\brac{d \log \frac{d \gamma}{\alpha}}$, 
where $   \gamma=R/\mbox{minwidth}(K) $ and $\mbox{minwidth}(K)=\min_{\norm{a}=1}\left\{\max_{y\in K}a^\top y-\min_{y\in K}a^\top y\right\}$. 
\end{theorem}

In the upcoming result, we rewrite~\eqref{silif:optimizingThetaasprojection}, a strongly convex optimization problem over the unbounded set $Q$, as an equivalent strongly convex program over the compact set $K$.

\begin{lemma} 
\label{silif:lem:equivopt}
Assume conditions (A1)-(A2) and define  
\begin{equation}
\label{silif:def:zeta}
     \zeta^{\eta}_{y,s}(x)= \frac{1}{2\eta}\|x-y\|^2+ \frac{1}{2\eta}\left[ \frac{f(x)}{a} - (s - a\eta)\right]_{+}^2. 
\end{equation}
The following statements hold: 
\begin{itemize}
    \item[(a)] the solution to \eqref{silif:optimizingThetaasprojection} is  $x_* =\argmin_{x\in K} \zeta^{\eta}_{y,s}(x)$ and $t_*= \max\left\{f(x_*)/a,~ s - a\eta \right\}$;
    \item[(b)]  $\zeta^{\eta}_{y,s}(x)$ is $\eta^{-1}$-strongly convex;
    \item[(c)] if $\tilde{x}$ is a $(d+1)^{-1}$-solution to $\min_{x\in K} \zeta^{\eta}_{y,s}(x)$ and $\tilde{t}=\max\left\{ f(\tilde{x})/a, s - a\eta \right\}$, then $\tilde{w}=(\tilde{x},\tilde{t})$ satisfies \eqref{silif:difference:tildextildet:mainpaper}.
\end{itemize}
\end{lemma}
\begin{proof}  \underline{Part~a:} Per the definitions of $Q$ and $\Theta^{\eta,Q}_{y,s}(x,t)$ at respectively~\eqref{silif:def:pi} and~\eqref{silif:def:Theta}, 
 \begin{equation}
 \label{silif:Thetaexpand}
 \Theta^{\eta,Q}_{y,s}(x,t)
= I_{\{x\in K,f(x)\leq at\}}(x,t)+ \frac{1}{2\eta}\norm{x-y}^2+\frac{1}{2\eta}(t-(s-a\eta))^2+as-\frac{a^2\eta}{2}. 
 \end{equation}
 Then regarding  $\min \{\Theta^{\eta,Q}_{y,s}(x,t): (x,t)\in\R^{d+1}\}$, if we fix $x\in K$, the optimization over $t$ is 
\[
t^*= \underset{t \ge f(x)/a}{\argmin}\left\{
\frac{1}{2\eta}\left(t-(s - a\eta)\right)^2 \right\}
=\max\left\{ \frac{f(x)}{a}, s - a\eta \right\}.
\]
Then  $\frac{1}{2\eta}\left(t^*-(s - a\eta)\right)^2
=\frac{1}{2\eta}\left[ \frac{f(x)}{a} - (s - a\eta)\right]_{+}^2$. Plugging~this into~\eqref{silif:Thetaexpand} yields the stated formula for $x_*$. Having $x_*$ leads to the formula for $t_*$. 

\underline{Part~b:} Since $f$ is convex, $a>0$, and the maps $t \mapsto \max\{t,0\}$ and $t \mapsto t^2, t\ge0$ are convex and non-decreasing, $\left[ \frac{f(x)}{a} - (s - a\eta)\right]_{+}^2$ is convex. The strong convexity of $\zeta^{\eta}_{y,s}(x)$ is due to the quadratic term $\frac{1}{2\eta}\|x-y\|^2$. 

\underline{Part c:}  As we have shown that $\zeta^{\eta}_{y,s}$ is $\eta^{-1}$-strongly convex per Part~b, so
    \begin{equation}
    \label{silif:difference:tildex}
        \norm{\tilde{x}-x_*}\leq \sqrt{2\eta\brac{\zeta^{\eta}_{y,s}(\tilde{x})-\zeta^{\eta}_{y,s}(x_*) }}\leq \sqrt{\frac{2\eta}{d+1}}. 
    \end{equation}
Next, the fact that $f(x)$ is $L$-Lipschitz implies that $h(x)=\max \left\{f(x)/a,s-a\eta\right\}$ is $L/a$-Lipschitz. Combining this with \eqref{silif:difference:tildex} and $\tilde{t}=h(\tilde{x})$, we get 
\begin{equation}
\label{silif:difference:tildet}
    \abs{\tilde{t}-t_*}=\abs{h(\tilde{x})-h(x_*) }\stackrel{\eqref{silif:difference:tildex}}{\leq} \frac{L}{a}\norm{\tilde{x}-x_*}\leq \frac{L}{a}\sqrt{\frac{2\eta}{d+1}}.  
\end{equation}
Set $\tilde{w}=(\tilde{x},\tilde{t})$. Using \eqref{silif:difference:tildex},\eqref{silif:difference:tildet} and $(x_*,t_*)=\proj_Q(z)$ per \eqref{silif:optimizingThetaasprojection}, we arrive at the desired bound on $\norm{\tilde{w}- \proj_Q(z)}$. 
\end{proof}
Next, we apply Theorem~\ref{theo:vaidya} to $\zeta^{\eta}_{y,s}$ from Lemma~\ref{silif:lem:equivopt}  to produce $\tilde{w}$ satisfying \eqref{silif:difference:tildextildet:mainpaper}. 
\begin{theorem}
\label{silif:theo:cuttingplane} 
Assume conditions (A1), (A2) and (A4). Define
    \begin{align}
    \label{silif:def:alpha}
        F_K&=f(0)+LR_\mathrm{out}, \quad \alpha=\min\left\{
        \frac{\eta}{(d+1)D_{y,s}(a,\eta)},\frac12
        \right\}, \\
        D_{y,s}(a,\eta)&=
        2R_\mathrm{out}^2
        +2R_\mathrm{out}\left(R_\mathrm{out}+\norm y\right)
        +\frac{2R_\mathrm{out}L}{a}
        \left[\frac{F_K}{a}-s+a\eta\right]_+.\nn
    \end{align}
The following statements hold. 
\begin{itemize}
    \item[(a)] The cutting-plane method by \cite{vaidya,lee2015fastercutting} makes ${\mathcal{O}}\left(d \log \frac{d \gamma}{\alpha}\right)$ expected calls to the separation oracle of $K$ and ${\mathcal{O}}\left(d \log \frac{d \gamma}{\alpha}\right)$ expected calls to the subgradient oracle of $f$ in order to generate a $(d+1)^{-1}$-solution $\tilde{x}\in K$ to the optimization problem $\min_{x\in K} \zeta^{\eta}_{y,s}(x)$. Here $\gamma=R_\mathrm{out}/\mbox{minwidth}(K)$ and $\mbox{minwidth}(K)$ are defined in~Section~\ref{sec:prelim}. 
    \item[(b)] Suppose $(y,s)$ is the Step~1 in Algorithm~\ref{silif:alg:ASF_pi} from a feasible previous iterate. For every $h>0$, define
    \begin{align}
    \label{def:Dalpha}
    D_\alpha(a,\eta,h)
    &=6R_\mathrm{out}^2
    +2R_\mathrm{out}\sqrt\eta(\sqrt d+h)
    +\frac{2R_\mathrm{out}L}{a}
    \left(
    \frac{2LR_\mathrm{out}}{a}
    +h\sqrt\eta+a\eta
    \right),\notag\\
    \underline\alpha(a,\eta,h)
    &=\min\left\{
    \frac{\eta}{2(d+1)D_\alpha(a,\eta,h)},\frac14
    \right\}.
    \end{align}
    Then the parameter $\alpha$ in Part~a satisfies $\Pr\left(
    \alpha\leq\underline\alpha(a,\eta,h)
    \right)
    \le 2\exp(-h^2/2)$.
\end{itemize}
\end{theorem}

\begin{proof}
\underline{Part~a:} 
For any $x,x'\in K$, we have $\|x-x'\|\leq2R_\mathrm{out}$ and $\|x+x'-2y\|\leq2R_\mathrm{out}+2\|y\|$. By the Cauchy-Schwarz inequality,
\begin{equation}
\label{silif:zetafirstpart}
\left|\|x-y\|^2-\|x'-y\|^2\right|
=\left|\langle x-x',x+x'-2y\rangle\right|
\leq4R_\mathrm{out}(R_\mathrm{out}+\|y\|).
\end{equation}
Moreover, since $f$ is $L$-Lipschitz and for a 1-Lipschitz function $u$, $[u]_+$ is also $1$-Lipschitz,
$\big|[f(x)/a-s+a\eta]_+-[f(x')/a-s+a\eta]_+\big|\leq(L/a)\|x-x'\|$.
Since $0\in K$, we also have $f(x)\leq f(0)+L\|x\|\leq F_K$ for every $x\in K$. Using $|u^2-v^2|=|u-v|(u+v)$ for $u,v\geq0$, we obtain
\begin{align}
\label{silif:zetasecondpart}
&\left|
\left[\frac{f(x)}a-s+a\eta\right]_+^2
-\left[\frac{f(x')}a-s+a\eta\right]_+^2
\right|\nonumber\\
&\leq
\frac{2L}{a}
\left[\frac{F_K}{a}-s+a\eta\right]_+
\|x-x'\|
\leq
\frac{4R_\mathrm{out}L}{a}
\left[\frac{F_K}{a}-s+a\eta\right]_+.
\end{align}
By the definition of $\zeta^\eta_{y,s}$ in \eqref{silif:def:zeta},
\[
\max_{x\in K}\zeta^\eta_{y,s}(x)
-\min_{x\in K}\zeta^\eta_{y,s}(x)
\stackrel{\eqref{silif:zetafirstpart},\eqref{silif:zetasecondpart}}{\leq}
\frac{2R_\mathrm{out}(R_\mathrm{out}+\|y\|)}{\eta}
+\frac{2R_\mathrm{out}L}{a\eta}
\left[\frac{F_K}{a}-s+a\eta\right]_+ \stackrel{\eqref{silif:def:alpha}}{\leq}
\frac{D_{y,s}(a,\eta)}{\eta}.
\]
By Theorem~\ref{theo:vaidya} and the choice of $\alpha$ in \eqref{silif:def:alpha}, the cutting-plane method produces $\tilde x\in K$ satisfying
\[
\zeta^\eta_{y,s}(\tilde x)-\min_{x\in K}\zeta^\eta_{y,s}(x)
\stackrel{\eqref{ineq:opt}}{\leq}
\alpha\left(\max_{x\in K}\zeta^\eta_{y,s}(x)-\min_{x\in K}\zeta^\eta_{y,s}(x)\right)
\leq
\frac{\alpha D_{y,s}(a,\eta)}{\eta}
\stackrel{\eqref{silif:def:alpha}}{\leq}\frac{1}{d+1}. 
\]


\underline{Part~b:} 
Let $(y,s)$ be generated by Step~1 of Algorithm~\ref{silif:alg:ASF_pi}
from a feasible iterate $(x_{k-1},t_{k-1})\in Q$. Then
\begin{equation}
\label{silif:outputstep1}
y=x_{k-1}+\sqrt{\eta}Z_1,\qquad
s=t_{k-1}+\sqrt{\eta}Z_2,
\end{equation}
where $Z_1\sim\mathcal N(0,I_d)$ and $Z_2\sim\mathcal N(0,1)$
are independent.
For any $h>0$, define the event
\[
\mathcal E_h
=\left\{\|Z_1\|\le\sqrt d+h,\ Z_2\ge-h\right\}.
\]
Since $\E\|Z_1\|\leq\sqrt{\E\|Z_1\|^2}=\sqrt d$ and the Euclidean norm is $1$-Lipschitz, Gaussian concentration implies
$\Pr\{\|Z_1\|>\sqrt d+h\}\leq
\Pr\{\|Z_1\|>\E\|Z_1\|+h\}\leq\exp(-h^2/2)$.
The standard Gaussian tail bound also ensures $\Pr\{Z_2<-h\}\leq\exp(-h^2/2)$.
Hence,
\begin{equation}
\label{silif:goodprob}
\Pr(\mathcal E_h^c)\leq\Pr\{\|Z_1\|>\sqrt d+h\}+\Pr\{Z_2<-h\}
\leq2\exp(-h^2/2).
\end{equation}
On $\mathcal E_h$, since $x_{k-1}\in K\subseteq
R_\mathrm{out}\ball_d(0)$, the triangle inequality and \eqref{silif:outputstep1} yield
\begin{equation}
\label{silif:goodboundy}
\|y\|
\leq\|x_{k-1}\|+\sqrt\eta\|Z_1\|
\leq R_\mathrm{out}+\sqrt\eta(\sqrt d+h).
\end{equation}
Moreover, feasibility of $(y,s)$ implies $t_{k-1}\geq f(x_{k-1})/a$.
Since $F_K=f(0)+LR_\mathrm{out}$ and $f$ is $L$-Lipschitz on $K$,
$F_K-f(x_{k-1})\leq L\|x_{k-1}\|+LR_\mathrm{out}\leq2LR_\mathrm{out}$.
Using \eqref{silif:outputstep1} and $Z_2\geq-h$ on $\mathcal E_h$, we obtain
\begin{align}
\label{silif:goodbounds}
\left[\frac{F_K}{a}-s+a\eta\right]_+
&=\left[\frac{F_K}{a}-t_{k-1}-\sqrt\eta Z_2+a\eta\right]_+\nonumber\\
&\leq\left[\frac{F_K-f(x_{k-1})}{a}-\sqrt\eta Z_2+a\eta\right]_+\leq\frac{2LR_\mathrm{out}}{a}+h\sqrt\eta+a\eta,
\end{align}
where the last upper bound is positive. By the definition of $D_{y,s}(a,\eta)$ in \eqref{silif:def:alpha},
\begin{align}
&D_{y,s}(a,\eta)
=2R_\mathrm{out}^2
+2R_\mathrm{out}(R_\mathrm{out}+\|y\|)
+\frac{2R_\mathrm{out}L}{a}
\left[\frac{F_K}{a}-s+a\eta\right]_+\nonumber\\
&\stackrel{\eqref{silif:goodboundy},\eqref{silif:goodbounds}}{\leq}
6R_\mathrm{out}^2
+2R_\mathrm{out}\sqrt\eta(\sqrt d+h)
+\frac{2R_\mathrm{out}L}{a}
\left(\frac{2LR_\mathrm{out}}{a}+h\sqrt\eta+a\eta\right)=D_\alpha(a,\eta,h).
\label{silif:goodboundD}
\end{align}
Consequently, on $\mathcal E_h$,
\begin{equation}
\alpha
=\min\left\{
\frac{\eta}{(d+1)D_{y,s}(a,\eta)},\frac12
\right\}\stackrel{\eqref{silif:goodboundD}}{\geq}
\min\left\{
\frac{\eta}{(d+1)D_\alpha(a,\eta,h)},\frac12
\right\}
>\underline\alpha(a,\eta,h).
\label{silif:goodboundalpha}
\end{equation}
Thus, \eqref{silif:goodboundalpha} implies
$\{(Z_1,Z_2):\alpha\leq\underline\alpha(a,\eta,h)\}\subseteq\mathcal E_h^c$. Finally, it further follows from \eqref{silif:goodprob} that $\Pr((Z_1,Z_2):\alpha\leq\underline\alpha(a,\eta,h))
\leq\Pr(\mathcal E_h^c) \le 2\exp(-h^2/2)$. 
This completes the proof.
\end{proof}

We are now ready to prove Theorem~\ref{silif:theo:separation}(a).



\textbf{Proof of Theorem~\ref{silif:theo:separation}(a):}
 This part follows from Lemma~\ref{silif:lem:equivopt} and Theorem~\ref{silif:theo:cuttingplane}. 
A technical issue in solving~\eqref{silif:optimizingThetaasprojection} is that the cutting-plane method cannot be applied directly to $\Theta^{\eta,Q}_{y,s}$, since $Q$ in \eqref{silif:def:pi} is unbounded, whereas the cutting-plane method requires an initial bounded set. 
To address this, we invoke Lemma~\ref{silif:lem:equivopt} to rewrite~\eqref{silif:optimizingThetaasprojection} as an equivalent strongly convex program over the compact set $K$, namely
\begin{equation}
\label{silif:equivopt:mainpaper}
\min_{(x,t)\in \R^{d+1}}\Theta^{\eta,Q}_{y,s}(x,t)\;=\;\min_{x\in K}\zeta^{\eta}_{y,s}(x) +as-\frac{a^2\eta}{2},
\qquad 
\text{with }t(x)=\max\left\{\frac{f(x)}{a}, s-a\eta\right\}
\end{equation}
with $\zeta^{\eta}_{y,s}$ in~\eqref{silif:def:zeta}. 
We then run the cutting-plane method on $\zeta^{\eta}_{y,s}$ for an approximate minimizer and transfer the resulting guarantee back to $\Theta^{\eta,Q}_{y,s}$ via~\eqref{silif:equivopt:mainpaper}. 
This requires a separation oracle for $K$ and a subgradient oracle for $\zeta^{\eta}_{y,s}$, and the latter reduces to a subgradient oracle for $f$. 

Finally, regarding the parameter $\alpha\in(0,1)$ in Theorem~\ref{theo:vaidya}, we choose $\alpha$ as in~\eqref{silif:def:alpha} so that the cutting-plane method output $\tilde x$ is a $(d+1)^{-1}$-solution of $\min_{x\in K}\zeta^{\eta}_{y,s}(x)$, then we compute $\tilde{t}$ using~\eqref{silif:equivopt:mainpaper}. Finally, we set $\tilde{w}=(\tilde{x},\tilde{t})$, which satisfies ~\eqref{silif:difference:tildextildet:mainpaper} per Lemma~\ref{silif:lem:equivopt}(c). 
Then Theorem~\ref{theo:vaidya} yields the stated oracle complexity.  


Now we verify \eqref{silif:eq:concent}. Since $R_\mathrm{in}\leq R_\mathrm{out}$ and $a=(LR_\mathrm{out}+d)/R_\mathrm{in}$, we have $\frac{L}{a}\leq1$ and $\frac1a
=
\frac{R_\mathrm{in}}{LR_\mathrm{out}+d}
\leq
\frac{R_\mathrm{out}}d$. Taking $h=\sqrt d$ and using $\eta=a^{-2}$ therefore lead to $D_\alpha(a,\eta,\sqrt d)
\leq
10R_\mathrm{out}^2
+\frac{2R_\mathrm{out}}a(3\sqrt d+1)
\leq
18R_\mathrm{out}^2.$
Consequently, regarding $\underline\alpha$ in \eqref{def:Dalpha},  $\underline\alpha(a,\eta,\sqrt d)
\ge
\min\left\{
\frac{R_\mathrm{in}^2}
{36(d+1)R_\mathrm{out}^2(LR_\mathrm{out}+d)^2},
\frac14
\right\}$. Theorem~\ref{silif:theo:cuttingplane}(b)
now yields~\eqref{silif:eq:concent}.
\QEDA

\subsection{Sampling $w\sim \exp\brac{-\mathcal{P}_1(w)}$ in Algorithm~\ref{silif:RGO}}
\label{silif:appen:specialalgo}   

Recall the definition of $\mathcal{P}_1$ at~\eqref{silif:def:P1} and the point $\tilde{w}$ at~\eqref{silif:difference:tildextildet:mainpaper}. By completing the square, one can easily see that sampling $w\sim \exp\brac{-\mathcal{P}_1(w)}$ is equivalent to $w\sim \Lambda(w)\propto \exp(-(\norm{w-\tilde{w}}- \Delta)^2/(2\eta))$.
While $\Lambda(w)$ is not a Gaussian density, generating $w\sim \Lambda(w)$ is straightforward since it can be turned into a one-dimensional sampling problem. We state here a generic procedure for this sampling problem. An explanation is given in Lemma~\ref{silif:lem:samplingP2} below. 
\begin{algorithm}[H]
	\caption{Sample $w \sim \Lambda(w)$} 
	\label{silif:alg:samplingP1}
	\begin{algorithmic}
		\State 1. Generate $W\sim {\cal N}(0,I_{d+1})$ and set $\theta = W/\|W\|$;
        \State 2. Generate $r \propto r^{d} \exp\left(-\frac{1}{2\eta} \brac{r-\Delta}^2\right)$ by Adaptive Rejection Sampling for one-dimensional log-concave distribution by \cite{gilks1992adaptive}.
        \State 3. Output $w = \tilde{w} + r \theta$.
	\end{algorithmic}
\end{algorithm}

\begin{lemma}
\label{silif:lem:samplingP2}
    Algorithm \ref{silif:alg:samplingP1} generates $w \sim \Lambda(w)$.
\end{lemma}

\begin{proof}  
We rewrite $\Lambda(w)$ in polar coordinates.
Let $r = \| w - \widetilde{w} \|$. 
Since $\mathrm{d}x \mathrm{d}t = r^{d} \mathrm{d}r \mathrm{d}S(\theta)$,
where $\mathrm{d}S$ is surface measure on $\mathbb S^{d}$, the joint density of $(r,\theta)$ with respect to $\mathrm{d}r\mathrm{d}S(\theta)$ is $p(r,\theta)=r^d\Lambda(\tilde w+r\theta)\propto r^{d}\exp(-(r-\Delta)^2/(2\eta))$ for $r\ge0$ and $\theta\in\mathbb S^d$.
Its first marginal is  $p_r(r) \propto r^{d} \exp (-\left(r - \Delta\right)^2/(2\eta))$. As  $\log p_r(r) = d \log r - \left(r - \Delta\right)^2/(2\eta) + \text{const}$ and  $\frac{\mathrm{d}^2}{\mathrm{d}r^2}\log p_r(r) = -\frac{d}{r^2} - \frac{1}{\eta} < 0$, $p_r(r)$ is a one-dimensional log-concave distribution.
Thus, we can sample $r \sim p_r$ using any standard one-dimensional log-concave sampler (e.g., \cite{gilks1992adaptive}).
Then independently sample $W \sim \mathcal{N}(0, I_{d+1})$, set $\theta = W / \|W\|$ \cite{muller1959note}, and output  $w = \tilde{w} + r \theta$
as a sample with density $\Lambda(w)$.
\end{proof}

\section{Supporting Results and Proofs for Section~\ref{sec:composite:mainpaper}}
\label{appen:composite}

\subsection{Proof of Theorem~\ref{dolif:theo:outer}}
\label{dolif:appen:proofouter}

The upcoming result is analogous to Lemma~\ref{silif:lem:PIconstant}. The proof is omitted. 
\begin{lemma}
\label{dolif:lem:PIconstant}
With respect to $\tilde\pi$ at~\eqref{dolif:def:piandQ}, assume conditions~(B1) and (B2) and $a,b>0$. Then $\tilde\pi$ satisfies a PI and, for universal constants $C,C_0$,
\[
C_{\mathrm{PI}}(\tilde\pi)
\leq
C\log(d+2)
\left[
\left(
1+\frac{C_0L_h^2\log(d+1)}{a^2}
\right)
\opnorm{\cov{\tilde\nu}}
+\frac{1}{a^2}
+\frac{d+2}{b^2}
\right].
\]
In particular, if $a^2\geq C_0L_h^2\log(d+1)$ then for a universal constant $C'$,
\begin{equation}
\label{eq:dolif:PIboundsimplified}
C_{\mathrm{PI}}(\tilde\pi)
\leq
C'\log(d+2)
\left(
\opnorm{\cov{\tilde\nu}}
+\frac{d+3}{a^2}
+\frac{d+3}{b^2}
\right).
\end{equation}
\end{lemma}

The following warm-start result is analogous to Lemma~\ref{silif:lem:warmstart} in the constrained sampling case.
\begin{lemma}
\label{dolif:lem:warmstart}
Under condition~(B.3), if $\tilde{\pi}_\eta$ and $\tilde{\pi}^k$ are respectively the output of the first step of Algorithm~\ref{dolif:alg:ASF_pi} and the output of the $k$-th iterate, then it holds that $    \frac{\D\tilde{\pi}_\eta}{d\widetilde{\Pi}^{Y,U,V}}\leq M$ and $\frac{\D\tilde{\pi}^k}{d\widetilde{\Pi}^{X,S,T}}\leq M$ for $k\geq 0$. In other words, the warmness condition holds for every step of Algorithm~\ref{dolif:alg:ASF_pi}. 
\end{lemma}
\begin{proof} 
The fact that $\tilde{\pi}_0$ is $M$-warm with respect to $\tilde{\pi}$ has been shown in \eqref{dolif:eq:warmstart}. Now let $U\subseteq \R^{d+2}$ be measurable. For $(y,u,v)\in \R^{d+2}$, set
$U-(y,u,v)=\{(x,s,t)\in \R^{d+2}:(x,s,t)+(y,u,v)\in U \}$, and denote by $g$ the density of
$\mathcal{N}(0,\eta I_{d+2})$.
By $\frac{\D\tilde{\pi}_0}{\D\tilde{\pi}}\le M$, $\tilde{\pi}_\eta=\tilde{\pi}_0*g$, and
$\widetilde{\Pi}^{Y,U,V}=\widetilde{\Pi}^{X,S,T}*g$ with $\widetilde{\Pi}^{X,S,T}=\tilde{\pi}$, we have
\begin{align*}
\tilde{\pi}_\eta(U)
&=\int_{\R^{d+2}} \tilde{\pi}_0\left(U-(y,u,v)\right) g(y,u,v)  \D y \D u \D v\\
&=\int_{\R^{d+2}}\brac{\int_{U-(y,u,v)} \frac{\D\tilde{\pi}_0}{\D\widetilde{\Pi}^{X,S,T}}(x,s,t) \widetilde{\Pi}^{X,S,T}(\D x,\D s,\D t) }
g(y,u,v)  \D y \D u \D v\\
&\le M\int_{\R^{d+2}} \widetilde{\Pi}^{X,S,T}\left(U-(y,u,v)\right) g(y,u,v)  \D y \D u \D v
= M \widetilde{\Pi}^{Y,U,V}(U),
\end{align*}
and hence $\frac{\D\tilde{\pi}_\eta}{d\widetilde{\Pi}^{Y,U,V}}\le M$.
A similar argument (using the fact that Step~2 of Algorithm~\ref{dolif:alg:ASF_pi} is the exact conditional update with stationary
law $\widetilde{\Pi}^{X,S,T |  Y,U,V}$) gives $\frac{\D\tilde{\pi}^k}{d\widetilde{\Pi}^{X,S,T}}\le M$ for every $k$.
Thus, warmness holds for every step of Algorithm~\ref{dolif:alg:ASF_pi}.
\end{proof}


\textbf{Proof of Theorem~\ref{dolif:theo:outer}:} \underline{Step 1:}  \eqref{eq:dolif:PIboundsimplified} in Lemma~\ref{dolif:lem:PIconstant} provides the stated PI constant bound. 

\underline{Step 2:} In order to apply Theorem~\ref{theo:outer}, we first verify that the hypotheses \eqref{cond:chisquare:initialandtarget} and \eqref{cond:renyi:initialandtarget} hold for the lifted initial law $\tilde{\pi}_0$ and lifted target $\tilde{\pi}$. 
We show in \eqref{dolif:eq:warmstart} that $ \frac{\D\tilde{\pi}_0}{\D\tilde{\pi}}\le M$, which implies $\D\tilde{\pi}_0/\D\tilde{\pi}\in L^p(\tilde{\pi})$ for every $p<\infty$. Consequently, the first assumptions in \eqref{cond:chisquare:initialandtarget} and \eqref{cond:renyi:initialandtarget} are satisfied. It remains to check the second assumptions in \eqref{cond:chisquare:initialandtarget} and \eqref{cond:renyi:initialandtarget}. Let $(X,S,T)\sim\tilde{\pi}$ for $\tilde{\pi}$ in \eqref{dolif:def:piandQ}, then
\[
S\mid (X=x) \stackrel{\text{d}}{=} \frac{h(x)}{a}+E_1,
\qquad
T\mid (X=x,S=s) \stackrel{\text{d}}{=} \frac{f(x)+as}{b}+E_2,
\]
where $X\sim\tilde{\nu}$, $E_1\sim \mathrm{Exp}(a)$, and $E_2\sim \mathrm{Exp}(b)$. By (B2), $f$ and $h$ are Lipschitz, hence have at most linear growth. Since $\tilde{\nu}$ is log-concave and has finite normalizing constant, $X$ has finite moments of all orders. Therefore $S$ and $T$ have finite moments of all orders, and consequently $  \int_{\Q} \|(x,s,t)\|^p\D \tilde{\pi}(x,s,t)<\infty$ for every $p<\infty$. Thus, the hypotheses of Theorem~\ref{theo:outer} hold. 




\underline{Step 3:}   We apply Theorem~\ref{theo:outer} in this step. We have that \eqref{dolif:eq:warmstart} implies $\chi^2\brac{\tilde{\pi}_0||\tilde{\pi}}\leq M-1$. Then by applying Theorem~\ref{theo:outer}(a) iteratively,  we can solve for  $ \frac{M}{\brac{1+\eta/C_{\mathrm{PI}}}^{2k}} \leq \varepsilon$
to get the stated iteration complexity with respect to $\chi^2$ divergence. The calculation for $q$-R\'{e}nyi divergence when $q\ge2$ is along the same line with the use of Theorem~\ref{theo:outer}(b). For $q$-R\'{e}nyi divergence with $1\leq q<2$, we refer to Remark~\ref{rem:renyiqlessthan2}.  This completes the proof.
\QEDA

\subsection{Supporting lemmas}
\label{dolif:appen:supportlem}

The next result is analogous to Lemma~\ref{silif:lem:envelope}. The proof follows the same overall strategy, but the construction of the envelope set is different.
\begin{lemma}
\label{dolif:lem:envelope}
Assume condition~(B2) holds. For all $\tau>0$, we have 
\[\int_{0}^{\infty} \exp \left(-\frac{(r-\tau)^2}{2\eta}\right)
B_{\Q_r}  \D r
\ \le\
\widetilde{C}_L  Z_{\Q}  \sqrt{2\pi\eta}\;
\exp \left(\tau \widetilde{C}_L + \frac{\eta \widetilde{C}_L^2}{2}\right)+Z_{\Q}\brac{\exp(\tau \widetilde{C}_L)-1}. \]
\end{lemma}

\begin{proof}  
The proof consists of two steps. In Step 1, we construct an envelope set which contains $\Q_r$ and is much easier to integrate on compared to $\Q_r$; while in Step 2, we perform integration by parts (IBP) to obtain the desired bound. The reason we need IBP to bound the quantity $\int_{0}^{\infty} \exp \left(-\frac{(r-\tau)^2}{2\eta}\right)
B_{\Q_r}  \D r$ is that $B_{\Q_r}$ as an integral on the boundary $\partial \Q_r$ is quite hard to bound; so via IBP, $B_{\Q_r}$ is replaced by $Z_{\Q_r}$ which is easier to handle.

Note that relative to the proof of Lemma~\ref{silif:lem:envelope}, Step~1 below differs from the corresponding Step~1 there, whereas Step~2 is essentially the same.

\underline{Step 1:}
Write $\Q=\widetilde{C}_1\cap \widetilde{C}_2$ where \begin{equation*}
    \widetilde{C}_1=\left\{(x,s,t)\in\R^{d+2}:\ s\ge \frac{h(x)}a\right\},
\qquad
\widetilde{C}_2=\left\{(x,s,t)\in\R^{d+2}:\ t\ge \frac{f(x)+as}b\right\}. 
\end{equation*}
Fix $r>0$. Since $(A\cap B)+H\subseteq (A+H)\cap(B+H)$ for any sets $A,B,H$, let $A=\widetilde{C}_1,B=\widetilde{C}_2$ and $H=r\ball_{d+2}(0)$ to get
\begin{equation}
\label{dolif:setQr}
\Q_r=(\widetilde{C}_1\cap \widetilde{C}_2)_r \subseteq (\widetilde{C}_1)_r\cap (\widetilde{C}_2)_r.
\end{equation}
For $\widetilde{C}_1$, apply Lemma~\ref{lem:parallelsetcontainment} with $\phi=\frac{h}{a}$.
Since $\phi$ is $\frac{L_h}{a}$-Lipschitz, we obtain
\[
(\widetilde{C}_1)_r \subseteq \left\{(x,s,t):\ s\ge \frac{h(x)}{a}-r\sqrt{1+\frac{L_h^2}{a^2}}\right\}
=\left\{(x,s,t):\ s\ge \frac{h(x)}{a}-\frac r a \widetilde{C}_h\right\}.
\]
For $\widetilde{C}_2$, apply Lemma~\ref{lem:parallelsetcontainment} with $\phi(x,s)=\frac{f(x)+as}{b}$ for $(x,s)\in \R^{d+1}$. Since $\phi$ is $\frac{\sqrt{L_f^2+a^2}}{b}$-Lipschitz in $(x,s)$, we obtain \[(\widetilde{C}_2)_r \subseteq \left\{(x,s,t):\ t\ge \frac{f(x)+as}{b}-r\sqrt{1+\tfrac{L_f^2+a^2}{b^2}}\right\}
=\left\{(x,s,t):\ t\ge \frac{f(x)+as}{b}-\frac{r\widetilde{C}_f}{b} \right\}.\] 
Therefore, via \eqref{dolif:setQr}, we conclude $\Q_r\subseteq \widetilde{\mathcal{E}}_r$ where \[\widetilde{\mathcal{E}}_r
=\left\{(x,s,t):\ s\ge \frac{h(x)}{a}-\frac r a \widetilde{C}_h,\quad
t\ge \frac{f(x)+as}{b}-\frac r b \widetilde{C}_f\right\}.\] This further implies
\[
Z_{\Q_r}\le Z_{\widetilde{\mathcal{E}}_r},
\]
for which 
\begin{align*}
Z_{\widetilde{\mathcal{E}}_r}
&=\int_{\R^d}\int_{s\ge \frac{h(x)}{a}-\frac r a \widetilde{C}_h}
\int_{t\ge \frac{f(x)+as}{b}-\frac r b \widetilde{C}_f} e^{-bt}  \D t  \D s  \D x 
=\frac{e^{r\widetilde{C}_f}}{b}\int_{\R^d} e^{-f(x)}
\left(\int_{s\ge \frac{h(x)}{a}-\frac r a \widetilde{C}_h} e^{-as}  \D s\right)  \D x \nonumber\\
&\stackrel{\eqref{eq:intoverQ}}{=}\frac{e^{r\widetilde{C}_f}}{ab}\int_{\R^d} e^{-f(x)}e^{-h(x)} e^{r\widetilde{C}_h}  \D x
= e^{r(\widetilde{C}_f+\widetilde{C}_h)}  Z_{\Q}
= e^{r\widetilde{C}_L}  Z_{\Q}.
\end{align*}
In particular, we apply Lemma \ref{lem:basicfacts}(b) above with $F(x)=\frac{e^{r\widetilde C_f}}{b} e^{-f(x)},
\psi(x)=\frac{h(x)}{a}$ and $\gamma=\frac{r}{a}\widetilde C_h$. Therefore, we get $Z_{\Q_r}\le Z_{\widetilde{\mathcal{E}}_r}=e^{r\widetilde{C}_L}Z_{\Q}$.

\underline{Step 2:}
Set $w_{\eta,\tau}(r)=\exp \left(-\frac{(r-\tau)^2}{2\eta}\right)$.
applying \eqref{eq:IrderivativeofQr} in Lemma~\ref{lem:coarea}  with with $c=b$ leads to
$\frac{\D}{\D r}Z_{\Q_r}=B_{\Q_r}$ a.e.\ $r$.
Moreover, since $Z_{\Q_r}\le e^{r\widetilde{C}_L}Z_{\Q},\forall r\ge0$ from the conclusion of Step~1 and $w_{\eta,\tau}(r)$ decays exponentially,
we have $Z_{\Q_r}w_{\eta,\tau}(r)\to0$ as $r\to\infty$.
Also $Z_{\Q_0}=Z_{\Q}$ and $w_{\eta,\tau}(0)=e^{-\tau^2/(2\eta)}$.
Integration by parts and the fact that  $w_{\eta,\tau}'(r)=-(r-\tau)/\eta\; w_{\eta,\tau}(r)$ imply
\begin{align}
\label{dolif:eq:IBP-master}
&\int_{0}^{\infty} B_{\Q_r}   w_{\eta,\tau}(r)  \D r
=\int_{0}^{\infty} Z_{\Q_r}'   w_{\eta,\tau}(r)  \D r
= Z_{\Q_r}w_{\eta,\tau}(r)\bigg|_{0}^{\infty}
-\int_{0}^{\infty} Z_{\Q_r}   w_{\eta,\tau}'(r)  \D r\nonumber\\
\quad&= -Z_{\Q} e^{-\tau^2/(2\eta)}
+ \int_{0}^{\infty} Z_{\Q_r}  \frac{r-\tau}{\eta}  w_{\eta,\tau}(r)  \D r.
\end{align}
We will bound the second term on the right-hand side of~\eqref{dolif:eq:IBP-master} using $Z_{\Q}\leq Z_{\Q_r}\leq e^{r\widetilde{C}_L}Z_{\Q}$ for $r\geq 0$.  
\begin{align*}
    \int_{0}^{\infty} Z_{\Q_r}  \frac{r-\tau}{\eta}  w_{\eta,\tau}(r)  \D r &= \int_0^\tau Z_{\Q_r}  \frac{r-\tau}{\eta}  w_{\eta,\tau}(r)  \D r + \int_\tau^\infty Z_{\Q_r}  \frac{r-\tau}{\eta}  w_{\eta,\tau}(r)  \D r\\
    &\leq Z_{\Q}\int_0^\tau \frac{r-\tau}{\eta}    w_{\eta,\tau}(r)  \D r+ Z_{\Q}\int_\tau^\infty \frac{r-\tau}{\eta}  e^{r\widetilde{C}_L}  w_{\eta,\tau}(r)  \D r \\
    &= - Z_{\Q} w_{\eta,\tau}(r) \bigg|_{0}^{\tau} + Z_{\Q}\int_\tau^\infty \frac{r-\tau}{\eta}  e^{r\widetilde{C}_L}  w_{\eta,\tau}(r)  \D r \\
    &= - Z_{\Q} (1 - e^{-\tau^2/(2\eta)}) + Z_{\Q}\int_\tau^\infty \frac{r-\tau}{\eta}  e^{r\widetilde{C}_L}  w_{\eta,\tau}(r)  \D r.
\end{align*}
Now via formula~\eqref{eq:kindofibp} in Lemma~\ref{lem:basicfacts}, 
\begin{equation*}
     \int_{0}^{\infty} Z_{\Q_r}  \frac{r-\tau}{\eta}  w_{\eta,\tau}(r)  \D r \leq - Z_{\Q} (1 - e^{-\tau^2/(2\eta)}) +Z_{\Q}\brac{ \widetilde{C}_L\int_\tau^{\infty} e^{\widetilde{C}_Lr}w_{\eta,\tau}(r) \D r + e^{\widetilde{C}_L \tau} }. 
\end{equation*}
Then \eqref{dolif:eq:IBP-master} becomes
\begin{equation}
\label{finalboundintBQr_comp}
\int_{0}^{\infty} B_{\Q_r}   w_{\eta,\tau}(r)  \D r
\le Z_{\Q}\left(e^{\widetilde{C}_L\tau}-1\right)+\widetilde{C}_L Z_{\Q}\int_{\tau}^{\infty} e^{r\widetilde{C}_L}w_{\eta,\tau}(r)  \D r.
\end{equation}
Finally, we estimate the Gaussian integral. Via $\widetilde{C}_L r-\frac{(r-\tau)^2}{2\eta}
=-\frac{1}{2\eta}\left(r-(\tau+\eta \widetilde{C}_L)\right)^2+\tau \widetilde{C}_L+\frac{\eta \widetilde{C}_L^2}{2}$, we can further write 
\[\int_{0}^{\infty} e^{r\widetilde{C}_L}w_{\eta,\tau}(r)  \D r
\le e^{\tau \widetilde{C}_L+\frac{\eta \widetilde{C}_L^2}{2}}
\int_{-\infty}^{\infty}\exp \left(-\frac{(r-(\tau+\eta \widetilde{C}_L))^2}{2\eta}\right)  \D r=\sqrt{2\pi\eta}\;\exp \left(\tau \widetilde{C}_L+\frac{\eta \widetilde{C}_L^2}{2}\right).\] Combining with~\eqref{finalboundintBQr_comp} to reach the stated bound. 
\end{proof}
The next result is a consequence of Lemma~\ref{dolif:lem:envelope} and analogous to Lemma \ref{silif:lem:keytechnical}.

\begin{lemma}
\label{dolif:lem:keytechnical}
Assume condition~(B2) holds. Let $Z_{\Q}=Z_{\Q_0}$. Then for $\tau\geq 0$ and $ q_{d+2}=v-b\eta$, we have \begin{equation*}\int_{\R^{d+2}} e^{-bq_{d+2}} \exp\brac{-\frac{\brac{\dist ( q,\Q)-\tau}^2}{2\eta}}\D q \le Z_{\Q} \exp(\tau \widetilde{C}_L) \left[ 
     \widetilde{C}_L \sqrt{2\pi\eta} \exp\brac{ \frac{\eta \widetilde{C}_L^2}{2}} +1\right].
    \end{equation*}
\end{lemma}
\begin{proof}   We split the proof into two cases. 

\underline{Case I:} When $q\in \Q$, we have $\dist(q,\Q)=0$. Then
\begin{equation*}
    \int_{\Q} \exp\brac{-bq_{d+2}} \exp\brac{-\frac{\brac{\dist (q,\Q)-\tau}^2}{2\eta}}\D q
    =\int_{\Q} \exp\brac{-bq_{d+2}} \exp\brac{-\frac{\tau^2}{2\eta}}   \D q \stackrel{\eqref{def:ZandB}}{=} e^{-\frac{\tau^2}{2\eta}} Z_{\Q}.
\end{equation*}

\underline{Case II:} When $q \notin \Q$, we denote $r(q) = \dist(q,\Q)$ and $w_{\eta,\tau}(r) = \exp\brac{-\frac{(r-\tau)^2}{2\eta}}$. Using the co-area formula, we write
\begin{align*}
    &\int_{\Q^C} \exp\brac{-bq_{d+2}} \exp\brac{-\frac{\brac{\dist (q,\Q)-\tau}^2}{2\eta}}\D q\\
    &=\int_0^\infty \exp\brac{-\frac{(r-\tau)^2}{2\eta}}\left(\int_{\partial \Q_r}\exp\brac{-bq_{d+2}}\D S^{d+1}(q)\right)\D r = \int_0^\infty B_{\Q_r} w_{\eta,\tau}(r) \D r.
\end{align*}
We apply the bound in Lemma~\ref{dolif:lem:envelope} to get 
\begin{equation*}
    \int_0^\infty B_{\Q_r} w_{\eta,\tau}(r) \D r\le  Z_{\Q} \brac{\widetilde{C}_L\sqrt{2\pi\eta}\;\exp\brac{\tau \widetilde{C}_L + \frac{\eta \widetilde{C}_L^2}{2}}+\exp(\tau \widetilde{C}_L)-1}.
\end{equation*}
Combining Case I, Case II and the fact that $e^{-\frac{\tau^2}{2\eta}} \leq 1$  yields the desired bound. 
\end{proof}


\subsection{Proof of Theorem \ref{dolif:theo:separation}(a)}
\label{dolif:appen:cuttingplane}

We need some lemmas to support the proof of Theorem \ref{dolif:theo:separation}(a) at the end of this subsection. To ensure that restricting the cutting-plane method to $\Q_{\local}$ is valid, we need the following result.

\begin{lemma}
\label{dolif:lem:minimizerinlocalball}  
Recall $\widetilde{\Theta}^{\eta,\Q}_{y,u,v}$ and $\Q_{\local}=\Q\cap \ball_{\local}$ as in~\eqref{dolif:def:Theta} and \eqref{dolif:def:localball}, respectively. The minimizer $ p_*=\argmin_{ p\in \R^{d+2}}\widetilde{\Theta}^{\eta,\Q}_{y,u,v}( p)=\proj_{\Q}( q)$
belongs to $\Q_{\local}$ a.s. In particular, $ p_*=\argmin_{p\in \Q_\local}\widetilde{\Theta}^{\eta,\Q}_{y,u,v}( p)$. 
\end{lemma}

\begin{proof}
Recall in Section~\ref{sec:composite:mainpaper}, we denote $p$ as the previous RGO output (Algorithm~\ref{dolif:RGO}), so that $ p\in \Q$ a.s. This implies
$\norm{ p_*- q}\leq \norm{ p- q}=r_\local$. This implies
$ p_*\in q+r_\local \ball_{d+2}(0)\subseteq q+2r_\local \ball_{d+2}(0)=\ball_\local$ a.s.
Since $ p_*\in \Q$ by definition ($\Q$ is closed), we deduce $ p_*\in \Q\cap \ball_\local=\Q_\local$ a.s.
Since $ p_*$ minimizes $\widetilde{\Theta}^{\eta,\Q}_{y,u,v}$ over $\Q$, it also minimizes $\widetilde{\Theta}^{\eta,\Q}_{y,u,v}$ over the smaller domain $\Q_\local=\Q\cap \ball_\local$.
\end{proof}
Per Lemma~\ref{dolif:lem:minimizerinlocalball}, in order to find an approximate solution to $ p_*=\argmin_{ p\in \R^{d+2}}\widetilde{\Theta}^{\eta,\Q}_{y,u,v}( p)$, we will apply the cutting-plane method (Theorem~\ref{theo:vaidya}) to optimize $\widetilde{\Theta}^{\eta,\Q}_{y,u,v}( p)$ over $\Q_{\local}=\Q\cap (q+2r_{\local}\ball_{d+2}(0))$. In view of Theorem~\ref{theo:vaidya}, we need to ensure $\gamma$ is not too big, which is the content of the next result. 

This next result is also one of the subtle points of this appendix. It turns out via Lipschitz property of $f$ and $h$, we are able to  upper bound $\gamma$ if the cutting-plane method is applied to $\Q_{\local}(c)=\Q\cap (q+cr_{\local}\ball_{d+2}(0))$ for any $c>1$. In fact, the upcoming proof argument would not go through for $c=1$, which is why we choose to work with the local set $\Q_{\local}=\Q_{\local}(2)=\Q\cap (q+2r_\local \ball_{d+2}(0))$ in Theorem~\ref{dolif:theo:separation}. 

\begin{lemma}
\label{dolif:lem:localset}
Assume condition (B2) on $f$ and $h$. Fix any $c>1$ and define
\[
\ball_{c,\local}=q+c r_\local \ball_{d+2}(0),
\qquad
\Q_{\local}(c)=\Q\cap \ball_{c,\local},
\]
where $ q$ and $r_\local$ are defined at \eqref{dolif:def:localball}. Recall  $\widetilde{C}_h$ and $\widetilde{C}_f$ in~\eqref{dolif:def:notation} and set $\mu=\min\left\{\frac{a}{\widetilde{C}_h},\ \frac{b}{\widetilde{C}_f}\right\}, \lambda=2+\frac{a}{b}, \kappa=\frac{c-1}{\lambda+\mu}.$ Moreover, define $\bar p
=
 p_*+\left(0,\ \Delta s,\ \frac{a}{b}\Delta s+\Delta t\right)$ and $
\Delta s=\Delta t=\kappa r_\local$. Then we have the inclusion $B\left(\bar p,\ \kappa\mu r_\local\right)\subseteq \Q_{\local}(c)$. In particular,
\[
\mbox{\em minwidth}(\Q_{\local}(c))\ge 2\kappa\mu r_\local
\quad\text{and}\quad
\gamma(c)=\frac{c r_\local}{\mbox{\em minwidth}(\Q_{\local}(c))}
\le \frac{c(\lambda+\mu)}{(c-1)\mu}. 
\]
\end{lemma}

\begin{proof} 
To control $\gamma(c)$, it suffices to lower bound
$\mbox{\em minwidth}(\Q_{\local}(c))$, which we do by constructing an inscribed ball inside
$\Q_{\local}(c)$.

\smallskip
\underline{Step 1:}  Set $g_1( p)=g_1(x,s,t)=h(x)-a s$ and $g_2( p)=g_2(x,s,t)=f(x)+a s-b t$
so that  $\Q=\{ p:\ g_1( p)\le 0,\ g_2( p)\le 0\}$. 
Because $h$ is $L_h$-Lipschitz and $f$ is $L_f$-Lipschitz, for all $ p=(x,s,t)$ and $ p'=(x',s',t')$,
\begin{equation}
\label{dolif:consequencelipschitz}
|g_1( p)-g_1( p')|
\le L_h\|x-x'\|+a|s-s'|
\le \sqrt{L_h^2+a^2} \| p- p'\|
= \widetilde{C}_h\| p- p'\|,
\end{equation}
and similarly $|g_2( p)-g_2( p')|\le \widetilde{C}_f\| p- p'\|$.

Since $ p_*\in \Q$, we have $g_1( p_*)\le 0$ and $g_2( p_*)\le 0$.
For $\bar p$ in the statement of the lemma, 
\begin{equation}
\label{dolif:compareg1g2}
g_1(\bar p)=g_1( p_*)-a\Delta s\le -a\Delta s,
\qquad
g_2(\bar p)=g_2( p_*)-b\Delta t\le -b\Delta t.
\end{equation}
Moreover, we have  $\kappa\mu r_\local=\min\left\{\frac{a\Delta s}{\widetilde{C}_h},\ \frac{b\Delta t}{\widetilde{C}_f}\right\}$. 
Then for any $\zeta$ with $\|\zeta\|\le \kappa\mu r_\local$, by \eqref{dolif:consequencelipschitz}, \eqref{dolif:compareg1g2},  we can write
\begin{align*}
g_1(\bar p+\zeta)&\stackrel{\eqref{dolif:consequencelipschitz}}{\le} g_1(\bar p)+\widetilde{C}_h\|\zeta\|\stackrel{\eqref{dolif:compareg1g2}}{\le} -a\Delta s+\widetilde{C}_h \kappa\mu r_\local{\le} 0,\\
g_2(\bar p+\zeta)&\le -b\Delta t+\widetilde{C}_f \kappa\mu r_\local\le 0. 
\end{align*}
Thus, we conclude $\bar p+\kappa\mu r_\local \ball_{d+2}(0)\subseteq \Q$. Notice in particular that the inscribed ball $\bar p+\kappa\mu r_\local \ball_{d+2}(0)$ is non-trivial only if $\kappa=\frac{c-1}{\lambda+\mu}>0$, and for that we require $   c>1$. 

\underline{Step 2:} We will show $\ball_{\mathrm{in}}=\bar p+\kappa\mu r_{\local}\ball_{d+2}(0)$ also lies in $q+c r_{\local} \ball_{d+2}(0)$ and hence in $\Q_{\local}(c)$. Since $\| p_*- q\|\le \| p_{k-1}- q\|=r_\local$, we have
\begin{align*}
\|\bar p- q\|
\le \| p_*- q\|+\|\bar p- p_*\|
&\le r_\local+\left(\Delta s+\left|\frac{a}{b}\Delta s+\Delta t\right|\right)\\
&= r_\local+\left(\left(1+\frac{a}{b}\right)\Delta s+\Delta t\right)
= r_\local+\kappa \lambda r_\local,
\end{align*}
which combined with $\kappa=\frac{c-1}{\lambda+\mu}$ leads to
\begin{equation}
\label{dolif:inequalityconnectcenters}
\|\bar p- q\|+\kappa\mu r_\local
\le r_\local+\kappa(\lambda+\mu)r_\local
= r_\local+(c-1)r_\local
= c r_\local. 
\end{equation}
Take any point $p\in \bar p+\kappa\mu r_\local \ball_{d+2}(0)$, then we get $\norm{p- q}\leq \norm{p-\bar p}+\norm{ q-\bar p}\leq  \kappa\mu r_\local+\|q-\bar p\|\stackrel{\eqref{dolif:inequalityconnectcenters}}{\leq} c r_\local$.
This implies $\bar p+\kappa\mu r_\local \ball_{d+2}(0)\subseteq q+c r_\local \ball_{d+2}(0)=\ball_{c,\local}$. Combining with Step~1 to get
$\bar p+\kappa\mu r_\local \ball_{d+2}(0)\subseteq \Q\cap \ball_{c,\local}=\Q_\local(c)$. 

\smallskip
\underline{Step 3:} Finally, since $\ball_{\mathrm{in}}=\bar p+\kappa\mu r_\local \ball_{d+2}(0)\subseteq \Q_\local(c)$ per Step 2, we have $\mbox{minwidth}(\Q_\local(c))\ge 2\kappa\mu r_\local$,
so that $\gamma(c)
\le \frac{c r_\local}{2\kappa\mu r_\local}
= \frac{c(\lambda+\mu)}{2(c-1)\mu}
\le \frac{c(\lambda+\mu)}{(c-1)\mu}$. 
\end{proof}


In the upcoming result, via Theorem~\ref{theo:vaidya}, we describe the oracle complexities of solving the minimization problem $\argmin_{p\in \Q_{\local}} \widetilde{\Theta}^{\eta,\Q}_{y,u,v}(p)$. We then invoke Lemma~\ref{dolif:lem:localset} to bound the parameters appearing in these complexity estimates.

\begin{theorem}
\label{dolif:theo:cuttingplane}
Assume conditions~(B1) and (B2) and recall $ p_*=\argmin_{p\in \Q_\local}    \widetilde{\Theta}^{\eta,\Q}_{y,u,v}( p)$
and the local radius $r_\local$ at \eqref{dolif:def:localball}. Let 
\begin{equation}
\label{dolif:def:deltaandalpha}
   \tilde\Delta = \sqrt{\frac{2\eta}{d+2}}, \qquad \alpha=\min\left\{\frac{\tilde\Delta^2}{4r_{\local}^2},\frac{1}{2}\right\}.
\end{equation}
Then the cutting-plane method by \cite{vaidya,lee2015fastercutting} makes
${\mathcal{O}} \left((d+2) \log \frac{d \gamma}{\alpha}\right)$ expected calls to the separation
oracle of $\Q$ in order to generate a $1/(d+2)$-solution
$\tilde p\in \Q_{\local}$ to the optimization problem
$\min_{ p\in \Q_{\local}} \widetilde{\Theta}^{\eta,\Q}_{y,u,v}( p)$. Here $\gamma=\frac{R}{\mbox{minwidth}(\Q_{\local})}$
with $R$ any radius such that $\Q_{\local}\subseteq q+R \ball_{d+2}(0)$.
Moreover, for arbitrary $a,b,\eta>0$ and $R=2r_\local$, set $\mu=\min\left\{
\frac{a}{\sqrt{a^2+L_h^2}},
\frac{b}{\sqrt{a^2+b^2+L_f^2}}
\right\}$ and $\lambda=2+\frac ab$. 
Then
\begin{enumerate}[label=(\alph*)]
    \item it holds that  $\gamma\leq 2(\lambda+\mu)/\mu$. 

\item for every $\ell>0$, we have the concentration inequality
\[
    \Pr\left(
        \alpha\leq
        \min\left\{
        \frac{1}{2(d+2)(\sqrt{d+2}+\ell+b\sqrt\eta)^2},\frac12
        \right\}
    \right)
    \leq2\exp\left(-\frac{\ell^2}{2}\right).
\]
\end{enumerate}
\end{theorem}

\begin{proof}  
As noted at the start of the proof of Lemma~\ref{dolif:lem:compareP1P2}, Condition~(B1) together with $a,b>0$ implies that the set $\Q$ in~\eqref{dolif:def:piandQ} is convex. Now let $p^*$ and $p_*$ be respectively the maximizer and minimizer over $\Q_\local$ of  the objective $\widetilde{\Theta}^{\eta,\Q}_{y,u,v}$ at~\eqref{dolif:def:Theta}. We can write $ \widetilde{\Theta}^{\eta,\Q}_{y,u,v}( p^*)-\widetilde{\Theta}^{\eta,\Q}_{y,u,v}( p_*)
    =
    \frac{1}{2\eta}\left(\| p^*- q\|^2-\| p_*- q\|^2\right)
    \le \frac{1}{2\eta}\| p^*- q\|^2$. 
Combining with the fact that 
$\| p- q\|\le 2r_{\local}$ for all $ p\in \Q_{\local}$ to get
\begin{equation}
\label{dolif:range:Theta}
    \widetilde{\Theta}^{\eta,\Q}_{y,u,v}( p^*)-\widetilde{\Theta}^{\eta,\Q}_{y,u,v}( p_*)
    \le \frac{(2r_{\local})^2}{2\eta}
    =\frac{2r_{\local}^2}{\eta}.
\end{equation}
In view of convexity of $\Q$, \eqref{dolif:range:Theta}, the inequality \eqref{ineq:opt} from
Theorem~\ref{theo:vaidya}, and the choice of $\alpha\in(0,1)$ at
\eqref{dolif:def:deltaandalpha}, Theorem~\ref{theo:vaidya} guarantees the cutting-plane method by~\cite{vaidya,lee2015fastercutting} produces a $1/(d+2)$ solution $\tilde{p}$ to the optimization problem
$\min_{ p\in \Q_{\local}} \widetilde{\Theta}^{\eta,\Q}_{y,u,v}( p)$, i.e.,
\[
    \widetilde{\Theta}^{\eta,\Q}_{y,u,v}(\tilde p)-\min_{ p\in \Q_{\local}}\widetilde{\Theta}^{\eta,\Q}_{y,u,v}( p)
    \stackrel{\eqref{ineq:opt}}{\le}
    \alpha\left(\widetilde{\Theta}^{\eta,\Q}_{y,u,v}( p^*)-\widetilde{\Theta}^{\eta,\Q}_{y,u,v}( p_*)\right)
    \stackrel{\eqref{dolif:def:deltaandalpha},\eqref{dolif:range:Theta}}{\le}
    \frac{\tilde\Delta^2}{2\eta}\stackrel{\eqref{dolif:def:deltaandalpha}}{=}\frac{1}{d+2}. 
\]
In particular, the cutting-plane method by~\cite{vaidya,lee2015fastercutting} requires a subgradient oracle for the objective $\widetilde{\Theta}^{\eta,\Q}_{y,u,v}$, which is available explicitly; and a separation oracle for $\Q_{\local}$. The latter is obtained by combining a free separation oracle for $q+2r_{\local}\ball_{d+2}(0)$ with a separation oracle for $\Q$. Hence, per Theorem~\ref{theo:vaidya}, we need ${\mathcal{O}} \left((d+2) \log \frac{(d+2) \gamma}{\alpha}\right)$ expected calls to the separation
oracle of $\Q_{\local}$, which translate to ${\mathcal{O}} \left((d+2) \log \frac{(d+2) \gamma}{\alpha}\right)$ expected calls to the separation
oracle of $\Q$. 

The upper bound on $\gamma$ follows from Lemma~\ref{dolif:lem:localset} with $c=2$, which uses condition~(B2). Finally, we derive the concentration inequalities for $r_{\local}$ and $\alpha$.
Write $Z\sim \mathcal{N}(0,I_{d+2})$. Since $ q=(y,u,v-b\eta)= p_{k-1}+\sqrt{\eta} Z-b\eta e_t$ for $e_t=\brac{0,\ldots,0,1}$, we have $ r_{\local}=\| p_{k-1}- q\|=\|\sqrt{\eta} Z-b\eta e_t\|
  \le \sqrt{\eta} \|Z\|+b\eta$
and hence $ \frac{r_{\local}}{\sqrt{\eta}}-b\sqrt{\eta}\le \|Z\|$. Combining this with 
the Gaussian concentration inequality from \cite[Equation (3.5)]{ledoux2013probability} which says $ \Pr(\|Z\|\ge \sqrt{d+2}+\ell)\le 2\exp (-\ell^2/2)$, we get $ \Pr \left(
    r_{\local}\ge (\sqrt{d+2}+\ell+b\sqrt{\eta})\sqrt{\eta}
  \right)
  \le 2\exp (-\ell^2/2)$. Thus, for arbitrary $a,b,\eta>0$,
\begin{equation}
\label{dolif:conc:rloc}
    \Pr\left(
        r_{\local}\geq\sqrt\eta(\sqrt{d+2}+\ell+b\sqrt\eta)
    \right)
    \leq2\exp\left(-\frac{\ell^2}{2}\right).
\end{equation}
Finally, when $r_{\local}\leq\sqrt\eta(\sqrt{d+2}+\ell+b\sqrt\eta)$, \eqref{dolif:def:deltaandalpha} implies $\alpha\geq
\min\left\{
\frac{1}{2(d+2)(\sqrt{d+2}+\ell+b\sqrt\eta)^2},\frac12
\right\}$.  Combining this estimate with~\eqref{dolif:conc:rloc} proves~Part a.
\end{proof}
We are now ready to prove Theorem \ref{dolif:theo:separation}(a).

\textbf{Proof of Theorem \ref{dolif:theo:separation}(a):} 
Lemma~\ref{dolif:lem:localset} with $c=2$ and
Theorem~\ref{dolif:theo:cuttingplane} imply  $\gamma
\leq
\frac{2(\lambda+\mu)}{\mu}$.   
Since $a=b=A$ and $A\geq L_f,L_h$, we have $\mu\geq1/\sqrt{3}$ and $\lambda=3$. Therefore, $\gamma
\leq
\frac{2(3+\mu)}{\mu}
\leq
2+6\sqrt{3}$.  Theorem~\ref{dolif:theo:cuttingplane} also says for every $\ell>0$, $\Pr\left(
\alpha\leq
\min\left\{
\frac{1}{
2(d+2)(\sqrt{d+2}+\ell+b\sqrt{\eta})^2
},\frac12
\right\}
\right)
\leq
2\exp\left(-\frac{\ell^2}{2}\right)$. 
Since $b=A$ and $\eta=A^{-2}$, we have $b\sqrt{\eta}=1$. Taking $\ell=\sqrt{d+2}$ leads to \eqref{dolif:eq:concent}. \QEDA

\subsection{Sampling $ p\sim \exp(-\widetilde{\mathcal{P}}_1(p))$ in Algorithm~\ref{dolif:RGO}}
\label{dolif:appen:specialalgo}

This part is similar to Appendix~\ref{silif:appen:specialalgo}. Recall $\widetilde{\mathcal{P}}_1$ is defined at~\eqref{dolif:def:tildeP1} and $\tilde p=(\tilde x,\tilde s,\tilde t)$ is defined in Algorithm~\ref{dolif:RGO}.
By completing the square, one can easily see that sampling $w\sim \exp(-\widetilde{\mathcal{P}}_1(p))$ is equivalent to $ p\sim \Lambda(p) \propto
 \exp(-\frac{1}{2\eta}(\norm{p-\tilde p}-\tilde \Delta)^2)$. While $\Lambda( p)$ is not a Gaussian density, generating $p\sim \Lambda(p)$ is straightforward since it can be turned into a one-dimensional sampling problem. We state here a generic procedure for this sampling problem. 
\begin{algorithm}[H]
	\caption{Sample $ p\sim  \Lambda( p)$ } 
	\label{dolif:alg:samplingP1}
	\begin{algorithmic}
		\State 1. Generate $W\sim {\cal N}(0,I_{d+2})$  and set $\theta = W/\|W\|$; 
		\State 2. Generate $r \propto r^{d+1} \exp(-\frac{1}{2\eta} (r - \tilde \Delta)^2)$ by Adaptive Rejection Sampling for one-dimensional log-concave distribution by \cite{gilks1992adaptive}. 
		\State 3. Output $ p = \tilde p + r \theta$.
	\end{algorithmic}
\end{algorithm}


\end{document}